\documentclass{article}
\usepackage[hmarginratio=1:1,top=10mm,right=10mm,left=10mm,bottom=10mm,columnsep=20pt, headsep=7mm]{geometry}
\usepackage{fullpage}
\usepackage[english]{babel}
\usepackage[utf8x]{inputenc}

\usepackage{cite}
\usepackage[usenames,dvipsnames]{xcolor}
\usepackage{lastpage} 
\usepackage{fancyhdr} 								
\usepackage[protrusion=true,expansion=true]{microtype}
\usepackage{xargs}

\usepackage{amsfonts, amsmath, amssymb, amsthm}
\usepackage{mathtools}
\usepackage{mathrsfs}
\usepackage{upgreek}
\usepackage{bbm}

\theoremstyle{plain}
	\newtheorem{theorem}{\sffamily Theorem}[section]
	
	\newtheorem{lemma}[theorem]{\sffamily Lemma}
	
	\newtheorem{assumption}[theorem]{\sffamily Assumption}
	\newtheorem{remark}[theorem]{\sffamily Remark}
	\newtheorem{definition}[theorem]{\sffamily Definition}
	\newtheorem{example}[theorem]{\sffamily Example}

\newcommand{\vbs}{{\boldsymbol{\mathsf{v}}}}
\newcommand{\wbs}{{\boldsymbol{\mathsf{w}}}}
\newcommand{\gbs}{{\boldsymbol{\mathsf{g}}}}
\newcommand{\xbs}{{\boldsymbol{\mathsf{x}}}}
\newcommand{\ybs}{{\boldsymbol{\mathsf{y}}}}
\newcommand{\zbs}{{\boldsymbol{\mathsf{z}}}}
\newcommand{\VA}{{\mathbf{A}}}
\newcommand{\VB}{{\mathbf{B}}}

\providecommand{\CX}{{\cal X}}
\providecommand{\CY}{{\cal Y}}
\providecommand{\CI}{{\cal I}}
\providecommand{\CF}{{\cal F}}
\providecommand{\CH}{{\cal H}}
\providecommand{\CJ}{{\cal J}}
\providecommand{\CM}{{\cal M}}
\providecommand{\CT}{{\cal T}}
\providecommand{\CO}{{\cal O}}
\providecommand{\CL}{{\cal L}}
\providecommand{\bbR}{\mathbb{R}}
\providecommand{\bbE}{\mathbb{E}}
\providecommand{\bbP}{\mathbb{P}}
\providecommand{\bbN}{\mathbb{N}}

\newcommand{\Vone}       {\boldsymbol{1}}

\newcommand*{\DP}[2]{\left<{#1},{#2}\right>} 					
\newcommand*{\LRR}[1]{\left(#1\right)}							

\newcommand{\xref}{\xbs^\dagger}

\newcommand{\dmapX}[1]{\CJ_{#1}^{\CX}}
\newcommand{\dmapXs}[1]{\CJ_{#1^\ast}^{\CX^{\ast}}}
\newcommand{\dmapY}[1]{\CJ_{#1}^{\CY}}

\newcommand{\svaldmapY}[1]{\jmath_{#1}^{\CY}}
\newcommand{\svaldmapX}[1]{\jmath_{#1}^{\CX}}

\newcommand{\bregman}[2]{\VB_p(#1, #2)}
\newcommand{\bregmanSdef}[2]{\VB_{p^\ast}\big(#1, #2\big)}
\newcommand{\bregmanS}[2]{\VB_{p^\ast}\!\Big(\!\dmapX{p}(#1), \dmapX{p}(#2)\!\Big)}

\newcommand{\norm}[1]{\|#1\|}

\providecommand{\sign}{\operatorname{sign}}						
\providecommand{\argmin}{\operatorname*{argmin}}				

\providecommand{\liminf}{\operatorname*{lim\,inf}}
\providecommand{\range}{\operatorname{range}}                  	
\providecommand{\Null}{\operatorname{null}}						

\newcommand{\xN}[1]{\|#1\|_{\CX}}
\newcommand{\xsN}[1]{\|#1\|_{\CX^\ast}}
\newcommand{\yN}[1]{\|#1\|_{\CY}}
\newcommand{\ysN}[1]{\|#1\|_{\CY^\ast}}

\begin{document}
  \title{On the Convergence of Stochastic Gradient Descent for Linear Inverse Problems in Banach Spaces}
  \author{B. Jin and Z. Kereta}
   \date{} 
\maketitle

\begin{abstract}
In this work we consider stochastic gradient descent (SGD) for solving linear inverse problems in Banach spaces. 
SGD and its variants have been established as one of the most successful optimisation methods in machine learning, imaging and signal processing, etc.
At each iteration SGD uses a single datum, or a small subset of data, resulting in highly scalable methods that are very attractive for large-scale inverse problems.
Nonetheless, the theoretical analysis of SGD-based approaches for inverse problems has thus far been largely limited to Euclidean and Hilbert spaces.
In this work we present a novel convergence analysis of SGD for linear inverse problems in general Banach spaces: we show the almost sure convergence of the iterates to the minimum norm solution and establish the regularising property for suitable \textit{a priori} stopping criteria. Numerical results are also presented to illustrate features of the approach.
\end{abstract}


\section{Introduction}

This work considers (stochastic) iterative solutions for linear operator equations of the form
\begin{align}\label{eqn:inv}
    \VA\xbs=\ybs
\end{align}
where $\VA:\CX\rightarrow\CY$ is a bounded linear operator between Banach spaces $\CX$ and $\CY$ (equipped with the norms $\|\cdot\|_{\CX}$ and $\|\cdot\|_{\CY}$, respectively), and $\ybs\in\range{(\VA)}$ is the exact data.
In practice, we only have access to noisy data $\ybs^\delta=\ybs+{\boldsymbol{\xi}}$, where {$\boldsymbol{\xi}$} denotes the measurement noise with a noise level $\delta\geq0$ such that $\yN{\ybs^\delta-\ybs}{\leq}\delta$.
Linear inverse problems arise naturally in many applications in science and engineering, and also form the basis for studying nonlinear inverse problems. Hence, design and analysis of stable reconstruction methods for linear inverse problems have received much attention.

Iterative regularisation is a powerful algorithmic paradigm that has been successfully employed 
for many inverse problems \cite[Chapters 6 and 7]{EnglHankeNeubauer:1996} \cite{KaltenbacherNeubauerScherzer:2008}.
Classical iterative methods for inverse problems include (accelerated) Landweber method, conjugate 
gradient method, Levenberg-Marquardt method, and Gauss-Newton method, to name a few.
The per-iteration computational bottleneck of many iterative methods lies in utilising all the data at each iteration, which can be of a prohibitively large size.
For example, this occurs while computing the derivative of an objective. One promising strategy to 
overcome this challenge is stochastic gradient descent (SGD), due to Robbins and Monro \cite{RM51}. SGD 
decomposes the original problem into (finitely many) sub-problems, and then at each iteration uses 
only a single datum, or a mini-batch of data, typically selected uniformly at random. This greatly 
reduces the computational complexity per-iteration, and enjoys excellent scalability with respect to
data size. In the standard, and best 
studied setting, $\CX$ and $\CY$ are finite dimensional Euclidean spaces and the corresponding data fitting objective is the 
(rescaled) least squares $\Psi(\xbs)=\frac{1}{2N}\yN{\VA\xbs-\ybs}^2 = 
\frac{1}{N}\sum_{i=1}^N \frac{1}{2}\yN{\VA_i\xbs-\ybs_i}^2$. In this setting SGD takes the form
\begin{align*}
    \xbs_{k+1} = \xbs_k - \mu_{k+1}\VA_{i_{k+1}}^\ast(\VA_{i_{k+1}}\xbs_k-\ybs_{i_{k+1}}), \quad k=0,1,\ldots,
\end{align*}
where $\mu_k>$ is the step-size, ${i_{k+1}}$ a randomly selected index, $\VA_i$ the $i$-th row of a matrix $\VA$, and $\ybs_i$ the $i$-th entry of $\ybs$. In the seminal work \cite{RM51}, Robbins and Monro presented SGD as a Markov chain, 
laying the groundwork for the field of stochastic approximation \cite{KushnerYin:2003}. SGD has 
since had a major impact on statistical inference and machine learning, especially for the training of neural networks. SGD has been extensively studied in the Euclidean setting; see \cite{BCN18} for an overview of the convergence theory from the viewpoint 
of optimisation.

SGD has also been a popular method for image reconstruction, especially in medical imaging.
For example, the (randomised) Kaczmarz method is a reweighted version of SGD that has been extensively 
used in computed tomography \cite{HM93, NZZ15}. Other applications of SGD and its variants
include optical tomography \cite{ChenLiLiu:2018}, phonon transmission coefficient recovery 
\cite{GambaLiNair:2022}, positron emission tomography \cite{Z+21}, as well as general sparse recovery \cite{SchopferLorenz:2019,SchopferLorenzTondji:2022}. For linear inverse problems in Euclidean spaces, Jin and Lu \cite{JL19} gave a first proof of convergence 
of SGD iterates towards the minimum norm solution, and analysed the regularising behaviour in the presence of noise; see \cite{JahnJin:2020,JinZhouZou:2020,JinZhouZou:2021siuq,LuMathe:2022,RabeloLeitao:2022} for further convergence results, \textit{a posteriori} stopping rules (discrepancy principle), nonlinear problems, and general step-size schedules, etc.

Iterative methods in Euclidean and Hilbert spaces are effective for reconstructing smooth solutions but fail to capture special features of the solutions, such as sparsity and piecewise constancy. In practice, many imaging inverse problems are more adequately described in non-Hilbert settings, including
sequence spaces $\ell^p(\bbR)$ and Lebesgue spaces $\CL^p(\Omega)$, with $p\in[1,\infty]
\setminus\{2\}$, which requires changing either the solution, the data space, or both. For example, inverse problems with impulse noise are better modelled by setting the data 
space $\CY$ to a Lebesgue space $\CL^p(\Omega)$ with $p\approx 1$ \cite{ClasonJinKunisch:2010}, whereas the recovery of sparse solutions 
is modelled by doing the same to the solution space $\CX$ \cite{CandesRombertTao:2006}. Thus, it is of great importance to develop 
and analyse algorithms for inverse problems in Banach spaces, and this has received much attention \cite{SG+09,TBHK_12}. For the Landweber method for linear inverse problems in Banach spaces, Sch\"opfer et al \cite{SLS_06} were the first to prove strong convergence of the iterates under a suitable 
step-size schedule for a smooth and uniformly convex Banach space $\CX$ and an arbitrary Banach space 
$\CY$. This has since been extended and refined in various aspects, e.g. regarding acceleration \cite{SS09, W18, GHC19,ZhongWangJin:2019}, nonlinear forward models \cite{dHQS12, M18}, 
and Gauss-Newton methods \cite{KH10}.

In this work, we investigate SGD for inverse problems in Banach spaces, which has thus far lagged behind due to outstanding challenges in extending the analysis of standard Hilbert 
space approaches to the Banach space setting. The main challenges in analysing SGD-like gradient-based 
methods in Banach spaces are two-fold: 
\begin{enumerate}
\item The use of duality maps results in non-linear update rules, which greatly complicates the convergence analysis. For example, the (expected) 
difference between successive updates can no longer be identified as the (sub-)gradient of the objective.
    \item Due to geometric characteristics of Banach spaces, it is more common to use the Bregman distance 
for the convergence analysis, which results in the loss of useful algebraic tools, e.g. triangle 
inequality and bias-variance decomposition, that are typically needed for the analysis.
\end{enumerate}

In this work, 
we develop an SGD approach for the numerical solution of linear inverse problems in Banach spaces,
using the sub-gradient approach based on duality maps, and present a novel convergence analysis.
We first consider the case of exact data, and show that SGD iterates converge to a minimising 
solution (first almost surely and then in expectation) under standard assumptions on summability of step-sizes, and geometric properties of the space $\CX$, cf. Theorems
\ref{thm:as_lin_convergence} and \ref{thm:L1_lin_convergence}. This solution is identified 
as the minimum norm solution if the initial guess $\xbs_0$ satisfies the range condition $\dmapX{p}(\xbs_0)\in\overline{\rm{range}(\VA^\ast)}$.
Further, we give a convergence rate in Theorem \ref{thm:lin-main} when the forward operator $\VA$ satisfies a conditional stability 
estimate. In case of noisy observations, we show the regularising property 
of SGD, for properly chosen stopping indices, cf. Theorem \ref{thm:regularisation_property}. The analysis rests on a descent property in Lemma \ref{lem:descent_property} and Robbins-Siegmund theorem for almost super-martingales. In addition, we perform extensive numerical experiments on a model inverse problem (linear integral equation) and computed tomography (with parallel beam geometry) to illustrate distinct features of the proposed Banach space SGD, and examine the influence of various factors, such as the choice of the spaces $\CX$ and $\CY$, mini-batch size and noise characteristics (Gaussian or impulse).

When finalising the paper, we became aware of the independent and simultaneous work \cite{JinLuZhang:2022} on a stochastic mirror descent method for linear inverse problems between a Banach space $\mathcal{X}$ and a Hilbert space $\mathcal{Y}$. The method is a randomised version of the well-known Landweber-Kaczmarz method. The authors prove convergence results under \textit{a priori} stopping rules, and also establish an order-optimal convergence rate when the exact solution $\xbs^\dag$ satisfies a benchmark source condition, by interpreting the method as a randomised block gradient method applied to the dual problem. Thus, the current work differs significantly from \cite{JinLuZhang:2022} in terms of problem setting, main results and analysis techniques.

The rest of the paper is organised as follows. In Section \ref{sec:prelims}, we recall background 
materials on the geometry of Banach spaces, e.g. duality maps and Bregman distance. In Section \ref{sec:linear}, we 
present the convergence of SGD for exact data and in Section \ref{sec:regularisation}, we discuss the regularising property of SGD for noisy observations. Finally, in Section \ref{sec:experiments}, we 
provide some experimental results on a model inverse problem and computed tomography. In the {A}ppendix we collect several useful inequalities and auxiliary estimates.

Throughout, let $\CX$ and $\CY$ be two real Banach spaces, with their norms denoted by $\xN{\cdot}$ and $\yN{\cdot}$, respectively.
$\CX^\ast$ and $\CY^\ast$ are their respective dual spaces, with their norms denoted by $\|\cdot\|_{\CX^\ast}$ and $\|\cdot\|_{\CY^\ast}$, respectively.
For $\xbs\in\CX$ and $\xbs^\ast\in\CX^\ast$, we denote the corresponding duality pairing by
$\DP{\xbs^\ast}{\xbs}=\DP{\xbs^\ast}{\xbs}_{\CX^\ast\times\CX}=\xbs^\ast(\xbs)$.
For a continuous linear operator $\VA:\CX\rightarrow \CY$, we use $\|\VA\|_{\CX\to \CY}$ to denote the operator norm (often with the subscript omitted).
The adjoint of $\VA$ is denoted as $\VA^\ast:\CY^\ast\rightarrow\CX^\ast$ and it is a continuous linear operator, with $\norm{\VA}_{\CX\to\CY}=\norm{\VA^\ast}_{\CY^\ast\to\CX^\ast}$.
The conjugate exponent of $p\in (1,\infty)$ is denoted by $p^\ast$, such that $1/p+1/p^\ast=1$ holds.
The Cauchy-Schwarz inequality of the following form holds for any $\xbs\in\CX$ and $\xbs^\ast\in\CX^\ast$
\begin{equation}\label{eqn:Cauchy_Schwarz}
    | \DP{\xbs^\ast}{\xbs}|\le \xsN{\xbs^\ast}\xN{\xbs}.
\end{equation}
For reals $a,b$ we write $a\wedge b=\min\{a,b\}$ and $a\vee b=\max\{a,b\}$. By $(\CF_k)_{k\in\bbN}$, 
we denote the natural filtration, i.e. a growing sequence of $\sigma$-algebras  such that $\CF_k
\subset\CF_{k+1}\subset \CF$, for all $k\in\bbN$ and a $\sigma$-algebra $\CF$, and $\CF_k$ is generated by 
random indices $i_j$, for $j\leq k$.
In the context of SGD, $k\in\bbN$ is the iteration number and 
$\CF_k$ denotes the iteration history, that is, information available at time $k$, and for a given 
initialisation $\xbs_0$, we can identify $\CF_k=\sigma(\xbs_1,\ldots,\xbs_k)$.
For a filtration $(\CF_k)_{k\in\bbN}$ we denote by $\bbE_k[\cdot]=\bbE[\cdot\mid\xbs_1,\ldots\xbs_k]$ the conditional expectation with respect to $\CF_k$. 
{A sequence of random variables $(x_k)_{k\in\bbN}$ (adapted to the filtration $(\mathcal{F}_k)_{k\in\mathbb{N}}$) is a called super-martingale if $\bbE_k[x_{k+1}]\leq x_k.$}
Throughout, the notation a.s. denotes almost sure events.


\section{Preliminaries on Banach spaces}\label{sec:prelims}
In this section we recall relevant concepts from Banach space theory and the geometry of Banach spaces.

\subsection{Duality map}

In a Hilbert space $\CH$, for every $\xbs\in\CH$, there exists a unique $\xbs^\ast\in\CH^\ast$ such that $\DP{\xbs^\ast}{\xbs}=\|\xbs\|_{\CH}\|\xbs^\ast\|_{\CH^\ast}$ and $\|\xbs^\ast\|_{\CH^\ast}=\|\xbs\|_{\CH}$, by the Riesz representation theorem.
For Banach spaces, however, such an $\xbs^\ast$ is not necessarily unique, motivating the notion of duality maps.

\begin{definition}[Duality map]\label{defn:duality_map}
For any $p>1$, a \emph{duality map} $\dmapX{p}:\CX\rightarrow 2^{\CX^\ast}$ is the sub-differential of the {\rm(}convex{\rm)} functional $\frac{1}{p}\|\xbs\|_{\CX}^p$
\begin{align}\label{eqn:dmap_defined}
\dmapX{p} (\xbs) = \left\{\xbs^\ast\in\CX^\ast : \DP{\xbs^\ast}{\xbs}=\xN{\xbs}\xsN{\xbs^\ast}, \text{ and } \xN{\xbs}^{p-1}=\xsN{\xbs^\ast} \right\},
\end{align}
with gauge function $t\mapsto t^{p-1}$.
A single-valued selection of $\dmapX{p}$ is denoted by $\svaldmapX{p}$.
\end{definition}

In practice, the choice of the power parameter $p$ depends on geometric properties of the space $\CX$. {For single-valued duality maps, we use $\dmapX{p}$ and $\svaldmapX{p}$ interchangeably.}
Next we recall standard notions of smoothness and convexity of Banach spaces.
For an overview of Banach space geometry, we refer an interested reader to the monographs \cite{C_09, Cioranescu:1990,TBHK_12}.

\begin{definition}\label{defn:smoothness_and_convexity}
Let $\CX$ be a Banach space. 
{$\CX$ is said to be reflexive if the canonical map $\xbs\mapsto\widehat\xbs$ between $\CX$ and the bidual $\CX^{\ast\ast}$, defined by $\widehat\xbs(\xbs^\ast) = \xbs^\ast(\xbs)$, is surjective. $\CX$ is smooth if for every $0\neq\xbs\in\CX$ there is a unique $\xbs^\ast\in\CX^\ast$ such that $\DP{\xbs^\ast}{\xbs}=\xN{\xbs}$ and $\xsN{\xbs^\ast}=1$.}
The function $\delta_\CX:(0,2]\rightarrow\bbR$ defined as
\begin{align*}
    \delta_\CX(\tau) = \inf\Big\{1-\tfrac{1}{2}\xN{\zbs+\wbs} : \xN{\zbs}=\xN{\wbs}=1, \xN{\zbs-\wbs}\geq\tau \Big\}
\end{align*}
is the \emph{modulus of convexity} of $\CX$.
$\CX$ is said to be \emph{uniformly convex} if $\delta_\CX(\tau)>0$ for all $\tau\in(0,2]$, and \emph{$p$-convex}, for $p>1$, if $\delta_\CX(\tau)\geq K_p\tau^p$ for some $K_p>0$ and all $\tau\in(0,2]$. The function $\rho_\CX:{[0,\infty)\to[0,\infty)}$ defined as
\begin{align*}
    \rho_\CX(\tau) = \sup\Big\{\frac{\xN{\zbs+\tau\wbs}+\xN{\zbs-\tau\wbs}}{2}-1 : \xN{\zbs}=\xN{\wbs}=1\Big\}
\end{align*}
is the \emph{modulus of smoothness} of $\CX$, and is a convex and continuous function such that $\frac{\rho_\CX(\tau)}{\tau}$ is a non-decreasing function with $\rho_\CX(\tau)\leq\tau$.
$\CX$ is said to be \emph{uniformly smooth} if $\lim_{\tau\searrow0}\frac{\rho_\CX(\tau)}{\tau}=0$, and \emph{$p$-smooth}, for $p>1$, if $\rho_\CX(\tau)\leq K_p\tau^p$ for some $K_p>0$ and all $\tau\in(0,\infty)$.
\end{definition}
The following relationships between Banach spaces and duality maps will be used extensively. 
\begin{theorem}[{\cite[Theorems 2.52 and 2.53, and Lemma 5.16]{TBHK_12}}]\label{rem:dmap_properties}
\begin{enumerate}
\item[{\rm(i)}] For every $\xbs\in\CX$, the set $\dmapX{p}(\xbs)$ is non-empty, convex, and weakly-$\star$ closed in $\CX^\ast$.
\item[{\rm(ii)}]\label{rem:Xp_Xstarpstar} $\CX$ is $p$-smooth if and only if $\CX^\ast$ is $p^\ast$-convex.
$\CX$ is $p$-convex if and only if $\CX^\ast$ is $p^\ast$-smooth.
\item[{\rm(iii)}]\label{property:dual_inverse}
$\CX$ is smooth if and only if $\dmapX{p}$ is single valued. 
If $\CX$ is {convex of power type and smooth,} then $\dmapX{p}$ is invertible and $\big(\dmapX{p}\big)^{-1}=\dmapXs{p}$.
If $\CX$ is uniformly smooth and uniformly convex, then $\dmapX{p}$ and $\dmapXs{p}$ are both {uniformly} continuous.
\item[{\rm(iv)}] {Let $\CX$ be a uniformly smooth Banach space with duality map $\dmapX{p}$ with $p\geq 2$. Then, for all $\xbs,\widetilde\xbs\in \CX$, there holds
\begin{align*}
    \xsN{\dmapX{p}(\xbs)-\dmapX{p}(\widetilde \xbs)}^{p^\ast}\leq C \max\{1, \xN{\xbs}, \xN{\widetilde\xbs}\}^{p}\, {\overline{\rho}_\CX}(\xN{\xbs-\widetilde\xbs})^{p^\ast},
\end{align*}
where $\overline{\rho}_\CX(\tau)=\rho_\CX(\tau)/\tau$ is a modulus of smoothness function such that $\overline{\rho}(\tau)\leq1$.}
\end{enumerate}
\end{theorem}

Next we list some common Banach spaces, the corresponding duality maps and convexity and smoothness properties.
\begin{example}\label{eg:spaces}
\begin{enumerate}
\item[{\rm(i)}] A Hilbert space $\CX$ is $2$-smooth and $2$-convex, and $\dmapX{2}$ is the identity.
\item[{\rm(ii)}] If $\CX$ is smooth, then $\dmapX{p}$ is the Gateaux derivative of the functional $\xbs\mapsto\frac{1}{p}\xN{\xbs}^p$.
\item[{\rm(iii)}] If $\CX=\ell^r(\bbR)$ with $1<r<\infty$, then $\dmapX{p}$ is single-valued, and the duality map is given by
$\dmapX{p}(\xbs) = \|\xbs\|_r^{p-r} |\xbs|^{r-1}\sign(\xbs).$
Moreover, $\dmapX{p}=\nabla({\frac{1}{p}\xN{\cdot}^p})$ since $\CX$ is smooth.
\item[{\rm(iv)}] Lebesgue spaces $\CL^p(\Omega)$, Sobolev spaces $W^{s,p}(\Omega)$, with $s>0$, {\rm(}for an open bounded domain $\Omega${\rm)}, and sequence spaces $\ell^p(\bbR)$ are $p\wedge 2$-smooth and $p\vee2$-convex, for $1<p<\infty$.
For $p\in\{1,\infty\}$, they are neither smooth nor strictly convex.
\end{enumerate}
\end{example}

\subsection{Bregman distance}
Due to the geometry of Banach spaces, it is often more convenient to use the
Bregman distance than the standard Banach space norm $\xN{\cdot}$ in the convergence analysis.
\begin{definition}[Bregman distance]\label{defn:Bregman_distance}
For a smooth Banach space $\CX$, the functional
\begin{align*} \bregman{\zbs}{\wbs}
&= \frac{1}{p^\ast}\xN{\zbs}^p+\frac{1}{p}\xN{\wbs}^p-\DP{\dmapX{p}(\zbs)}{\wbs},
\end{align*}
is called the \emph{Bregman distance}, where $1/p+1/p^\ast=1$.
\end{definition}

Note that the dependence of the Bregman distance $\bregman{\zbs}{\wbs}
$ on the space $\CX$ is omitted, which is often clear from the context. The Bregman distance does not satisfy the triangle inequality, and is generally non-symmetric. Thus it is not a distance. The next theorem lists
useful properties of the Bregman distance, which show the relationship between the
geometry of the underlying Banach space and duality maps.
\begin{theorem}[{\cite[Theorem 2.60, Lemmas 2.62 and 2.63]{TBHK_12}}]\label{thm:bregman_properties}
The following properties hold.
\begin{enumerate}
\item[{\rm(i)}] If $\CX$ is smooth and reflexive, then
$ \bregman{\zbs}{\wbs}={\bregmanS{\wbs}{\zbs}}.$
\item[{\rm(ii)}] Bregman distance satisfies the three-point identity
\begin{align}\label{eqn:3_point_id} \bregman{\zbs}{\wbs}=\bregman{\zbs}{\vbs}+\bregman{\vbs}{\wbs}+\DP{\dmapX{p}(\vbs)-\dmapX{p}(\zbs)}{\wbs-\vbs}.\end{align}
\item[{\rm(iii)}]\label{rem:bregman_pconv} If $\CX$ is $p$-convex, then it is reflexive, $p\geq 2$ and there exists $C_p>0$ such that
\begin{align}\label{eqn:norm_leq_bregman} \bregman{\zbs}{\wbs}\geq p^{-1}C_p\xN{\wbs-\zbs}^p.\end{align}
\item[{\rm(iv)}]\label{rem:bregman_psmooth} If $\CX^\ast$ is $p^\ast$-smooth, then it is reflexive, $p^\ast\le 2$ and there exists $G_{p^\ast}>0$ such that
\begin{align}\label{eqn:norm_geq_bregman} {\bregmanSdef{\zbs^\ast}{\wbs^\ast}\leq (p^\ast)^{-1}G_{p^\ast}\xsN{\wbs^\ast-\zbs^\ast}^{p^\ast}}.\end{align}
\item[{\rm(v)}] $\bregman{\zbs}{\wbs}\geq0$, and if $\CX$ is uniformly convex, $\bregman{\zbs}{\wbs}=0$ if and only if {$\zbs=\wbs$}. 
\item[{\rm(vi)}] $\bregman{\zbs}{\wbs}$ is continuous in the second argument. If $\CX$ is smooth and uniformly convex, then $\dmapX{p}$ is continuous on bounded subsets and $\bregman{\zbs}{\wbs}$ is continuous in its first argument.
\end{enumerate}
\end{theorem}


\section{Convergence analysis for exact data}\label{sec:linear}

Now we develop an SGD type approach for problem \eqref{eqn:inv} and analyse its convergence. 
{Throughout, we make the following assumption on the Banach spaces $\CX$ and $\CY$, unless indicated otherwise.}
{\begin{assumption}\label{ass:space-basic}
The Banach space $\CX$ 
is $p$-convex and smooth, and $\CY$ is arbitrary.
\end{assumption}}

To recover the solution $\xbs^\dag$, we minimise a least-squares type objective
$\argmin_{\xbs\in\CX} \tfrac{1}{p}\yN{\VA\xbs-\ybs}^p,$ for some $p>1$.
By $\CX_{\min}$, we denote the (non-empty) set of minimisers over $\CX$.
Among the elements of $\CX_{\min}$, the regularisation theory focuses on the so-called minimum norm solution.

\begin{definition}\label{defn:minimum_norm}
An element $\xref\in\CX$ is called a \emph{minimum norm solution} {\rm(}MNS\,{\rm)} of \eqref{eqn:inv} if
\[ \VA\xref=\ybs\quad\text{ and }\quad \xN{\xref}=\inf \{\xN{\xbs}:\xbs\in\CX,\, \VA\xbs=\ybs \}.\]
\end{definition}

The MNS $\xref$ is not unique for general Banach spaces. The following 
lemma states sufficient geometric assumptions on $\CX$ for uniqueness. 
\begin{lemma}[{\cite[Lemma 3.3]{TBHK_12}}]
\label{lem:minnorm}
Let Assumption \ref{ass:space-basic} hold.
Then there exists a unique MNS $\xref$.
Furthermore, $\dmapX{p}(\xref)\in\overline{\range(\VA^\ast)}$, for $1<p<\infty$.
If some $\widehat{\xbs}\in\CX$ satisfies $\dmapX{p}(\widehat{\xbs})\in\overline{\range(\VA^\ast)}$ and $\widehat{\xbs}-\xref\in\Null(\VA)$, then $\widehat{\xbs}=\xref$.
\end{lemma}

By Lemma \ref{lem:minnorm}, the MNS $\xref$ is unique modulo 
the null space of $\VA$, under certain smoothness and convexity assumptions on $\CX$.
These conditions exclude Lebesgue and sequence spaces $\CL^1(\Omega)$ and $\ell^1(\bbR)$, cf. Example \ref{eg:spaces}(iv).
The standard Landweber method \cite{Landweber:1951,SLS_06} constructs an approximation to the MNS $\xref$ by running the  iterations
\begin{align}\label{eqn:Landweber}
\xbs_{k+1} = \dmapXs{p}\LRR{\dmapX{p}(\xbs_k) - \mu_{k+1}\VA^\ast\svaldmapY{p}(\VA\xbs_k-\ybs)}, \quad k=0,1,\ldots,
\end{align}
where $\mu_{k+1}>0$ is the step-size.
Asplund's theorem \cite[Theorem 2.28]{TBHK_12} allows characterising the duality map as the 
sub-differential, $\dmapX{p}=\partial({\frac{1}{p}\xN{\cdot}^p})$ for $p>1$.
This identifies the descent direction $\VA^\ast\svaldmapY{p}(\VA\xbs_k-\ybs)$ as the sub-gradient: $\VA^\ast\svaldmapY{p}(\VA\xbs_k-\ybs)=\partial(\frac{1}{p}\yN{\VA\cdot-\ybs})(\xbs_k)$.
{Note that $\dmapX{p}$ is single valued by Assumption \ref{ass:space-basic} and Theorem \ref{rem:dmap_properties}, though $\dmapY{p}$ is not.}
For well selected step-sizes, Landweber iterations \eqref{eqn:Landweber} converge to an MNS of \eqref{eqn:inv} \cite[Theorem 3.3]{SLS_06}. 

The evaluation of the sub-gradient $\VA^\ast \svaldmapY{p}(\VA\xbs_k-\ybs)$ represents the main
per-iteration cost of the iteration \eqref{eqn:Landweber}. {In this work, we consider the following Kaczmarz type setting: 
\begin{align}\label{eqn:Kaczmarz_model}
    \VA=\begin{pmatrix}\VA_1\\\vdots\\\VA_N\end{pmatrix}\quad \text{and} \quad\VA\xbs = \begin{pmatrix}\VA_1\xbs\\\vdots\\\VA_N\xbs\end{pmatrix}=\begin{pmatrix}\ybs_1\\\vdots\\\ybs_N\end{pmatrix},
\end{align}
where $\VA_i:\CX\rightarrow\CY_i$, $\ybs_i\in\CY_i$, for $i\in[N]=\{1,\ldots,N\}$. Problem \eqref{eqn:Kaczmarz_model} is defined on the direct product $(\otimes_{i=1}^N\CY_i, \ell^r)$, equipped with the $\ell^r$ norm, for $r\geq 1$
\begin{align}\label{eqn:Norm-Y}
    \|\ybs\|_{\mathcal{Y}}:=\|(\ybs_1,\ldots,\ybs_N)\|_{\mathcal{Y}}= \|(\|\ybs_1\|_{\CY_1},\ldots,\|\ybs_N\|_{\CY_N})\|_{\ell^r}=\Big(\sum_{i=1}^N \|\ybs_i\|_{\mathcal{Y}_i}^r\Big)^{1/r}.
\end{align}
Below we identify $\CY_i=\CY$ for notational brevity, and use $\yN{\cdot}$ to denote both the norm of the direct product space and the component spaces, though all the relevant proofs and concepts easily extend to the general case. Then the objective $\Psi(\xbs)$ is given by 
\begin{equation*}
    \Psi(\xbs)=\frac{1}{N}\sum_{i=1}^N \Psi_i(\xbs), \quad \mbox{with } \Psi_i(\xbs)=\frac{1}{p}\yN{\VA_i\xbs-\ybs_i}^p.
\end{equation*}
Note that for many common imaging problems we use $\mathcal{Y}=\ell^p(\mathbb{R})$, which then naturally gives 
$\Psi(\xbs)=\frac{1}{pN}\yN{\VA\xbs-\ybs}^p$. To reduce the computational cost per-iteration, we exploit the finite-sum structure of the objective $\Psi(\xbs)$ and adopt SGD iterations of the form
\begin{align}\label{eqn:sgd}
\xbs_{k+1} = \dmapXs{p}\LRR{\dmapX{p}(\xbs_k) - \mu_{k+1}\gbs_{k+1}},
\end{align}
where $\gbs_{k+1}=g(\xbs_k,\ybs,i_{k+1})$ is the stochastic update direction given by
\begin{align}\label{eqn:Kaczmarz_linear_gradients} g(\xbs,\ybs,i)=\VA_i^\ast \svaldmapY{p}(\VA_i\xbs-\ybs_i)= \partial\Big(\tfrac{1}{p}\yN{\VA_i\cdot-\ybs_i}^p\Big)(\xbs),
\end{align}
and the random index $i_{k}$ {is sampled uniformly over the index set $[N]$,}
independent of $\xbs_k$. Clearly, it is an unbiased estimator of the sub-gradient $\partial\Psi(\xbs)$, i.e. $\bbE[g(\xbs,\ybs,i)]=\partial\Psi(\xbs)$, and the per-iteration cost is reduced by a factor of $N$.}

\begin{remark}\label{rmk:product}
In the model \eqref{eqn:Kaczmarz_model}, if $\CY$ admits a complemented sum $\CY=\sum_{i=1}^N\CY_i$, we can take the {\rm(}internal\,{\rm)} direct sum $(\oplus_{i=1}^N \CY_i, \ell^r)$, so that $\ybs=\ybs_1+\ldots+\ybs_N$ and the corresponding norm 
$\|\ybs\|=\|(\|{\rm Proj}_{\CY_1}(\ybs) \|_{\CY_1},\ldots, \|{\rm Proj}_{\CY_N}(\ybs) \|_{\CY_N})\|_r$.
With this identification the spaces $(\otimes_{i=1}^N\CY_i, \ell^r)$ and $(\oplus_{i=1}^N \CY_i, \ell^r)$ are isometrically isomorphic \cite{Unser:2022} and the norms are equivalent for all $r\ge1$.
\end{remark}

{We now collect some useful properties about the objective $\Psi$ and the Bregman divergence. Throughout, $L_{\max}=\max_{i\in[N]}\|\VA_i\|.$ Note that $c_N=1/N$ if $\CY=\CL^p(\Omega)$.
\begin{lemma}\label{lem:Kaczmarz-basic}
For all $i\in[N]$, $\xbs\in\CX$, and any $\widehat\xbs\in\mathcal{X}_{\min}$ (such that $\VA\widehat\xbs=\ybs$), we have
\begin{align}\label{eqn:id-subgrad}
\DP{\partial\Psi_i(\xbs)}{\xbs-\widehat\xbs}  = p\Psi_i(\xbs)\quad \text{and}\quad  \DP{\partial\Psi(\xbs)}{\xbs-\widehat\xbs}  = p\Psi(\xbs).
\end{align}
Moreover, $\Psi_i(\xbs)\leq \frac{\|\VA_i\|^p}{C_p}\bregman{\xbs}{\widehat\xbs}$, $\Psi(\xbs)\leq \frac{L_{\max}^p}{C_p}\bregman{\xbs}{\widehat\xbs}$, and for some $C_N>0$ we have
$\Psi(\xbs)\geq \frac{C_N}{p}\yN{\VA\xbs-\ybs}^p.$
\end{lemma}
\begin{proof}
It follows from the identity $\VA\widehat\xbs=\ybs$ that
\begin{align*}
\DP{\partial\Psi_i(\xbs)}{\xbs-\widehat\xbs}  = \DP{\VA_i^\ast \svaldmapY{p}(\VA_i\xbs-\ybs_i)}{\xbs-\widehat\xbs}=\DP{ \svaldmapY{p}(\VA_i\xbs-\ybs_i)}{\VA_i\xbs-\ybs_i} = p \Psi_i(\xbs).
\end{align*}
Since $\partial\Psi(\xbs)=\frac{1}{N}\sum_{i=1}^N \partial\Psi_i(\xbs)$, the second identity in \eqref{eqn:id-subgrad} follows from the linearity of the dual product.
By the $p$-convexity of the space $\CX$ and Theorem \ref{thm:bregman_properties}(iii), we get
\begin{align*}
    \Psi_i(\xbs)=\frac{1}{p}\yN{\VA_i\xbs-\ybs_i}^p =\frac{1}{p}\yN{\VA_i(\xbs-\widehat\xbs)}^p\leq  \frac{\|\VA_i\|^p}{p}\xN{\xbs-\widehat\xbs}^p\leq \frac{\|\VA_i\|^p}{C_p}\bregman{\xbs}{\widehat\xbs}.
\end{align*}
The second claim follows since $\Psi(\xbs)=\frac{1}{N}\sum_{i=1}^N\Psi_i(\xbs)$.
Lastly, by the norm equivalence \eqref{eqn:Norm-Y} for $1<r<\infty$, there exists $C_N>0$ such that
\begin{align*}
\Psi(\xbs)=\frac{1}{N}\sum_{i=1}^N\frac{1}{p}\yN{\VA_i\xbs-\ybs_i}^p\geq \frac{C_N}{p}\yN{\VA\xbs-\ybs}^p.
\end{align*}
\end{proof}}

We now focus on the convergence study of the iterations \eqref{eqn:sgd}, without and with noise in the data, and discuss convergence rates under conditional stability. 
\subsection{Convergence for the Kaczmarz model}

Below the notation $\bbE[\cdot]$ denotes taking expectation with respect to the sampling of the random indices $i_k$ and $\bbE_k[\cdot]$ denotes taking conditional expectation with respect to $\mathcal{F}_k$. The remaining variables, e.g. $\xbs$ and $\ybs$, are measurable with respect to the underlying probability measure. To study the convergence of SGD \eqref{eqn:sgd}, we first establish
a descent property in terms of the Bregman distance.
\begin{lemma}\label{lem:descent_property}
Let Assumption \ref{ass:space-basic} hold. For any $\widehat\xbs\in\CX$, the iterates in \eqref{eqn:sgd} satisfy
\begin{align}\label{eqn:descent_property}
\bregman{\xbs_{k+1}}{\widehat\xbs}\leq\bregman{\xbs_k}{\widehat\xbs}-\mu_{k+1}\DP{\gbs_{k+1}}{\xbs_k-\widehat\xbs} + \frac{G_{p^\ast}}{p^\ast}\mu_{k+1}^{p^\ast}\xsN{\gbs_{k+1}}^{p^\ast}.
\end{align}
\end{lemma}
\begin{proof}
Let $\Delta_k:=\bregman{\xbs_k}{\widehat\xbs}$.
By Definition \ref{defn:Bregman_distance} and expression \eqref{eqn:sgd}, we have
\begin{align*}
\Delta_{k+1}&=\frac{1}{p}\xN{\widehat\xbs}^p+\frac{1}{p^\ast}\xN{\xbs_{k+1}}^{p}-\DP{\dmapX{p}(\xbs_{k+1})}{\widehat\xbs}\\
&{=\frac{1}{p}\xN{\widehat\xbs}^p+\frac{1}{p^\ast}\xN{\dmapXs{p}\LRR{\dmapX{p}(\xbs_k) - \mu_{k+1}\gbs_{k+1}}}^{p}-\DP{\dmapX{p}(\xbs_{k+1})}{\widehat\xbs}}.
\end{align*}
{Using Definition \ref{defn:duality_map}, the identity $p(p^\ast-1)=p^\ast$ and Theorem \ref{rem:dmap_properties}(iii), we deduce 
\begin{align*}
\Delta_{k+1}&=\frac{1}{p}\xN{\widehat\xbs}^p+\frac{1}{p^\ast}\xsN{\dmapX{p}(\xbs_k)-\mu_{k+1}\gbs_{k+1}}^{p(p^\ast-1)}-\DP{\dmapX{p}(\xbs_k)-\mu_{k+1}\gbs_{k+1}}{\widehat\xbs}\\
&=\frac{1}{p}\xN{\widehat\xbs}^p+\frac{1}{p^\ast}\xsN{\dmapX{p}(\xbs_k)-\mu_{k+1}\gbs_{k+1}}^{p^\ast}-\DP{\dmapX{p}(\xbs_k)-\mu_{k+1}\gbs_{k+1}}{\widehat\xbs}.
\end{align*}}
Since $\CX$ is $p$-convex, $\CX^\ast$ is $p^\ast$-smooth, cf. Theorem \ref{rem:dmap_properties}(i).
By \cite[Corollary 5.8]{C_09}, this implies
\[ \frac{1}{p^\ast}\xsN{\xbs^\ast -\tilde\xbs^\ast}^{p^\ast}\leq\frac{1}{p^\ast}\xsN{\xbs^\ast}^{p^\ast}
+\frac{G_{p^\ast}}{p^\ast}\xsN{\tilde\xbs^\ast}^{p^\ast}-\DP{\dmapXs{p}(\xbs^\ast)}{\tilde\xbs^\ast},\quad \forall \xbs^\ast,\tilde \xbs^\ast\in\CX^\ast.
\]
Using the identities $p^\ast(p-1)=p$ and $(\dmapX{p})^{-1}=\dmapXs{{p}}$, cf. Theorem \ref{rem:dmap_properties}(iii), we get
\begin{align*}\frac{1}{p^\ast}\xsN{\dmapX{p}(\xbs_k)-\mu_{k+1}\gbs_{k+1}}^{p^\ast}&\leq \frac{1}{p^\ast}\xsN{\dmapX{p}(\xbs_k)}^{p^\ast} + \frac{G_{p^\ast}}{p^\ast}\xsN{\mu_{k+1}\gbs_{k+1}}^{p^\ast} - \DP{\mu_{k+1}\gbs_{k+1}}{\xbs_k}\\
&= \frac{1}{p^\ast}\xN{\xbs_k}^{p} + \frac{G_{p^\ast}}{p^\ast}\mu_{k+1}^{p^\ast}\xsN{\gbs_{k+1}}^{p^\ast} - \mu_{k+1}\DP{\gbs_{k+1}}{\xbs_k}.
\end{align*}
Combining the preceding estimates gives the desired assertion through
\begin{align*}
\Delta_{k+1} &\leq\frac{1}{p}\xN{\widehat\xbs}^p+\frac{1}{p^\ast}\xsN{\xbs_k}^{p}-\DP{\dmapX{p}(\xbs_k)}{\widehat\xbs}+\frac{G_{p^\ast}}{p^\ast}\mu_{k+1}^{p^\ast}\xsN{\gbs_{k+1}}^{p^\ast} - \mu_{k+1}\DP{\gbs_{k+1}}{\xbs_k-\widehat\xbs}\\
&=\Delta_k-\mu_{k+1}\DP{\gbs_{k+1}}{\xbs_k-\widehat\xbs} + \frac{G_{p^\ast}}{p^\ast}\mu_{k+1}^{p^\ast}\xsN{\gbs_{k+1}}^{p^\ast}.
\end{align*}
\end{proof}

Lemma \ref{lem:descent_property} allows showing that the sequence of Bregman distances $(\bregman{\xbs_k}{\widehat\xbs})_{k\in\bbN}$ forms an almost super-martingale {(in the Robbins-Siegmund sense defined below)} for $\widehat\xbs\in\CX_{\min}$ and well chosen step-sizes $(\mu_k)_{k\in\bbN}$. 
We will show the almost sure convergence of the iterates using Robbins-Siegmund theorem.

\begin{theorem}[{Robbins-Siegmund theorem on the convergence of almost super-martingales, \cite[Lemma 11]{P87}}] \label{thm:sub_martingale_convergence}
Consider a filtration $(\CF_k)_{k\in\bbN}$ and four non-negative, $(\CF_k)_{k\in\bbN}$ adapted processes $(\alpha_k)_{k\in\bbN}$, $(\beta_k)_{k\in\bbN}$, $(\gamma_k)_{k\in\bbN}$, and $(\delta_k)_{k\in\bbN}$.
{Let the sequence $(\alpha_k)_{k\in\bbN}$ be an \emph{almost super-martingale}, i.e. for all $k$ we have 
$\bbE_k[\alpha_{k+1}]\le(1+\beta_k)\alpha_k+\gamma_k-\delta_k.$}
Then the sequence $(\alpha_k)_{k\in\bbN}$ converges a.s. to a random variable $\alpha_\infty$, and $\sum_{k=1}^\infty\delta_k <\infty$ a.s. {on the set $\{\sum_{k=1}^\infty \beta_k<\infty, \,\sum_{k=1}^\infty \gamma_k<\infty \}$.}
\end{theorem}

Under certain conditions on $\xbs_0$, the limit is the MNS $\xref$. Below $\bbE_k$ denotes the conditional expectation with respect to the filtration $\CF_k$.

\begin{theorem}\label{thm:as_lin_convergence}
Let $(\mu_k)_{k\in\bbN}$ satisfy $\sum_{k=1}^\infty\mu_{k}=\infty$  and $\sum_{k=1}^\infty \mu_{k}^{p^\ast} <\infty,$
Assumption \ref{ass:space-basic} hold, and $\xbs^\dag$ be the MNS.
{Then the sequence $(\xbs_k)_{k\in\bbN}$ converges a.s. to a solution of \eqref{eqn:inv}:
\begin{align*}
    \bbP\Big(\lim_{k\rightarrow\infty} \inf_{\widetilde \xbs\in \CX_{\min}}\xN{\xbs_k-\widetilde \xbs}=0\Big)=1.
\end{align*}}
Moreover, if $\dmapX{p}(\xbs_0)\in\overline{\range(\VA^\ast)}$, 
we have $\lim_{k\rightarrow\infty}\bregman{\xbs_k}{\xref}=0$ a.s.
\end{theorem}
\begin{proof}
{By Lemma \ref{lem:Kaczmarz-basic}, we have 
$\DP{\partial\Psi(\xbs_k)}{\xbs_k-\xref} = p\Psi(\xbs_k).$}
Moreover, 
\begin{align*}
    \xsN{g(\xbs,\ybs,i)}=\xsN{\VA_i^\ast \svaldmapY{p}(\VA_i\xbs-\ybs_i)}\leq\|\VA_i\| \ysN{\svaldmapY{p}(\VA_i\xbs-\ybs_i)}\leq L_{\max} \yN{\VA_i\xbs-\ybs_i}^{p-1},
\end{align*}
with $L_{\max}=\max_{i\in[N]}\|\VA_i\|$. Thus, since $p^\ast(p-1)=p$, we have 
\begin{align*}
    \bbE\big[\xsN{g(\xbs,\ybs,i)}^{p^\ast}\big] \leq pL_{\max}^{p^\ast}\frac{1}{N}\sum_{i=1}^N\frac{1}{p} \yN{\VA_i\xbs-\ybs_i}^p= pL_{\max}^{p^\ast}\Psi(\xbs).
\end{align*}
Upon taking the conditional expectation $\mathbb{E}_k[\cdot]$ of the descent property \eqref{eqn:descent_property} (with $\widehat\xbs=\xref$), and using the measurability of $\xbs_k$ with respect to $\CF_k$, we deduce
\begin{align*}
\bbE_k[\Delta_{k+1}] &\leq \Delta_{k}-p\mu_{k+1}\Psi(\xbs_k)+pL_{\max}^{p^\ast}\frac{G_{p^\ast}}{p^\ast}\mu_{k+1}^{p^\ast}\Psi(\xbs_k).
\end{align*}
{Using Lemma \ref{lem:Kaczmarz-basic} again we have $\Psi(\xbs_k)\leq \frac{L_{\max}^p}{C_p}\Delta_k$,}
which yields
\begin{align*}
\bbE_k[\Delta_{k+1}] &\leq \LRR{1+L_{\max}^{p^\ast+p}\frac{p}{C_p}\frac{G_{p^\ast}}{p^\ast}\mu_{k+1}^{p^\ast}}\Delta_{k}-p\mu_{k+1}\Psi(\xbs_k).
\end{align*}
Since $\sum_{k=1}^\infty\mu_{k}^{p^\ast}<\infty$ by assumption, we can apply Theorem \ref{thm:sub_martingale_convergence}
and deduce that the sequence $(\Delta_k)_{k\in\bbN}$ converges a.s. to a random variable $\Delta_\infty$ and
 $\sum_{k=0}^\infty\mu_{k+1}\Psi(\xbs_k)<\infty$  a.s.
{Let $\Omega$ be the measurable set on which $(\Delta_k)_{k\in\bbN}$ converges, $\sum_{k=0}^\infty\mu_{k+1}\Psi(\xbs_k)<\infty$, and  $\bbP(\Omega)=1$.}
Next we show $\liminf_{k}  \Psi(\xbs_k) =0$ a.s. {Consider an event $\omega$ on which this is not the case, i.e. where} $\liminf_{k}  \Psi(\xbs_k) >0$.
Then there exist $\epsilon>0$ and $k_\epsilon\in\bbN$ such that for all $k\geq k_\epsilon$, $\Psi(\xbs_k)\geq \epsilon$, giving
$\sum_{k\geq k_\epsilon} \mu_{k+1} \Psi(\xbs_k) \geq \epsilon \sum_{k\geq k_\epsilon} \mu_{k+1}$. 
{Since for all events in $\Omega$ this would lead to a contradiction: the right hand side diverges ($\sum_{k=1}^\infty\mu_{k}=\infty$ by assumption), whereas the left hand side is the remainder of a convergent series, we conclude $\omega\not\in\Omega$. Since $\bbP(\Omega^c)=0$, we have $\liminf_k \Psi(\xbs_k)=0$ a.s.
For every event in the set where $\liminf_k \Psi(\xbs_k)=0$ holds we can then find a sub-sequence $(\xbs_{n_k})_{k\in\bbN}$ such that 
$\lim_{k\rightarrow\infty} \Psi(\xbs_{n_k})=0$.}
Define also $\widehat\Psi(\xbs)=\sum_{i=1}^N \widehat\Psi_i(\xbs)$, with $\widehat\Psi_i(\xbs)=\yN{\VA_i\xbs-\ybs_i}$.
We have $\liminf_k \widehat\Psi(\xbs_k)=0$ and $\lim_{j\rightarrow\infty} \widehat\Psi(\xbs_{n_j})=0$ (on the same subsequence), since by Young's inequality,
\begin{align*}
\Big(\sum_{i=1}^N \yN{\VA_i\xbs-\ybs_i}^p\Big)^{1/p}\leq\sum_{i=1}^N \yN{\VA_i\xbs-\ybs_i}\leq N\Big(\frac{1}{N}\sum_{i=1}^N \yN{\VA_i\xbs-\ybs_i}^p\Big)^{1/p}.
\end{align*}
Moreover, $\widehat\Psi(\xbs)^p \leq pN^p \Psi(\xbs)$.
{The following argument is understood pointwise on the a.s. set $\Omega$ where $(\Delta_k)_{k\in\bbN}$ converges, $\sum_{k=0}^\infty\mu_{k+1}\Psi(\xbs_k)<\infty$, and $\liminf_k \widehat\Psi(\xbs_k)=0$}.
Since $(\Delta_k)_{k\in\bbN}$ {converges it is bounded.} 
By the coercivity of the Bregman distance {(see Lemma \ref{lem:bregman_bound_xk_bound})} so are $(\xbs_k)_{k\in\bbN}$ and  $(\dmapX{p}(\xbs_k))_{k\in\bbN}$.
By further passing to a subsequence, we can find a subsequence of $(\xbs_{n_k})_{k\in\bbN}$, that we denote the same, such that
$(\xN{\xbs_{n_k}})_{k\in\bbN}$ is convergent, $(\dmapX{p}(\xbs_{n_k}))_{k\in\bbN}$ is weakly convergent, and 
\begin{align} \label{eqn:monotonic_subseq}
\lim_{k\rightarrow\infty} \widehat\Psi(\xbs_{n_k})=0\quad \text{and}\quad \widehat\Psi(\xbs_{n_k}) \leq \widehat\Psi(\xbs_{n}) \text{ for all } n <n_k.
\end{align}
The latter can be obtained by setting $n_1=1$, and then recursively defining  $n_{k+1}=\min\{k> n_k: \Psi(\xbs_k)\leq \Psi(\xbs_{n_k})/2 \}$, for $k\in\bbN$. Any following subsequence satisfies the same property.
Using {Theorem \ref{thm:bregman_properties}(ii),} 
we have for $k>\ell$
\begin{align*}
    \bregman{\xbs_{n_\ell}}{\xbs_{n_k}} &\!=\!\frac{1}{p^\ast} \!\Big(\!\xN{\xbs_{n_\ell}}^p\!-\!\xN{\xbs_{n_k}}^p\!\Big) \!+ \!\DP{\!\dmapX{p}(\xbs_{n_k}\!)\!-\!\dmapX{p}(\xbs_{n_\ell}\!)}{\xref\!}\!+\!\DP{\!\dmapX{p}(\xbs_{n_k}\!)\!-\!\dmapX{p}(\xbs_{n_\ell}\!)}{\xbs_{n_k}\!\!-\!\xref\!}.
\end{align*}
Since the first two terms involve Cauchy sequences, it suffices to treat the last term, denoted by ${\rm I}_{k,\ell}$. 
Using telescopic sum and applying the iterate update rule, we have
\begin{align*}
    {\rm I}_{k,\ell}&=\sum_{n=n_\ell}^{n_k-1} \!\DP{\dmapX{p}(\xbs_{n+1})\!-\!\dmapX{p}(\xbs_{n})}{\xbs_{n_k}\!-\!\xref}=\!\sum_{n=n_\ell}^{n_k-1}\! \mu_{n+1}\DP{\VA_{i_{n+1}}^\ast\svaldmapY{p}(\VA_{i_{k+1}}\xbs_n-\ybs_{i_{n+1}})}{\xbs_{n_k}-\xref}\\
        &=\sum_{n=n_\ell}^{n_k-1} \mu_{n+1}\DP{\svaldmapY{p}(\VA_{i_{n+1}}\xbs_n-\ybs_{i_{k+1}})}{\VA_{i_{n+1}}\xbs_{n_k}-\ybs_{i_{n+1}}}.
\end{align*}
By the Cauchy-Schwarz inequality and properties of the duality map, we get
\begin{align*}
|{\rm I}_{k,\ell}|&\leq\sum_{n=n_\ell}^{n_k-1}\!\! \mu_{n+1}\!\yN{\VA_{i_{n+1}}\!\xbs_n\!-\!\ybs_{i_{n+1}}}^{p-1}\yN{\VA_{i_{n+1}}\xbs_{n_k}\!-\!\ybs_{i_{n+1}}\!}
    \leq\!\sum_{n=n_\ell}^{n_k-1}\!\mu_{n+1}\widehat\Psi_{i_{n+1}}(\xbs_{n})^{p-1}\widehat\Psi_{i_{n+1}}(\xbs_{n_k}).
\end{align*}
Since $\widehat\Psi_i(\xbs)\leq \widehat\Psi(\xbs)$, for all $i\in[N]$, we use \eqref{eqn:monotonic_subseq} and get
\begin{align*}
    |{\rm I}_{k,\ell}|&\leq\sum_{n=n_\ell}^{n_k-1} \mu_{n+1}\widehat\Psi(\xbs_{n})^{p-1}\widehat\Psi(\xbs_{n_k})\leq\sum_{n=n_\ell}^{n_k-1} \mu_{n+1}\widehat\Psi(\xbs_{n})^{p}.
\end{align*}
Since $\widehat\Psi(\xbs)^{p}\le pN^p \Psi(\xbs)$, the right hand side of the inequality converges to $0$ as $n_\ell\to\infty$. 
Therefore, by \cite[Theorem 2.12(e)]{SLS_06}, it follows that $(\xbs_{n_k})_{k\in\bbN}$, is a Cauchy sequence, and thus converges strongly to an $\widehat\xbs$ such that $\Psi(\widehat\xbs)=0$.

{The above argument showing the a.s. convergence of
$(\Delta_k)_{k\in\bbN}$ can be applied pointwise to any solution.
Namely, on the event where $(\xbs_{n_k})_{k\in\bbN}$ converges strongly to an $\widehat \xbs\in\mathcal{X}_{\min}$ (i.e. $\VA\widehat\xbs=\ybs$), define $\widehat\Delta_k:=\bregman{\xbs_k}{\widehat\xbs}$. 
By repeating the argument using Lemma \ref{lem:Kaczmarz-basic}, we deduce
\begin{align*}
    \widehat\Delta_{k+1} \leq \LRR{1+L_{\max}^{p^\ast+p}\frac{p}{C_p}\frac{G_{p^\ast}}{p^\ast}\mu_{k+1}^{p^\ast}}\widehat\Delta_k - p\mu_{k+1}\Psi_{i_k}(\xbs_k).
\end{align*}
Since $\sum_{k=1}^\infty\mu_{k}^{p^\ast}<\infty$, it follows that the (deterministic) sequence $(\widehat\Delta_k)_{k\in\bbN}$ converges to a $\widehat\Delta_\infty\geq0$.
The continuity of the Bregman distance in the first argument (Theorem \ref{thm:bregman_properties}(vi)) gives
$\lim_{j\rightarrow\infty}\bregman{\xbs_{n_j}}{\widehat\xbs}=\bregman{\widehat\xbs}{\widehat\xbs}=0$,
and thus $\widehat\Delta_\infty=0$.
Moreover, by the $p$-convexity of $\CX$ (Theorem \ref{thm:bregman_properties}(iii)), we have 
$0\leq \xN{\xbs_k-\widehat\xbs}^p \leq \frac{p}{C_p} \widehat\Delta_k.$
From the squeeze theorem it follows that  $\lim_{k\rightarrow\infty}\xN{\xbs_k-\widehat\xbs}=0$.
Thus, for every event in an a.s. set $\Omega$, the sequence $(\xbs_k)_{k\in\bbN}$ strongly converge to some minimising solution, that is
\begin{align*}
    \bbP\Big(\lim_{k\rightarrow\infty} \inf_{\widetilde \xbs\in \CX_{\min}}\xN{\xbs_k-\widetilde \xbs}=0\Big)=1.
\end{align*}
}

Next assume $\dmapX{p}(\xbs_0)\in\overline{\mathrm{range}(\VA^\ast)}$.
From \eqref{eqn:sgd}, it follows that $\dmapX{p}(\xbs_k)\in\overline{\range(\VA^\ast)}$ holds for all $k\geq1$.
By the continuity of $\dmapX{p}$, we have $\dmapX{p}(\widehat\xbs)\in\overline{\range(\VA^\ast)}$.
Thus, from $\VA(\widehat\xbs-\xref)=0$ and Lemma \ref{lem:minnorm} it follows $\widehat\xbs=\xref$.
\end{proof}

The assumptions and conclusions of Theorem \ref{thm:as_lin_convergence} can be broken down into two parts.
The step-size conditions  $\sum_{k=1}^\infty\mu_{k}=\infty$ and $\sum_{k=1}^\infty \mu_{k}^{p^\ast} <\infty$ are required to show the a.s. convergence of $(\bregman{\xbs_k}{\widehat\xbs})_{k\in\bbN}$ to $0$, for some non-deterministic $\widehat\xbs\in \mathcal{X}_{\min}$.
The remaining assumption $\dmapX{p}(\xbs_0)\in\overline{\mathrm{range}(\VA^\ast)}$ is needed to identify this limit as the MNS $\xref$, as the Landweber method \cite[Remark 3.12]{SLS_06}.
If
$\dmapX{p}(\xbs_0)\not\in\overline{\mathrm{range}(\VA^\ast)}$, we commonly establish 
convergence to an MNS relative to $\xbs_0$, i.e. a solution 
which minimises $\xN{\xbs-\xbs_0}$, analogous to the Euclidean case \cite{JL19}.

\begin{remark}
The stepsize conditions
$\sum_{k=1}^\infty \mu_{k} =\infty$ and $\sum_{k=1}^\infty \mu_{k}^{p^\ast}<\infty$ are satisfied by a polynomially decaying step-size schedule $(\mu_{k})_{k\in\bbN}=(\mu_0k^{-\beta})_{k\in\bbN}$, with {$\frac{1}{p^\ast}<\beta\le1$}.
\end{remark}

Theorem \ref{thm:as_lin_convergence} states sufficient conditions ensuring the a.s. convergence of $(\bregman{\xbs_k}{\xref})_{k\in\bbN}$ to $0$. To strengthen this to the convergence 
in expectation, we require an additional assumption to ensure that $(\bregman{\xbs_k}{\xref})_{k\in\bbN}$ 
is a uniformly integrable super-martingale and the space $\CX$ being uniformly smooth. Note that removing the assumptions of Theorem 
\ref{thm:as_lin_convergence} from Theorem \ref{thm:L1_lin_convergence} 
would still result in convergence in expectation to 
some non-negative random variable, but not necessarily to $0$. {Recall that a family $(X_t)_t$ of random variables is uniformly integrable provided $\lim_{k\rightarrow\infty}\sup_{t}\bbE[\|X_t\| \mid  \Vone_{\|X_t\|\geq k}]=0$, where $\Vone(\cdot)$ is the indicator function.}
\begin{theorem}\label{thm:L1_lin_convergence}
Let the conditions of Theorem \ref{thm:as_lin_convergence} hold {with $\dmapX{p}(\xbs_0)\in\overline{\range(\VA^\ast)}$} and let
$\mu_{k}^{p^\ast-1}\leq \frac{p^\ast}{G_{p^\ast}L_{\max}^{p^\ast}}$ for all $k\in\bbN$.
Then there holds
$\lim_{k\rightarrow\infty}\bbE[\bregman{\xbs_k}{\xref}]=0.$
{Moreover, for $1\leq r\leq p$, we have $\lim_{k\rightarrow\infty} \bbE[\xN{\xbs_k-\xref}^r]=0$ and if $\CX$ is additionally uniformly smooth, then $\lim_{k\rightarrow\infty} \bbE[\xsN{\dmapX{p}(\xbs_k)-\dmapX{p}(\xref)}^{p^\ast}]=0$.}
\end{theorem}
\begin{proof}
{The step-size conditions allow applying Lemma \ref{lem:iterate_boundedness}, which yields $\bregman{\xbs_k}{\xref}\leq\bregman{\xbs_0}{\xref}$ for all $k$.}
{It follows that $(\bregman{\xbs_k}{\xref})_{k\in\bbN}$ is bounded, and is thus uniformly integrable, and by Theorem \ref{thm:as_lin_convergence}, it converges a.s. to $0$.}
Then, by Vitali's convergence  theorem \cite[Theorem 4.5.4]{B07}, we deduce that
$(\Delta_k)_{k\in\bbN}$ 
converges to $0$ in expectation as well.
Using now the $p$-convexity of $\CX$ and the monotonicity of expectation, we have
\begin{align*}
    0\leq\frac{C_p}{p}\lim_{k\rightarrow\infty}\bbE[\xN{\xbs_k-\xref}^p] \leq \lim_{k\rightarrow\infty}\bbE[\bregman{\xbs_k}{\xref}]=0.
\end{align*}
By the continuity of the power function and the Lyapunov inequality for $1\leq r {\leq p}$, we have
\begin{align*}
0\leq\lim_{k\rightarrow\infty}\bbE[\xN{\xbs_k-\xref}^r]\leq\lim_{k\rightarrow\infty}(\bbE[\xN{\xbs_k-\xref}^p])^{r/p}=0.
\end{align*}
{To prove the last claim we use uniform smoothness of $\CX$ and Theorem \ref{rem:dmap_properties}(iv), to deduce}
\begin{align*}
    \xsN{\dmapX{p}(\xbs_k)-\dmapX{p}(\xref)}^{p^\ast}\leq C \max\{1, \xN{\xbs_k}, \xN{\xref}\}^{p}\, {\overline{\rho}_\CX}(\xN{\xbs_k-\xref})^{p^\ast},
\end{align*}
where $\overline{\rho}_\CX(\tau)=\rho_\CX(\tau)/\tau$ is a modulus of smoothness function such that $\overline{\rho}(\tau)\leq1$ and $\lim_{\tau\rightarrow0}\overline{\rho}(\tau)=0$, cf. Definition \ref{defn:smoothness_and_convexity}.
{By Lemmas \ref{lem:iterate_boundedness} and \ref{lem:bregman_bound_xk_bound} 
$(\xN{\xbs_k}^p)_{k\in\bbN}$ is {(uniformly) bounded, giving that the sequence $(\xsN{\dmapX{p}(\xbs_k)-\dmapX{p}(\xref)}^{p^\ast})_{k\in\bbN}$ is bounded and thus} uniformly integrable.}
Since $\lim_{k\rightarrow\infty}\bbE[\xN{\xbs_k-\xref}]=0$, it follows that $\xN{\xbs_k-\xref}$ converges to $0$ in probability, and thus by the continuous mapping theorem  ${\overline{\rho}_\CX}(\xN{\xbs_k-\xref})^{p^\ast}$ also converges to $0$ in probability.
{Applying Vitaly's theorem to the uniformly integrable sequence $(\xsN{\dmapX{p}(\xbs_k)-\dmapX{p}(\xref)}^{p^\ast})_{k\in\bbN}$ yields that it converges to $0$ in measure,} and the claim follows.
\end{proof}

\begin{remark}
Note that the condition $\dmapX{p}(\xbs_0)\in\overline{\range(\VA^\ast)}$ on $\xbs_0$ is crucial for ensuring that all the limits are the same.
Landweber iterations converge for uniformly convex and smooth $\CX$, and any Banach 
space $\CY$ \cite[Theorem 3.3]{SLS_06}. In our analysis, we have assumed that $\CX$ is $p$-convex to simplify the analysis. First, $p$-convexity is used in the proof of Lemma  \ref{lem:descent_property}.
If $\CX$ were only uniformly convex (and $\CX^\ast$ only uniformly smooth), 
then we may use the modulus of smoothness function $\rho_{\mathcal{X}}$, cf. \eqref{defn:smoothness_and_convexity} 
and \cite[Theorem 2.41]{TBHK_12}, to establish a suitable analogue of the descent 
property \eqref{eqn:descent_property}. Second, $p$-convexity is used in the proof of Theorem \ref{thm:as_lin_convergence}, allowing a more direct application of Robbins-Siegmund theorem by relating the objective 
values to Bregman distances. Meanwhile, the Landweber method in \cite{SLS_06} requires 
step-sizes that depend on the modulus of smoothness, the current iterate and objective 
value, which is more restrictive than that in this work.
\end{remark}

\subsection{Convergence analysis for the generalised Kaczmarz model}\label{sec:further_kaczmarz}
Sch\"opfer et al \cite{SLS_06} studied general powers of 
the Banach space norm and sub-gradients of the form $\partial(\frac{1}{q}\yN{\VA\cdot-\ybs}^q)(\xbs)$.
Now we take an analogous perspective for the objective
\begin{equation*}
\Psi(\xbs)=\frac{1}{N}\sum_{i=1}^N \Psi_i(\xbs),\quad\mbox{with }\Psi_i(\xbs) := \frac{1}{q}\yN{\VA_i\xbs-\ybs_i}^q,
\end{equation*}  
with $1<q\leq 2$. {This model is herein called the generalised Kaczmarz model. (Note that this is different from the randomised extended Kaczmarz method \cite{Zouzias:2013}.)} We shall show the convergence of SGD with stochastic directions
\begin{align}\label{eqn:q_kaczmarz}
g(\xbs,\ybs,i)=\VA_i^\ast \svaldmapY{q}(\VA_i\xbs-\ybs_i)= \partial(\tfrac{1}{q}\yN{\VA_i\cdot-\ybs_i}^q)(\xbs).
\end{align}
The descent property \eqref{eqn:descent_property} is unaffected, and a direct computation again yields
\begin{equation}\label{eqn:descent_property-2}
\bregman{\xbs_{k+1}}{\xref}\leq\bregman{\xbs_k}{\xref}-\mu_{k+1}\DP{\gbs_{k+1}}{\xbs_k-\xref} + \frac{G_{p^\ast}}{p^\ast}\mu_{k+1}^{p^\ast}\xsN{\gbs_{k+1}}^{p^\ast}.
\end{equation}
However, Robbins-Siegmund theorem cannot be applied directly.
Instead, we pursue a different proof strategy by first establishing the uniform boundedness of iterates.

\begin{lemma}\label{lem:bounded_q_kaczmarz}
Let Assumption \ref{ass:space-basic} hold.
Consider SGD with descent directions \eqref{eqn:q_kaczmarz} for $1<q\leq 2$, and assume that $\mu_{k}^{p^\ast-1}<\frac{p^\ast}{G_{p^\ast}L_{\max}^{p^\ast}}$ holds for all $k\in\bbN$ and $\sum_{k=1}^\infty\mu_{k}^{p^\ast}=:\Gamma<\infty$.
Then $(\bregman{\xbs_k}{\xref})_{k\in\bbN}$ and $(\xbs_k)_{k\in\bbN}$ are uniformly bounded.
\end{lemma}
\begin{proof}
Let $\overline{\Psi}_i(\xbs)=\yN{\VA_i\xbs-\ybs_i}^q$, and $\Delta_k=\bregman{\xbs_k}{\xref}$. Then we have
$\DP{\gbs_{k+1}}{\xbs_k-\xref}=\overline{\Psi}_{i_{k+1}}(\xbs_k)$ and
\begin{align*}
\xsN{\gbs_{k+1}}^{p^\ast}&=\xsN{\VA_{i_{k+1}}^\ast \svaldmapY{q}(\VA_{i_{k+1}}\xbs-\ybs_{i_{k+1}})}^{p^\ast} \leq L_{\max}^{p^\ast}\yN{\VA_{i_{k+1}}\xbs-\ybs_{i_{k+1}}}^{p^\ast (q-1)} \\&\leq L_{\max}^{p^\ast}\overline{\Psi}_{i_{k+1}}(\xbs_k)^{p^\ast\frac{q-1}{q}}=L_{\max}^{p^\ast}\overline{\Psi}_{i_{k+1}}(\xbs_k)^{\frac{p^\ast}{q^\ast}},
\end{align*}
where $ q^\ast\geq2$ is the conjugate exponent of $q$.
Plugging this into  \eqref{eqn:descent_property-2} gives
\begin{align}\label{eqn:descent_for_q_kaczmarz}
\Delta_{k+1}\leq\Delta_k-\mu_{k+1}\overline{\Psi}_{i_{k+1}}(\xbs_k) + L_{\max}^{p^\ast}\frac{G_{p^\ast}}{p^\ast}\mu_{k+1}^{p^\ast}\overline{\Psi}_{i_{k+1}}(\xbs_k)^{\frac{p^\ast}{q^\ast}}.
\end{align}
Since $1<p^\ast\leq 2$ by Theorem \ref{thm:bregman_properties}(iii), and $q^\ast\geq2$, we have $\frac{p^\ast}{q^\ast}\leq1$.
Now we define two sets of indices
\[ \CI = \{j\leq k: \overline{\Psi}_{i_{j+1}}(\xbs_j) \geq 1\} \text{ and } \CJ = \{j\leq k: \overline{\Psi}_{i_{j+1}}(\xbs_j) <1\},\]
so that $\CI\cap\CJ=\emptyset$, and $\CI\cup\CJ=[k]$. Note that $\CI$ and $\CJ$ actually depend on the current iterate index $k$. Applying the inductive argument to \eqref{eqn:descent_for_q_kaczmarz} gives
\begin{align*}
    \Delta_{k+1} &\leq \Delta_0 -\sum_{j=0}^k\mu_{j+1}\overline{\Psi}_{i_{j+1}}(\xbs_j) + L_{\max}^{p^\ast}\frac{G_{p^\ast}}{p^\ast}\sum_{j=0}^k\mu_{j+1}^{p^\ast}\overline{\Psi}_{i_{j+1}}(\xbs_j)^{\frac{p^\ast}{q^\ast}}\\
    &=\Delta_0 \underbrace{-\sum_{j\in\CI}\mu_{j+1}\overline{\Psi}_{i_{j+1}}(\xbs_j) + L_{\max}^{p^\ast}\frac{G_{p^\ast}}{p^\ast}\sum_{j\in\CI}\mu_{j+1}^{p^\ast}\overline{\Psi}_{i_{j+1}}(\xbs_j)^{\frac{p^\ast}{q^\ast}}}_{(\star)}\underbrace{-\sum_{j\in\CJ}\mu_{j+1}\overline{\Psi}_{i_{j+1}}(\xbs_j)}_{(\star\star)}\\
    &\qquad\qquad\qquad\qquad\qquad\quad+\underbrace{L_{\max}^{p^\ast}\frac{G_{p^\ast}}{p^\ast}\sum_{j\in\CJ}\mu_{j+1}^{p^\ast}\overline{\Psi}_{i_{j+1}}(\xbs_j)^{\frac{p^\ast}{q^\ast}}}_{(\star\star\star)}.
\end{align*}
Next we analyse these three terms separately. First, for $j\in\CI$, we have $\overline{\Psi}_{i_{j+1}}(\xbs_j)\geq 1$ and since $\frac{p^\ast}{q^\ast}<1$, we have $\overline{\Psi}_{i_{j+1}}(\xbs_j)^{\frac{p^\ast}{q^\ast}}\leq\overline{\Psi}_{i_{j+1}}(\xbs_j)$, giving
\begin{align*}
(\star)&\!\leq\!-\!\sum_{j\in\CI}\mu_{j+1}\!\overline{\Psi}_{i_{j+1}}(\xbs_j)\!+\! L_{\max}^{p^\ast}\!\frac{G_{p^\ast}}{p^\ast}\!\sum_{j\in\CI}\mu_{j+1}^{p^\ast}\overline{\Psi}_{i_{j+1}}(\xbs_j)\\
&=\!-\!\sum_{j\in\CI}\!\Big(\!1\!-\!L_{\max}^{p^\ast}\frac{G_{p^\ast}}{p^\ast}\mu_{j+1}^{p^\ast-1}\!\Big)\!\mu_{j+1}\!\overline{\Psi}_{i_{j+1}}(\xbs_j).
\end{align*}
Since $\mu_{j+1}^{p^\ast-1}<\frac{p^\ast}{G_{p^\ast}L_{\max}^{p^\ast}}$ holds by assumption,
the term $(\star)$ is non-positive. Moreover, $(\star\star)$ is trivially non-positive.
Since $\overline{\Psi}_{i_{j+1}}(\xbs_j) <1$ for $j\in\CJ$, the last term $(\star\star\star)$ can be bounded as
\begin{align*}
    L_{\max}^{p^\ast}\frac{G_{p^\ast}}{p^\ast}\sum_{j\in\CJ}\mu_{j+1}^{p^\ast}\overline{\Psi}_{i_{j+1}}(\xbs_j)^{\frac{p^\ast}{q^\ast}}\leq L_{\max}^{p^\ast}\frac{G_{p^\ast}}{p^\ast}\sum_{j\in\CJ}\mu_{j+1}^{p^\ast}\leq L_{\max}^{p^\ast}\frac{G_{p^\ast}}{p^\ast}\sum_{j=1}^\infty\mu_{j}^{p^\ast}=L_{\max}^{p^\ast}\frac{G_{p^\ast}}{p^\ast}\Gamma.
\end{align*}
By combining the last three bounds on $(\star)$, $(\star\star)$ and $(\star\star\star)$, we get
\[\Delta_{k+1}\leq \Delta_0+L_{\max}^{p^\ast}\frac{G_{p^\ast}}{p^\ast}\Gamma, \text{ for all } k\geq 0.\]
Thus, $(\Delta_k)_{k\in\bbN}$ is uniformly bounded and by Lemma \ref{lem:bregman_bound_xk_bound}, so is $(\xbs_k)_{k\in\bbN}$.
\end{proof}

The proof of Lemma \ref{lem:bounded_q_kaczmarz} exposes the challenge in extending the convergence 
results to general stochastic directions. Namely, in the proof of Theorem 
\ref{thm:as_lin_convergence}, we showed the convergence by taking conditional expectation of  \eqref{eqn:descent_property}, recasting the resulting expression
as an almost super-martingale, and then relating objective values to Bregman distances via $\Psi(\xbs_k)\leq C\Delta_k$, for some $C>0$.
Here, using $\frac{q}{q^\ast}=q-1$ and $\frac{p^\ast}{p}=p^\ast-1$, we instead have 
\begin{equation*}
    \Psi(\xbs_k)^{\frac{p^\ast}{q^\ast}}\leq C\Delta_k^{(p^\ast-1)(q-1)}, \quad \mbox{with } C=q^{-\frac{p^\ast}{q^\ast}}L_{\max}^{p^\ast(q-1)}\Big(\frac{p}{C_p}\Big)^{(p^\ast-1)(q-1)},
\end{equation*}
which gives
\begin{align*}
    \bbE_k[\Delta_{k+1}]\leq \Delta_k + CL_{\max}^{p^\ast}q^\frac{p^\ast}{q^\ast}\frac{G_{p^\ast}}{p^{\ast}}\mu_{k+1}^{p^\ast}\Delta_k^{(p^\ast-1)(q-1)}   -q\mu_{k+1}\Psi(\xbs_k).
\end{align*}
Here $0<(p^\ast-1)(q-1)<1$, provided $p^\ast\neq2$ and $q\neq2$.
Therefore, Robbins-Siegmund theorem cannot be applied directly. Nonetheless, we still have the following analogue of Theorem \ref{thm:L1_lin_convergence}.

\begin{theorem}\label{thm:L1_q_Kaczmarz_convergence}
Consider iterations \eqref{eqn:sgd} with descent directions \eqref{eqn:q_kaczmarz} for $1<q\leq 2$ and let Assumption \ref{ass:space-basic} hold and $\xref$ be the MNS.
Let the step-sizes $(\mu_k)_{k\in\bbN}$ satisfy $\sum_{k=1}^\infty\mu_{k}=\infty$, $\sum_{k=1}^\infty\mu_{k}^{p^\ast}<\infty$, and $\mu_{k}^{p^\ast-1}<\frac{p^\ast}{G_{p^\ast}L_{\max}^{p^\ast}}$ for all $k\in\bbN$.
{Then the sequence $(\xbs_k)_{k\in\bbN}$ converges a.s. to a solution of \eqref{eqn:inv}:
\begin{align*}
    \bbP\Big(\lim_{k\rightarrow\infty} \inf_{\widetilde \xbs\in \CX_{\min}}\xN{\xbs_k-\widetilde \xbs}=0\Big)=1.
\end{align*}}
Moreover, if $\dmapX{p}(\xbs_0)\in\overline{\mathrm{range}(\VA^\ast)}$, we have
$$\lim_{k\rightarrow\infty}\bregman{\xbs_k}{\xref}=0\ \mbox{ a.s.}\quad\mbox{and}\quad \lim_{k\rightarrow\infty}\bbE[\bregman{\xbs_k}{\xref}]=0.$$
\end{theorem}
\begin{proof}
To establish the a.s. convergence of iterates, we first take the conditional
expectation of the descent property \eqref{eqn:descent_property-2} and obtain
\begin{equation}\label{eqn:descent-proper-22}
\bbE_k[\Delta_{k+1}]\leq\Delta_k-\mu_{k+1}\DP{\bbE_k[\gbs_{k+1}]}{\xbs_k-\xref} + \frac{G_{p^\ast}}{p^\ast}\mu_{k+1}^{p^\ast}\bbE_k\big[\xsN{\gbs_{k+1}}^{p^\ast}\big].
\end{equation}
We now have
$\DP{\bbE_k[\gbs_{k+1}]}{\xbs_k-\xref}=\DP{\partial\Psi(\xbs_k)}{\xbs_k-\xref}=q\Psi(\xbs_k)$,
and 
\begin{equation*}
    \xsN{g(\xbs,\ybs,i)}\leq L_{\max} \yN{\VA_i\xbs-\ybs_i}^{q-1}.
\end{equation*}
Then taking the conditional expectation of $\xsN{g(\xbs,\ybs,i)}^{p^\ast}$ yields
\begin{align*}
    \bbE\big[\xsN{g(\xbs,\ybs,i)}^{p^\ast}\big] &\leq L_{\max}^{p^\ast} \bbE\Big[\yN{\VA_i\xbs-\ybs_i}^{p^\ast(q-1)}\Big]=L_{\max}^{p^\ast} \bbE\Big[(\yN{\VA_i\xbs-\ybs_i}^{q})^{\frac{p^\ast}{q^\ast}}\Big].
\end{align*}
We have $0<\frac{p^\ast}{q^\ast}\leq 1$, with the equality achieved only if $p^\ast=q^\ast=2$.
In the latter case, it trivially follows that $\bbE[\xsN{g(\xbs,\ybs,i)}^{p^\ast}] \leq qL_{\max}^{p^\ast}\Psi(\xbs)$.
If $0<\frac{p^\ast}{q^\ast}<1$, by Jensen's inequality, we have
\begin{align*}
    \bbE\big[\xsN{g(\xbs,\ybs,i)}^{p^\ast}\big] &\leq L_{\max}^{p^\ast} \bbE\Big[(\yN{\VA_i\xbs-\ybs_i}^{q})^{\frac{p^\ast}{q^\ast}}\Big]\leq L_{\max}^{p^\ast} (\bbE[\yN{\VA_i\xbs-\ybs_i}^{q}])^{\frac{p^\ast}{q^\ast}}\leq L_{\max}^{p^\ast} q^{\frac{p^\ast}{q^\ast}} \Psi(\xbs)^{\frac{p^\ast}{q^\ast}}.
\end{align*}
Plugging this estimate into the conditional descent property \eqref{eqn:descent-proper-22} yields
\begin{align*}
    \bbE_k[\Delta_{k+1}]\leq \Delta_k -q\mu_{k+1}\Psi(\xbs_k) + L_{\max}^{p^\ast}q^\frac{p^\ast}{q^\ast}\frac{G_{p^\ast}}{p^\ast}\mu_{k+1}^{p^\ast} \Psi(\xbs_k)^{\frac{p^\ast}{q^\ast}}.
\end{align*}
Since the sequence $(\xbs_k)_{k\in\bbN}$ is uniformly bounded by Lemma \ref{lem:bounded_q_kaczmarz}, so is $(\Psi(\xbs_k))_{k\in\bbN}$, and we thus have
\[
\sum_{k=0}^\infty\mu_{k+1}^{p^\ast} \Psi(\xbs_k)^{\frac{p^\ast}{q^\ast}} \leq C \sum_{k=0}^\infty\mu_{k+1}^{p^\ast}<\infty.
\]
Thus, we can apply Robbins-Siegmund theorem for almost super-martingales, and deduce 
that $(\Delta_k)_{k\in\bbN}$ converges a.s. to a non-negative random variable $\Delta_\infty$.
Moreover, $\sum_{k=0}^\infty \mu_{k+1}\Psi(\xbs_k)<\infty$ holds a.s. By repeating the argument for Theorem \ref{thm:as_lin_convergence}, there exists a subsequence 
$(\xbs_{k_j})_{j\in\bbN}$ that a.s. converges to some $\widehat\xbs\in\CX_{\min}$, 
and hence $\Delta_\infty=0$, as desired. Moreover, by Lemma \ref{lem:bounded_q_kaczmarz}, 
the sequence $(\Delta_k)_{k\in\bbN}$ is bounded, and thus uniformly integrable.
Since it converges to $0$ a.s., from Vitali's theorem it follows that $\lim_{k\rightarrow
\infty}\bbE[\bregman{\xbs_k}{\xref}]=0$.
\end{proof}

The results in Theorem \ref{thm:L1_q_Kaczmarz_convergence} are similar to that of Theorem \ref{thm:L1_lin_convergence}, but the generality of the latter is compensated for by an additional step-size assumption ensuring boundedness of iterates $(\xbs_k)_{k\in\bbN}$.

\subsection{Convergence rates for conditionally stable operators}

Theorem \ref{thm:L1_lin_convergence} states the conditions needed for the convergence of Bregman distances in expectation.
However, it does not provide convergence rates. In order to obtain convergence rates, one needs additional conditions on the MNS $\xref$, which are collectively known as source conditions. One approach is via conditional stability:
for a locally conditionally stable operator, we can extract convergence in expectation and quantify the convergence speed.
Conditional stability is known for many inverse problems for PDEs, and has been used extensively to investigate regularised solutions \cite{ChengYamamoto:2000,EggerHofmann:2018}.
It is useful for analysing ill-posed problems that are locally well-posed, and in case of a (possibly) non-linear forward operator $F$ it is of the form
\begin{align}\label{eqn:cond_stab_measure}
\xN{\xbs_1-\xbs_2}\leq \Phi(\yN{F(\xbs_1)-F(\xbs_2)}),\quad \forall \xbs_1,\xbs_2\in\CM\subset\CX,
\end{align}
where $\Phi:[0,\infty)\rightarrow[0,\infty)$ with $\Phi(0)=0$ is a continuous, non-decreasing function, and $\CM$ is typically a ball in the ambient norm \cite{H94}.
In Banach space settings, the conditional stability needs to be adjusted, by replacing the left hand side of \eqref{eqn:cond_stab_measure} with a non-negative error measure \cite{CHL14}. Since the most relevant error measure for Banach space analysis is the Bregman distance $\bregman{\xbs_1}{\xbs_2}$, a H\"{o}lder type stability estimate then reads: for some $\alpha\ge1$ and $C_\alpha>0$
\begin{equation}
\bregman{\xbs}{\xref}^\alpha\le C_\alpha^{-1} \yN{\VA\xbs-\VA\xref}^p.\label{eqn:cond-stab}
\end{equation}

Now we give a convergence rate under conditional stability bound \eqref{eqn:cond-stab}. The constant $C_N$ appears in Lemma \ref{lem:Kaczmarz-basic} and denotes the norm equivalence constant.
\begin{theorem}\label{thm:lin-main}
Let the forward operator $\VA$ satisfy the conditional stability bound
\eqref{eqn:cond-stab} for some $\alpha\ge1$ and $C_\alpha>0$.
Let $\dmapX{p}(\xbs_0)\in\overline{\range(\VA^\ast)}$, and for $C_k=C_NC_\alpha(1 - L_{\max}^{p^\ast}\frac{G_{p^\ast}}{p^\ast}\mu_{k}^{p^\ast-1})>0$, the step-sizes satisfy $\sum_{k=1}^\infty \mu_{k} C_k=\infty$.
Then there holds
\begin{equation*}
    \lim_{k\rightarrow\infty}\bbE[\bregman{\xbs_{k}}{\xref}]=0.
\end{equation*}
Moreover,
\[
\bbE[\bregman{\xbs_{k}}{\xref}]\leq\left\{\begin{aligned}
   \frac{\bregman{\xbs_{0}}{\xref}}{\Big(1+(\alpha-1)\bregman{\xbs_{0}}{\xref}^{\alpha-1}\sum_{j=1}^k \mu_{j}C_j\Big)^{\frac{1}{\alpha-1}}}, & \quad \text{ if } \alpha>1,\\
   {\exp\Big(-\sum_{j=1}^k \mu_{j}C_j\Big)\bregman{\xbs_{0}}{\xref},} &\quad\text{ if } \alpha=1. 
\end{aligned}\right.
\]
\end{theorem}
\begin{proof}
Let $\Delta_k:=\bregman{\xbs_k}{\xref}$. The proof of Theorem \ref{thm:as_lin_convergence} and
the conditional stability bound \eqref{eqn:cond-stab} imply
\begin{align}\label{eqn:decreasing}
\bbE_k[\Delta_{k+1}] &\leq \Delta_{k}-p\mu_{k+1}\LRR{1-L_{\max}^{p^\ast}\frac{G_{p^\ast}}{p^\ast}\mu_{k+1}^{p^\ast-1}}\Psi(\xbs_k)\\
&\leq\Delta_k-p\mu_{k+1}\frac{C_NC_\alpha}{p}\Big(1 - L_{\max}^{p^\ast}\frac{G_{p^\ast}}{p^\ast}\mu_{k+1}^{p^\ast-1}\Big) \Delta_k^{\alpha},\nonumber
\end{align}
{since by Lemma \ref{lem:Kaczmarz-basic}, there exists a $C_N>0$ such that
$\Psi(\xbs)\geq \frac{C_N}{p}\yN{\VA\xbs-\ybs}^p$.}
Taking the full expectation and using Jensen's inequality lead to
\[ \bbE[\Delta_{k+1}] \leq \bbE[\Delta_k] -\mu_{k+1}C_{k+1} \bbE[\Delta_k]^\alpha.\]
Since $C_{k+1}>0$ by assumption, $(\bbE[\Delta_k])_{k\in\bbN}$ is a monotonically decreasing sequence.
By the convexity of the function $x\mapsto x^\alpha$ (for $\alpha\geq1$), for any $\epsilon>0$ and $x\ge\epsilon$, we have $\epsilon^\alpha\geq \frac{\epsilon}{x} x^\alpha$. We claim that for every $\epsilon>0$, there exists a $k_\epsilon\in\bbN$ such that $\bbE[\Delta_k]\le\epsilon$ for all $k\ge k_\epsilon$.
Assuming the contrary, $\bbE[\Delta_k]\ge\epsilon$ for all $k$, gives
\begin{align*}\bbE[\Delta_{k+1}]\leq \bbE[\Delta_k]-\mu_{k+1}C_{k+1}\bbE[\Delta_k]^{\alpha}\le \bbE[\Delta_k]-\mu_{k+1}C_{k+1}\epsilon^\alpha\le\Delta_0-\epsilon^\alpha\sum_{j=1}^{k+1} \mu_{j}C_j\rightarrow -\infty, \end{align*}
since $\sum_{j=1}^\infty \mu_{j}C_j=\infty$ by assumption, which is a contradiction.
Therefore, $\lim_{k\rightarrow\infty}\bbE[\Delta_k]=0$.
For $\alpha>1$, by Polyak's inequality (cf. Lemma \ref{lem:polyak_series}), we have
\begin{align*}
\bbE[\Delta_{k+1}]\leq \frac{\Delta_0}{\Big(1+(\alpha-1)\Delta_0^{\alpha-1} \sum_{j=1}^{k+1} \mu_{j} C_j\Big)^{\frac{1}{\alpha-1}}}.
\end{align*}
Meanwhile, for $\alpha=1$, {using the inequality $1-x\leq e^{-x}$ for $x\geq0$}, a direct computation yields
\begin{align*}
\bbE[\Delta_{k+1}] \leq (1-\mu_{k+1}C_{k+1})\bbE[\Delta_{k}]\leq \prod_{j=1}^{k+1} (1-\mu_jC_j) \Delta_0 \leq{ \exp\Big(-\sum_{j=1}^{k+1} \mu_jC_j\Big)\Delta_0,}
\end{align*}
completing the proof of the theorem.
\end{proof}

\begin{remark}
We have the following comments on Theorem \ref{thm:lin-main}.
\begin{itemize}
\item[(i)] {The estimates  for $\alpha>1$ and $\alpha=1$ in Theorem \ref{thm:lin-main} are consistent in the sense that
\begin{equation*}
   \lim_{\alpha\searrow 1} 
   \frac{\bregman{\xbs_{0}}{\xref}}{\Big(1+(\alpha-1)\bregman{\xbs_{0}}{\xref}^{\alpha-1}\sum_{j=1}^k \mu_{j}C_j\Big)^{\frac{1}{\alpha-1}}} = \exp\Big(-\sum_{j=1}^k \mu_{j}C_j\Big)\bregman{\xbs_{0}}{\xref}.
\end{equation*}}
\item[(ii)] {While it might seem counter-intuitive, $\alpha=1$ gives a better convergence rate than $\alpha>1$, because of the following
\begin{align*}
\bregman{\xbs}{\xref}^\alpha\ge\bregman{\xbs}{\xref}^{\tilde\alpha} \text{ if and only if } \alpha\log\bregman{\xbs}{\xref} \ge \tilde\alpha\log\bregman{\xbs}{\xref}.
\end{align*}
Hence, whenever $\bregman{\xbs}{\xref}<1$, we have $\bregman{\xbs}{\xref}\ge\bregman{\xbs}{\xref}^{\alpha}$ for $\alpha>1$. 
Plugging this into the conditional stability bound \eqref{eqn:cond-stab} yields
\begin{align*}
\bregman{\xbs}{\xref}^\alpha\le\bregman{\xbs}{\xref}\le C_1^{-1} \yN{\VA\xbs-\VA\xref}^p=C_1^{-1}pN\Psi(\xbs).
\end{align*}
Meanwhile, the proof of Theorem \ref{thm:lin-main} uses the conditional stability bound to establish a relationship between the objective value and the Bregman distance to the MNS $\xbs^\dag$, cf. \eqref{eqn:decreasing}.
Putting these together gives that $\alpha=1$ provides a greater decrease of the expected Bregman distance, once we are close enough to the solution.} 
\end{itemize}
\end{remark}

The conditional stability estimate \eqref{eqn:cond-stab} for a linear operator $\VA$ {implies} its injectivity. Then the objective $\Psi(\xbs)$ is strongly convex.
Under condition \eqref{eqn:cond-stab}, there can indeed be only one solution: if $\VA\tilde\xbs=\VA\xbs$, then $\bregman{\tilde\xbs}{\xbs}=0$ follows from \eqref{eqn:cond-stab}.
The step-size condition $\sum_{k=1}^\infty \mu_{k} C_k=\infty$ is weaker than that in Theorem \ref{thm:L1_lin_convergence}.
Namely, it follows from step-size conditions in Theorem \ref{thm:as_lin_convergence}, since
$$\sum_{k=1}^\infty \mu_{k} C_k=C_NC_\alpha \Big(\sum_{k=1}^\infty\mu_k- L_{\max}^{p^\ast}\frac{G_{p^\ast}}{p^\ast}\sum_{k=1}^\infty\mu_k^{p^\ast}\Big)=\infty$$ 
holds if $\sum_{k=1}^\infty \mu_{k}=\infty$ and $\sum_{k=1}^\infty\mu_k^{p^\ast}<\infty$.
Further, if there exists a $C>0$ such that $1 - L_{\max}^{p^\ast}\frac{G_{p^\ast}}{p^\ast}\mu_{k}^{p^\ast-1}>C$ holds for all $k\in\bbN$, e.g. if $\mu_k$ is a constant satisfying this condition, then $\sum_{k=1}^\infty \mu_kC_k=\infty$ is weaker than the conditions in Theorem \ref{thm:as_lin_convergence}, since the condition $\sum_{k=1}^\infty\mu_k^{p^\ast}<\infty$ is no longer needed for convergence, and $\sum_{k=1}^\infty\mu_k=\infty$ suffices. {Moreover,
we can choose constant step-sizes. Indeed, setting $\mu_k=\mu_0$, with $1-L_{\max}^{p^\ast}\frac{G_{p^\ast}}{p^\ast}\mu_0^{p^\ast-1}=\frac{1}{2}$, 
we get an exponential convergence rate for $\alpha=1$, since $C_k=\frac{C_NC_\alpha}{2}$, we have
\begin{align*}
\bbE[\Delta_{k+1}] &\leq (1-\mu_0C_{k+1})\bbE[\Delta_{k}]\leq \bigg(1-2^{-1-1/(p^\ast-1)}L_{\max}^{-p^\ast/p^\ast-1}\Big(\frac{p^\ast}{G_{p^\ast}}C_NC_\alpha\Big)^{1/p^\ast-1}\bigg)^k\bbE[\Delta_{0}]\\
&\leq \bigg(1-2^{-p}L_{\max}^{-p}\Big(\frac{p^\ast}{G_{p^\ast}}\Big)^{p^\ast/p}C_NC_\alpha\bigg)^k\Delta_{0}.
\end{align*}
Note that this convergence rate is largely comparable with that in the Hilbert case: the conditional stability bound implies the strict convexity of the quadratic objective $\Psi(\xbs)$, and the SGD is known to converge exponentially fast (see e.g. \cite[Theorem 3.1]{Gower:2019}), with the rate determined by a variant of the condition number.}

\begin{remark}
The conditional stability bound \eqref{eqn:cond-stab} is stated globally. However,
such conditions are often valid only locally. A local definition could have been 
employed in \eqref{eqn:cond-stab}, with minor modifications of the argument. Indeed, by the argument of Theorem 
\ref{thm:L1_lin_convergence}, we appeal to Lemma \ref{lem:iterate_boundedness}, 
showing that the Bregman distances of the iterates are non-increasing. Thus, it
suffices to assume that the initial point $\xbs_0$ is sufficiently close to the MNS $\xbs^\dag$.
\end{remark}
\begin{remark}
Conditional stability is intimately tied with classical source conditions.
For example, as shown in \cite{SG+09}, assuming $\alpha=1$ in \eqref{eqn:cond-stab} allows to show a variational inequality
\[ \DP{\dmapX{p}(\xref)}{\xbs-\xref}\leq \xN{\xref}^{p-1}C_\alpha^{-1}(pC_p^{-1})^{1/p}\yN{\VA(\xbs-\xref)}.\]
Then Hahn-Banach theorem and \cite[Lemma 8.21]{SG+09} give the canonical range type condition $\dmapX{p}(\xref)=\VA^\ast \wbs$, for $\wbs\in\CX$ such that $\xN{\wbs}\leq1$.
Connections between source conditions and conditional stability estimates have been studied, e.g. for linear operators in Hilbert spaces \cite{T+13} and in $\CL^p$ spaces \cite{CY21}.
Moreover, variational source conditions often imply conditional stability estimates \cite{WH17}, and in case of bijective and continuous operators they are trivially inferred by a standard source condition {\rm(}albeit only in a possibly small neighbourhood around the solution{\rm)}.
See the book \cite{W19} about the connections between source conditions and conditional stability estimates, and \cite{I17} for inverse problems for differential equations.
\end{remark}


\section{Regularising property}\label{sec:regularisation}

In practice, we often do not have access to the exact data $\ybs$ but only to noisy 
observations $\ybs^\delta$, such that $\yN{\ybs^\delta-\ybs}{\leq}\delta$. The convergence study 
in the presence of observational noise requires a different approach, since the sequence of objective 
values $(\yN{\VA\xbs_k^\delta-\ybs^\delta}^p)_{k\in\bbN}$ generally will not converge to $0$. In this section 
we show that SGD has a regularising effect, in the sense that the expected error 
$\bbE[\bregman{\xbs_{k(\delta)}^{\delta}}{\xref}]$ converges to $0$ as the noise level $\delta$ 
decays to $0$, for properly selected stopping indices $k(\delta)$.

Let $(\xbs_k)_{k\in\bbN}$ and $(\xbs_k^\delta)_{k\in\bbN}$ be the 
noiseless and noisy iterates, defined respectively by 
\begin{align}
\xbs_{k+1} & = \dmapXs{p}\LRR{\dmapX{p}(\xbs_k) - \mu_{k+1}\gbs_{k+1}},\quad \mbox{with }
\gbs_{k+1} = g(\xbs_{k},\ybs,i_{k+1}),\label{eqn:sgd_cleaniterates}\\
\xbs_{k+1}^\delta &= \dmapXs{p}\LRR{\dmapX{p}(\xbs_k^\delta) - \mu_{k+1}\gbs^\delta_{k+1}},
\quad \mbox{with }\gbs_{k+1}^\delta = g(\xbs_{k}^\delta,\ybs^\delta, i_{k+1}).\label{eqn:sgd_noisyiterates}
\end{align}

The key step in proving the regularising property is to show the stability
of SGD iterates with respect to noise. The noise enters into the iterations 
through the update directions $\gbs^\delta_{k+1}$ and thus, the stability of the 
iterates requires that of update directions. This however requires imposing suitable assumptions 
on the observation space $\CY$ since in general, the single valued duality maps $\svaldmapY{p}$ are 
continuous only at $0$. If $\CY$ is uniformly smooth, the corresponding duality maps 
are also smooth. This assumption is also needed for deterministic iterates, cf. 
\cite[Proposition 6.17]{TBHK_12} or \cite[Lemma 9]{M18}. Thus we make the following assumption. 
\begin{assumption}\label{ass:smooth-Y}
The Banach space $\CX$ is $p$-convex and uniformly smooth, and $\CY$ is uniformly smooth.
\end{assumption}

We then have the following stability result on the iterates with respect to noise, whose elementary but lengthy proof is deferred to the appendix.
\begin{lemma}\label{lem:coupled_noise_convergence}
Let Assumption \ref{ass:smooth-Y} hold.
Consider the iterations \eqref{eqn:sgd_cleaniterates} and \eqref{eqn:sgd_noisyiterates} with 
the same initialisation $\xbs_0^\delta=\xbs_0$, and following the same path {\rm(}i.e.  
using same random indices $i_{k}${\rm)}. Then, for any fixed $k\in\bbN$, we have
\begin{align*}
    \lim_{\delta\searrow0}\bbE[\bregman{\xbs^{\delta}_k}{\xbs_k}]=\lim_{\delta\searrow0}\bbE[\xN{\xbs^{\delta}_k -\xbs_k}]=\lim_{\delta\searrow0}\bbE[\xsN{\dmapX{p}(\xbs_k^\delta)-\dmapX{p}(\xbs_k)}]=0.
\end{align*}
\end{lemma}

Now we show the regularising property of SGD for suitable stopping indices $k(\delta)$.
\begin{theorem}\label{thm:regularisation_property}
Let Assumption \ref{ass:smooth-Y} hold, and the step-sizes $(\mu_k)_{k\in\bbN}$ satisfy $\sum_{k=1}^\infty\mu_{k}=\infty$, $\sum_{k=1}^\infty \mu_{k}^{p^\ast} <\infty$ and $1 - L_{\max}^{p^\ast}\frac{G_{p^\ast}}{p^\ast}\mu_{k}^{p^\ast-1}>C>0$. If $\lim_{\delta\searrow0}k(\delta)=\infty$ and $\lim_{\delta\searrow0} \delta^p\sum_{\ell=1}^{k(\delta)}\mu_\ell=0$, then
\begin{align*}
    \lim_{\delta\searrow0} \bbE[\bregman{\xbs_{k(\delta)}^{\delta}}{\xref}]=0.
\end{align*}
\end{theorem}
\begin{proof}
Let $\Delta_k = \bregman{\xbs_{k}}{\xref}$ and $\Delta_k^\delta = \bregman{\xbs_{k}^\delta}{\xref}$. Take any $\delta>0$ and $k\in\bbN$.
By the three point identity \eqref{eqn:3_point_id}, we have
\begin{align}\label{eqn:noisy_3point}
    \Delta_k^\delta &= \bregman{\xbs_k^\delta}{\xbs_k} + \Delta_k  + \DP{\dmapX{p}(\xbs_k)-\dmapX{p}(\xbs_k^\delta)}{\xbs_k-\xref}\nonumber \\
    &\leq \bregman{\xbs_k^\delta}{\xbs_k} + \Delta_k + \xsN{\dmapX{p}(\xbs_k)-\dmapX{p}(\xbs_k^\delta)}\xN{\xbs_k-\xref}.
\end{align}
{Consider a sequence $(\delta_j)_{j\in\bbN}$ decaying to zero.
Taking any $\epsilon>0$, it suffices to find a $j_\epsilon\in\bbN$ such that for all $j\geq j_\epsilon$ we have $\bbE[\Delta_{k(\delta_j)}^{\delta_j}]\leq 4\epsilon$.}
By Theorem \ref{thm:L1_lin_convergence},
there exists a $k_\epsilon\in\bbN$ such that for all $k\geq k_\epsilon$ we have
\begin{align}\label{eqn:noiseless_kbound}
    \bbE[\Delta_k]<\epsilon\quad \text{and}\quad \bbE[\xN{\xbs_k-\xref}]<\epsilon^{1/2}.
\end{align}
Moreover, for any fixed $k_\epsilon$, by Lemma \ref{lem:coupled_noise_convergence},
there exists $j_1\in\bbN$ such that for all $j\geq j_1$ we have
\begin{align}\label{eqn:noisy_kepsilon_bound}
    \bbE[\bregman{\xbs_{k_\epsilon}^{\delta_j}}{\xbs_{k_\epsilon}}] <\epsilon \quad \text{and}\quad \bbE[\xsN{\dmapX{p}(\xbs_{k_\epsilon})-\dmapX{p}(\xbs_{k_\epsilon}^{\delta_j})}]<\epsilon^{1/2}.
\end{align}
Thus, plugging the estimates \eqref{eqn:noiseless_kbound} and \eqref{eqn:noisy_kepsilon_bound} into \eqref{eqn:noisy_3point}, we have $\bbE[\Delta_{k_\epsilon}^{\delta_j}] < 3\epsilon$, for all $j\geq j_1$.
Note, however, that the same does not necessarily hold for all $k\geq k_\epsilon$, and thus for a monotonically increasing sequence of stopping indices $k(\delta_j)$, since  $\bbE[\Delta_{k(\delta_j)}^{\delta_j}]$ are not necessarily monotone.
Instead, taking the expectation of the descent property \eqref{eqn:descent_property} with respect to $\CF_k$ yields
\begin{align*}
    \bbE_k[\Delta_{k+1}^\delta] \leq\Delta_k^\delta -\mu_{k+1}\DP{\bbE_k[\gbs_{k+1}^\delta]}{\xbs_k^\delta-\xref} +pL_{\max}^{p^\ast}\frac{G_{p^\ast}}{p^\ast}\mu_{k+1}^{p^\ast} \Psi(\xbs_k^\delta).
\end{align*}
Then we decompose the middle term into
\begin{align*}
    \DP{\bbE_k[\gbs_{k+1}^\delta]}{\xref-\xbs_k^\delta}
    &=\frac{1}{N}\sum_{i=1}^N\DP{\svaldmapY{p}(\VA_{i}\xbs_k^\delta-\ybs_i^\delta)}{-(\VA_i\xbs_k^\delta-\ybs_i^\delta)+\ybs_i-\ybs_i^\delta} \\
    &=-p\Psi(\xbs_k^\delta) +\frac{1}{N}\sum_{i=1}^N\DP{\svaldmapY{p}(\VA_{i}\xbs_k^\delta-\ybs_i^\delta)}{\ybs_i-\ybs_i^\delta}\\
    &\leq -p\Psi(\xbs_k^\delta) +\frac{1}{N}\sum_{i=1}^N\yN{\VA_{i}\xbs_k^\delta-\ybs_i^\delta}^{p-1}\yN{\ybs_i-\ybs_i^\delta}\\
    &\leq -p\Psi(\xbs_k^\delta) +\delta\frac{1}{N}\sum_{i=1}^N\yN{\VA_{i}\xbs_k^\delta-\ybs_i^\delta}^{p-1},
\end{align*}
where we have used \eqref{eqn:dmap_defined} and the Cauchy-Schwarz inequality.
Taking the full expectation gives
\begin{align*}
    \bbE[\Delta^\delta_{k+1}]&\leq \bbE[\Delta^\delta_{k}] \!- \!p\mu_{k+1} \bbE[\Psi(\xbs_k^\delta)]\!+\!pL_{\max}^{p^\ast}\frac{G_{p^\ast}}{p^\ast}\mu_{k+1}^{p^\ast}\bbE[\Psi(\xbs_k^\delta)] \!+\! \delta\mu_{k+1} \frac{1}{N}\sum_{i=1}^N\bbE[\yN{\VA_{i}\xbs_k^\delta-\ybs_i^\delta}^{p-1}] \\
    &= \bbE[\Delta^\delta_{k}] - p\mu_{k+1}C_{k+1}\bbE[\Psi(\xbs_k^\delta)]+ \delta\mu_{k+1}\frac{1}{N}\sum_{i=1}^N\bbE[\yN{\VA_{i}\xbs_k^\delta-\ybs_i^\delta}^{p-1}]  ,
\end{align*}
where $C_k\!=\!1-L_{\max}^{p^\ast}\frac{G_{p^\ast}}{p^\ast}\mu_k^{p^\ast-1}>C>0$. Now using the Lyapunov inequality
\begin{align*}
    \frac{1}{N}\sum_{i=1}^N\bbE[\yN{\VA_{i}\xbs_k^\delta-\ybs_i^\delta}^{p-1}]\leq\frac{1}{N}\sum_{i=1}^N\Big(\bbE[\yN{\VA_{i}\xbs_k^\delta-\ybs_i^\delta}^{p}]\Big)^{(p-1)/p}=p^{1/p^\ast}\frac{1}{N}\sum_{i=1}^N\Big(\bbE[\Psi_i(\xbs_k^\delta)]\Big)^{1/p^\ast},
\end{align*}
we deduce
\begin{align}\label{eqn:Delta-new}
    \bbE[\Delta^\delta_{k+1}]&\leq \bbE[\Delta^\delta_{k}] - p\mu_{k+1}C_{k+1}\bbE[\Psi(\xbs_k^\delta)]+ \delta\mu_{k+1}p^{1/p^\ast}\frac{1}{N}\sum_{i=1}^N\Big(\bbE[\Psi_i(\xbs_k^\delta)]\Big)^{1/p^\ast}.
\end{align}
Next we remove the exponent in the last term. Using Young's inequality $ab\leq \frac{a^p}{p}\omega^{-p}+\frac{b^{p^\ast}}{p^\ast}\omega^{p^\ast}$, with $a=\delta$ and $b=\bbE[\Psi_i(\xbs_k^\delta)]^{1/p^\ast}$, we have
\begin{align*}
    \frac{1}{N}\sum_{i=1}^N\delta\Big(\bbE[\Psi_i(\xbs_k^\delta)]\Big)^{1/p^\ast} \leq \delta^p\frac{\omega^{-p}}{p} + \bbE\Big[\frac{1}{N}\sum_{i=1}^N\Psi_i(\xbs_k^\delta)\Big]\frac{\omega^{p^\ast}}{p^\ast}\leq\delta^p\frac{\omega^{-p}}{p} + \bbE[\Psi(\xbs_k^\delta)]\frac{\omega^{p^\ast}}{p^\ast}.
\end{align*}
Plugging this back in \eqref{eqn:Delta-new} gives
\begin{align*}
\bbE[\Delta^\delta_{k+1}] &\leq  \bbE[\Delta^\delta_{k}] - p\mu_{k+1}C_{k+1}\bbE[\Psi(\xbs_k^\delta)]+ p^{1/p^\ast}(p^\ast)^{-1}\omega^{p^\ast}\mu_{k+1}\bbE[\Psi(\xbs_k^\delta)]+  p^{-1/p}\delta^p\omega^{-p}\mu_{k+1}.
\end{align*}
Taking $\omega>0$ small enough so that $\omega^{p^\ast}\leq p^\ast p^{1/p}C_k$ (which can be made uniformly on $k$, thanks to the positive lower bound on $C_k$), replacing $k+1$ with $k(\delta)$ and using the inductive argument, we have
\begin{align*}
    \bbE[\Delta^\delta_{k(\delta)}]&\leq \bbE[\Delta^\delta_{k(\delta)-1}] +p^{-1/p}\omega^{-p}\delta^p\mu_{k(\delta)}
    \leq\bbE[\Delta^\delta_{k_\epsilon}] +p^{-1/p}\omega^{-p}\delta^p\sum_{\ell=1}^{k(\delta)}\mu_{\ell}.
\end{align*}
Since $\lim_{\delta\searrow0}\delta^p\sum_{\ell=1}^{k(\delta)}\mu_{\ell}=0$ and $\lim_{\delta\searrow0}k(\delta)=\infty$, there exists $j_2\in\bbN$ such that for all $j\geq j_2$ we have $k(\delta_j)\geq k_\epsilon$ and $p^{-1/p}\omega^{-p}\delta_j^p\sum_{\ell=1}^{k(\delta_j)}\mu_{\ell}<\epsilon$.
{Taking ${j_\epsilon}=j_1\vee j_2$ shows $\bbE[\Delta_{k(\delta_j)}^{\delta_j}] < 4\epsilon$ for all $j\geq j_\epsilon$, and hence the desired claim follows.}
\end{proof}
\begin{remark}
In the constant step-size regime, such as in the case of conditionally stable operators, the correspondence between the noise level and the step-size regime takes a more standard form.
Namely, the condition in Theorem \ref{thm:regularisation_property} reduces to
$\lim_{\delta\searrow0} \delta^p k(\delta) = 0$. In other words, we have $k(\delta)=\CO(\delta^{-p})$, mirroring the traditional conditions in Euclidean spaces. {Note that the condition on $k(\delta)$ is fairly broad, and does not give useful concrete stopping rules directly. Generally, the issue of a posterior stopping rules for stochastic iterative methods is completely open, even for the Hilbert setting \cite{JahnJin:2020}. For a polynomially decaying step-sizes $\mu_k=c_0k^{-\beta}$, the conditions $\frac{1}{p^*}<\beta\leq1$ and $c_0<(\frac{P^\ast}{L_{\max}^{p^\ast}G_{p^\ast}})^{\frac{1}{p^\ast-1}}$ give a valid step-size choice, and the stopping index $k(\delta)$ should satisfy $\lim_{\delta\searrow0}k(\delta)=\infty$ and $\lim_{\delta\searrow0} k(\delta)\delta^\frac{p}{1-\beta}=0$.} \end{remark}

\begin{remark}
{It is of much interest to derive a convergence rate for noisy data under a conditional stability condition as in Theorem \ref{thm:lin-main}, as a natural extension of the regularising property. However, this is still unavailable. Within the current analysis strategy, deriving the rate would require quantitative versions of stability estimates in Lemma \ref{lem:coupled_noise_convergence} in terms of $\delta$ and $k$. Generally the convergence rate analysis for iterative regularisation methods in Banach space remains a very challenging task, and much more work is still needed.} 
\end{remark}


\section{Numerical experiments}\label{sec:experiments}

We present numerical results on two sets of experiments to illustrate distinct features of the SGD \eqref{eqn:sgd}. The first set of experiments deals with an integral operator and the reconstruction of a sparse signal in the presence of either Gaussian or impulse noise. On this model example, we investigate the impact of the number of batches and the choice of the spaces $\CX$ and $\CY$ on the performance of the algorithm. 
{To simplify the study we investigate spaces $\CX$ and $\CY$ that are smooth and convex of power type, and thus the corresponding duality maps are singletons.}
{To facilitate a direct comparison of the SGD with the Landweber method, we count the computational complexity with respect to the number of epochs, i.e. the size $N_b$ of partition defined below.
Note moreover that our implementation of the Landweber method does not use the stepsizes described in \cite[Method 3.1]{SLS_06}, since the latter requires knowledge of quantities that are inconvenient to compute in practice.} 
The second set of experiments is about tomographic reconstruction, with respect to different types of noise. {All the shown reconstructions are obtained with a single stochastic run, as is often done in practice, and the stopping index is determined in a trial and error manner so that the corresponding reconstruction yields small errors.}

\subsection{Model linear inverse problem}
First we consider the following model inverse problem studied in \cite{JinStals:2012}. Let $\kappa:\overline{\Omega}\times\overline{\Omega}\rightarrow\bbR^+$, with $\Omega=(0,1)$, be a continuous function, and define an integral operator $\CT_\kappa:\CL^{r_{\CX}}(\Omega)\rightarrow\CL^{r_{\CY}}(\Omega)$, for $1<r_{\CX},r_{\CY}<\infty$,  by
\begin{align}
    (\CT_\kappa x)(t) = \int_\Omega \kappa(t,s) x(s) ds.
\end{align}
This is a compact linear operator between $\CL^{r_{\CX}}(\Omega)$ and $\CL^{r_{\CY}}(\Omega)$, with the adjoint $\CT_\kappa^\ast\!:\!\CL^{r^\ast_{\CY}}(\Omega)\!\rightarrow\!\CL^{r^\ast_{\CX}}(\Omega)$ given by
$(\CT_\kappa^\ast y)(s)\! =\! \int_\Omega \kappa(t,s) y(t) dt$.
To approximate the integrals, we subdivide the interval $\overline\Omega$ into $N\!=\!1000$ subintervals $[\frac{k}{N}, \frac{k+1}{N}]$, for $k\!=\!0,\!\ldots\!,N\!-\!1$, and then use quadrature, giving a finite-dimensional model $\VA\xbs\!=\!\ybs$, with $\VA \!=\! \frac{1}{N}\!\left(\kappa\!\left(\frac{j-1}{N}, \frac{2k-1}{N}\right) \!\right)_{j,k=1}^N$ and $\xbs\!=\!\left(x\left(\frac{2j-1}{2N}\right)\right)_{j=1}^N$.
For SGD we use $N_b\in[N]$ mini-batches. To obtain equisized batches, we assume that $N_b$ divides $N$.
The mini-batch matrices $\VA_j$ are then constructed by taking every $N_b$-th row of $\VA$, shifted by $j$, resulting in well-balanced mini-batches, in the sense that the norm $\|\VA_j\|$ is (nearly) independent of $j$.

The kernel function $k(t,s)$ and the exact signal $x^\dag$ are defined respectively by
\begin{align*}
    \kappa(t,s) = \begin{cases}40t(1-s), &\text{ if } t\leq s,\\
    40s(1-t), &\text{ otherwise},\end{cases}
\quad\mbox{and} \quad   x^\dag(s) = \begin{cases}1, &\text{ if } s\in[\frac{9}{40},\frac{11}{40}]\cup[\frac{29}{40},\frac{31}{40}],\\
                2, &\text{ if } s\in[\frac{19}{40},\frac{21}{40}],\\
                0, &\text{ otherwise}.
    \end{cases}
\end{align*}
This is a sparse signal and we expect sparsity promoting norms to perform well.
To illustrate this, we compare the following four settings: (a) $\CX=\CY=\CL^2(\Omega)$; (b) $\CX=\CL^2(\Omega)$ and $\CY=\CL^{1.1}(\Omega)$; (c) $\CX=\CL^{1.5}(\Omega)$ and $\CY=\CL^{2}(\Omega)$; (d) $\CX=\CL^{1.1}(\Omega)$ and $\CY=\CL^{2}(\Omega)$.
Setting (a) is the standard Hilbert space setting, suitable for recovering smooth solutions from measurement data with i.i.d. Gaussian noise, whereas settings (b)-(d) use Banach spaces.
Settings (c) and (d) both aim at sparse solutions, and we expect the latter to yield sparser solutions, since spaces $\CL^{r}(\Omega)$ progressively enforce sparser solutions as the exponent $r$ gets closer to $1$.
In the experiments, we employ the step-size schedule  
$\mu_k = \frac{L_{\max}}{1+0.05 (k/N_b)^{1/p^\ast+0.01}},$
with $L_{\max} = \max_{j\in[N_b]} \|\VA_j\|$.
{This satisfies the summability conditions $\sum_{k=1}^\infty \mu_k=\infty$ and $\sum_{k=1}^\infty \mu_k^{p^*}<\infty$ required by Theorem \ref{thm:as_lin_convergence}}.
The operator norm $\|\VA_j\|=\|\VA_j\|_{\CL^{r_{\CX}}\rightarrow \CL^{r_{\CY}}}=\max_{\xbs\neq 0} \frac{\|\VA_j\xbs\|_{\CL^{r_{\CY}}}}{\|\xbs\|_{\CL^{r_{\CX}}}}$ is estimated using Boyd's power method \cite{B74}.
All the reconstruction algorithms are initialised with a zero vector.

In Fig. \ref{fig:sparse_solution_comparison}, we compare the reconstructions with settings (a)-(d) for exact data.
We observe from Fig. \ref{fig:sparse_solution_comparison}(a) that settings (a) and (b), with $\CX=\CL^2(\Omega)$, result in smooth solutions that fail to capture the sparsity structure of the true signal $\xbs^\dag$. In contrast, the choice $\CX=\CL^{1.5}(\Omega)$ recovers a sparser solution, and the choice $\CX=\CL^{1.1}(\Omega)$ gives a truly sparse reconstruction, but with peaks that overshoot the magnitude of $\xbs^\dag$.
This might be related to the fact $\xbs^\dag$ exhibits a cluster structure in addition to sparsity, which is not accounted for in the choice of the space $\CX=\CL^{1.1}(\Omega)$ \cite{ZouHastie:2005,JinLorenzSchiffler:2009}. Fig. \ref{fig:sparse_solution_comparison}(b) indicates that early stopping would result in lower peaks and significantly reduce the overshooting, but a more explicit form of regularisation \cite{ZouHastie:2005,JinLorenzSchiffler:2009} might allow faster convergence.

\begin{figure}[h!]
\centering
\setlength{\tabcolsep}{0pt}
\begin{tabular}{cc}
\includegraphics[width=.488\textwidth]{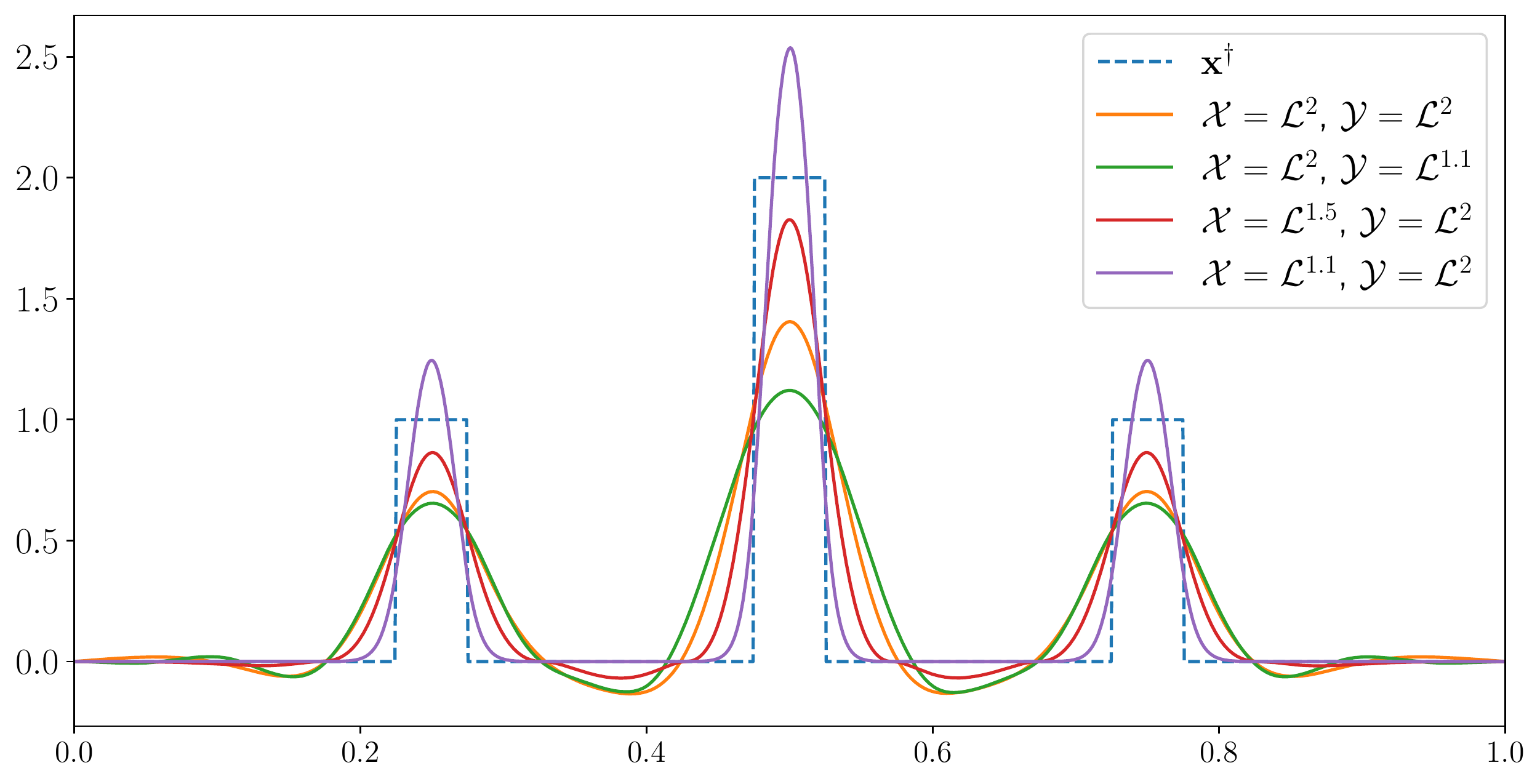} &
\includegraphics[width=.488\textwidth]{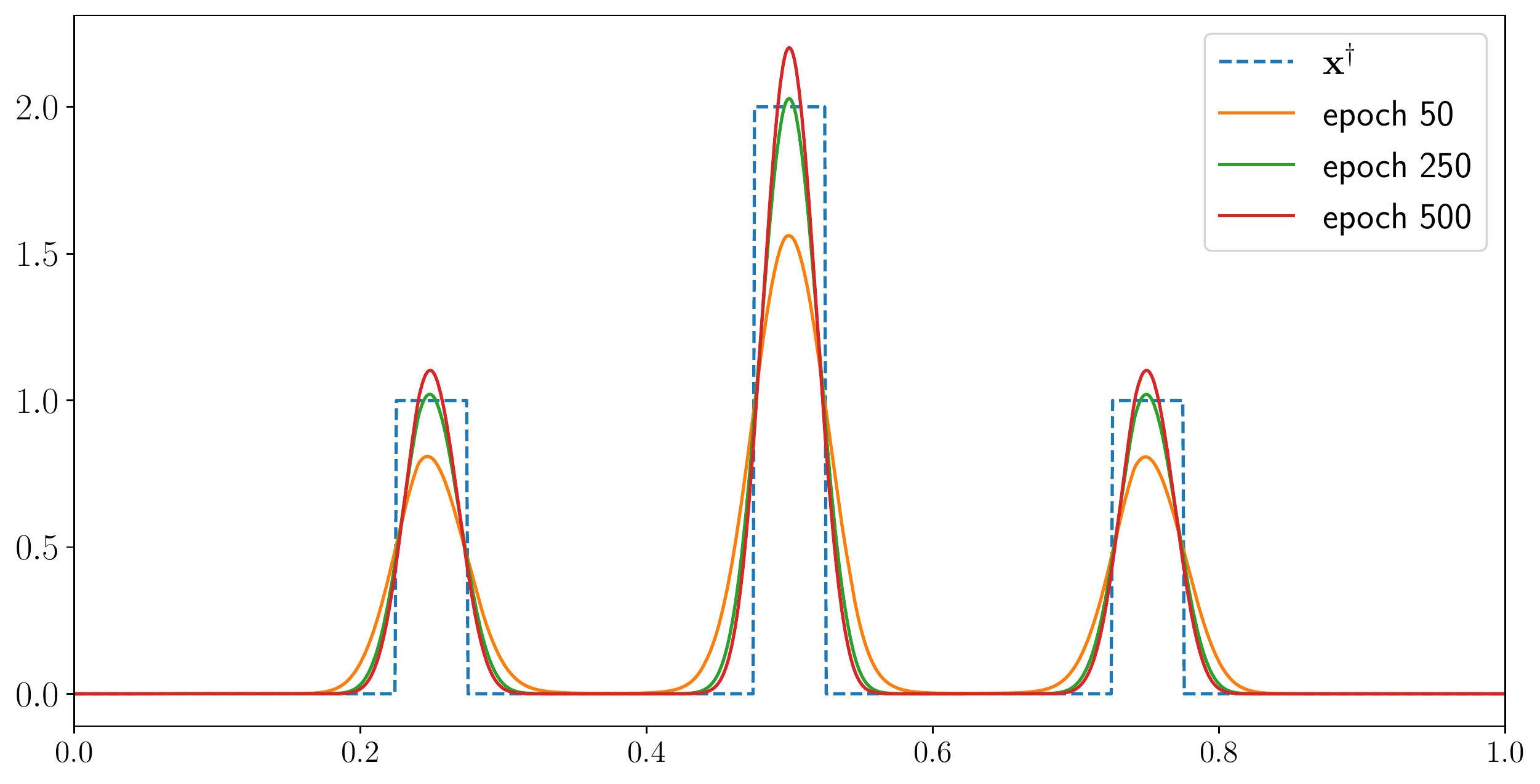}\\
{\scriptsize{(a) Changing $\CX$ and $\CY$ for $N_b=100$ }}& {\scriptsize{(b) Progression of iterates for $\CX=\CL^{1.1}(\Omega)$}}
\end{tabular}
\caption{Comparison of reconstructed solutions after $500$ epochs.}
\label{fig:sparse_solution_comparison}
\end{figure}

In Fig. \ref{fig:objective_function}, we investigate the convergence of the objective value with respect to the number of batches $N_b$ and the choice of the solution space $\CX$.
As expected, having a larger number of batches results in a faster initial convergence, but also in increased variance, as shown by the oscillations.
Moreover, the variance is lower in the case of a smoother space $\CX$ (promoting smoother solutions), where the variance existing in early epochs is dramatically reduced later on.
This observation can be explained by the gradient expression $g(\xbs,\ybs,i)=\VA_i^\ast \svaldmapY{p}(\VA_i\xbs-\ybs_i)$, which tends to zero as SGD iterates converge to the true solution $\xref$ and so does its variance, and the larger is the exponent $p$, the faster is the convergence.

\begin{figure}[h!]
\centering
\setlength{\tabcolsep}{0pt}
\begin{tabular}{cc}
\includegraphics[width=.488\textwidth]{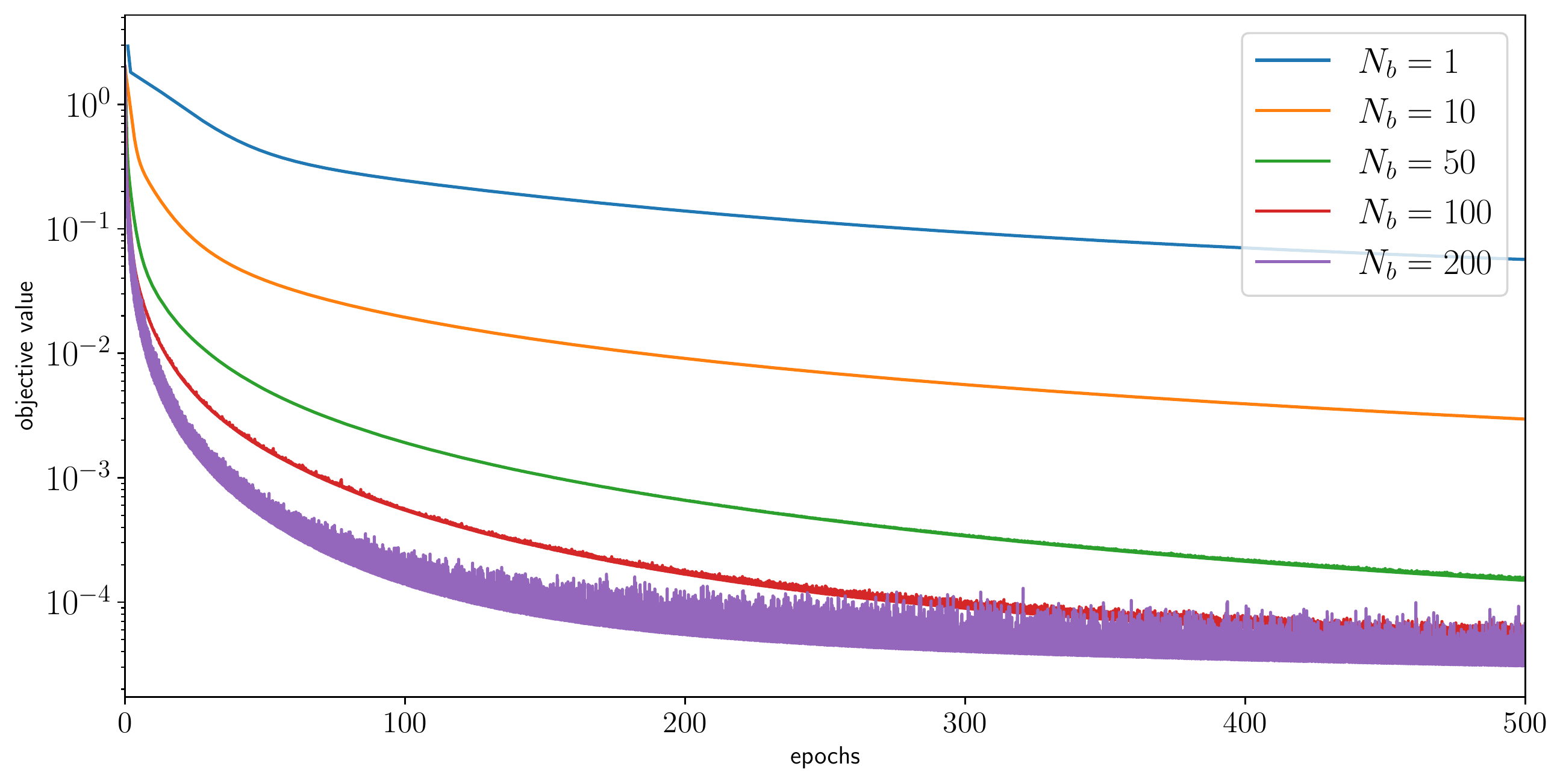}
& \includegraphics[width=.488\textwidth]{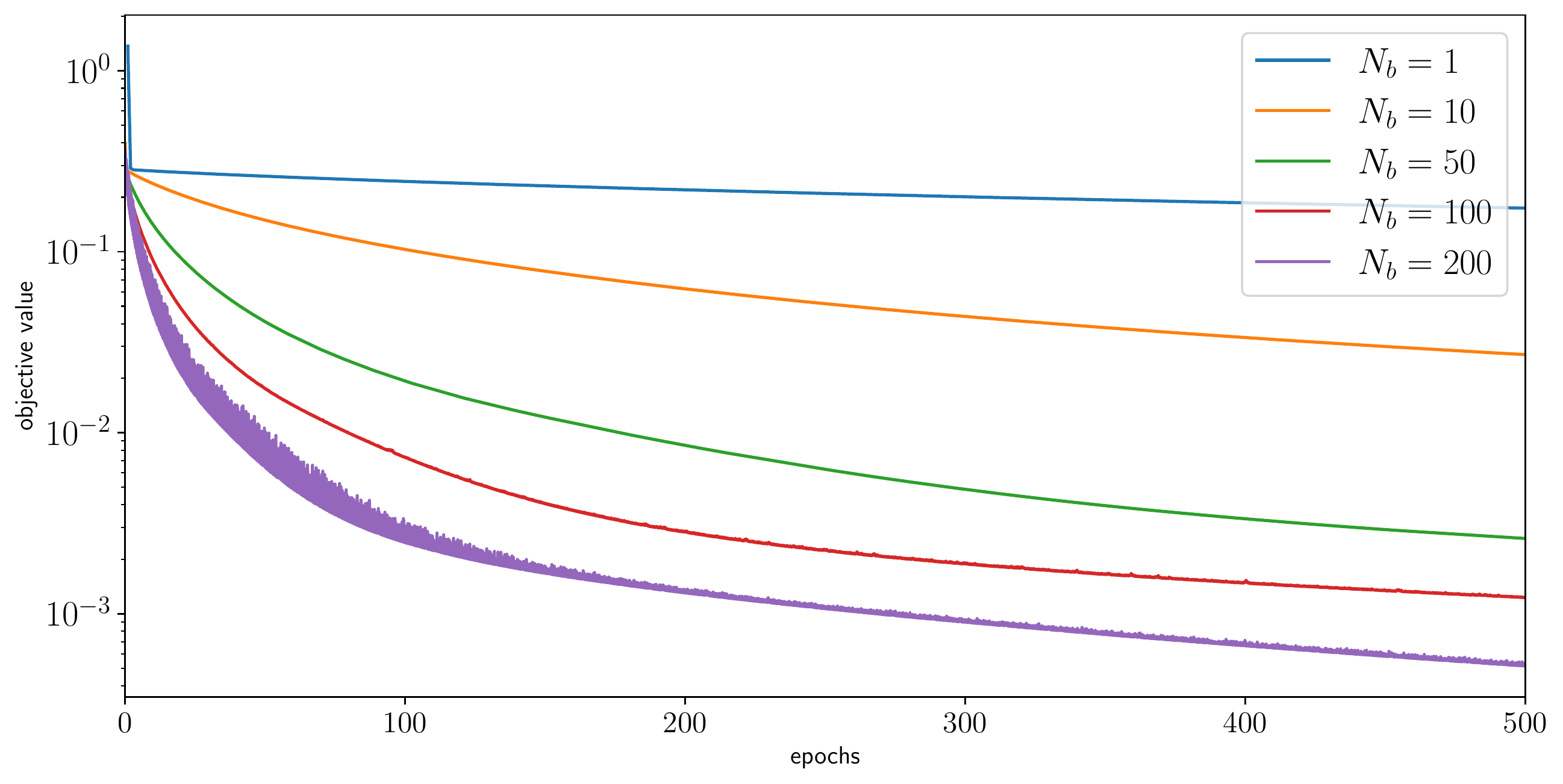} \\
{\scriptsize{(a) $\CX=\CL^{1.1}(\Omega)$ and $\CY=\CL^2(\Omega)$ }}& {\scriptsize{(b) $\CX=\CL^{1.5}(\Omega)$ and $\CY=\CL^2(\Omega)$}}
\end{tabular}
\caption{The variation of $\frac{1}{p}\sum_{i=1}^{N_b}\yN{\VA_i\xbs_k-\ybs_i}^p$ with respect to the number of batches $N_b$.}
\label{fig:objective_function}
\end{figure}

Next we examine the performance of the algorithm when the observational data $\ybs^\delta$ contains (random-valued) impulse noise, cf. Fig. \ref{fig:noisy_reconstruction}, {which is generated by 
\begin{align*}
y^\delta_i = \left\{\begin{aligned} 
 y_i^\dag, & \quad\mbox{with probability } 1-p,\\
 (1-\xi)y_i^\dag,    & \quad\mbox{with probability } p/2,\\
 1.4\xi +(1-\xi)y_i^\dag,    & \quad\mbox{with probability } p/2,
  \end{aligned}\right.
\end{align*}
where $p\in (0,1)$ denotes the percentage of corruption (which is set to $0.05$ in the experiment) and $\xi\sim{\rm Uni}(0.1, 0.4)$ follows a uniform distribution over the interval $(0.1,0.4)$.
It is known that $\mathcal{L}^r(\Omega)$ fittings with $r$ close to 1 is suitable for impulsive noise.}
This allows investigating the role of not only the space $\CX$ but also $\CY$.
The results in Fig. \ref{fig:noisy_reconstruction}(b) show that the choice $\CY=\CL^{r_{\CY}}(\Omega)$, with $r_{\CY}$ close to $1$, performs significantly better.
Indeed, the Hilbert setting $\CX=\CY=\CL^2(\Omega)$ produces overly smooth, non-sparse solutions with pronounced artefacts. In sharp contrast, setting $\CX=\CY=\CL^{1.1}(\Omega)$ yields solutions that can correctly identify the sparsity structure of the true solution, and have no artefacts.
Similar as before, the reconstruction in this setting overestimates the signal magnitude on its support, which is exacerbated as the exponent $r_{\CY}$ gets closer to $1$.

\begin{figure}[h!]
\centering
\setlength{\tabcolsep}{0pt}
\begin{tabular}{cc}
\includegraphics[width=.488\textwidth]{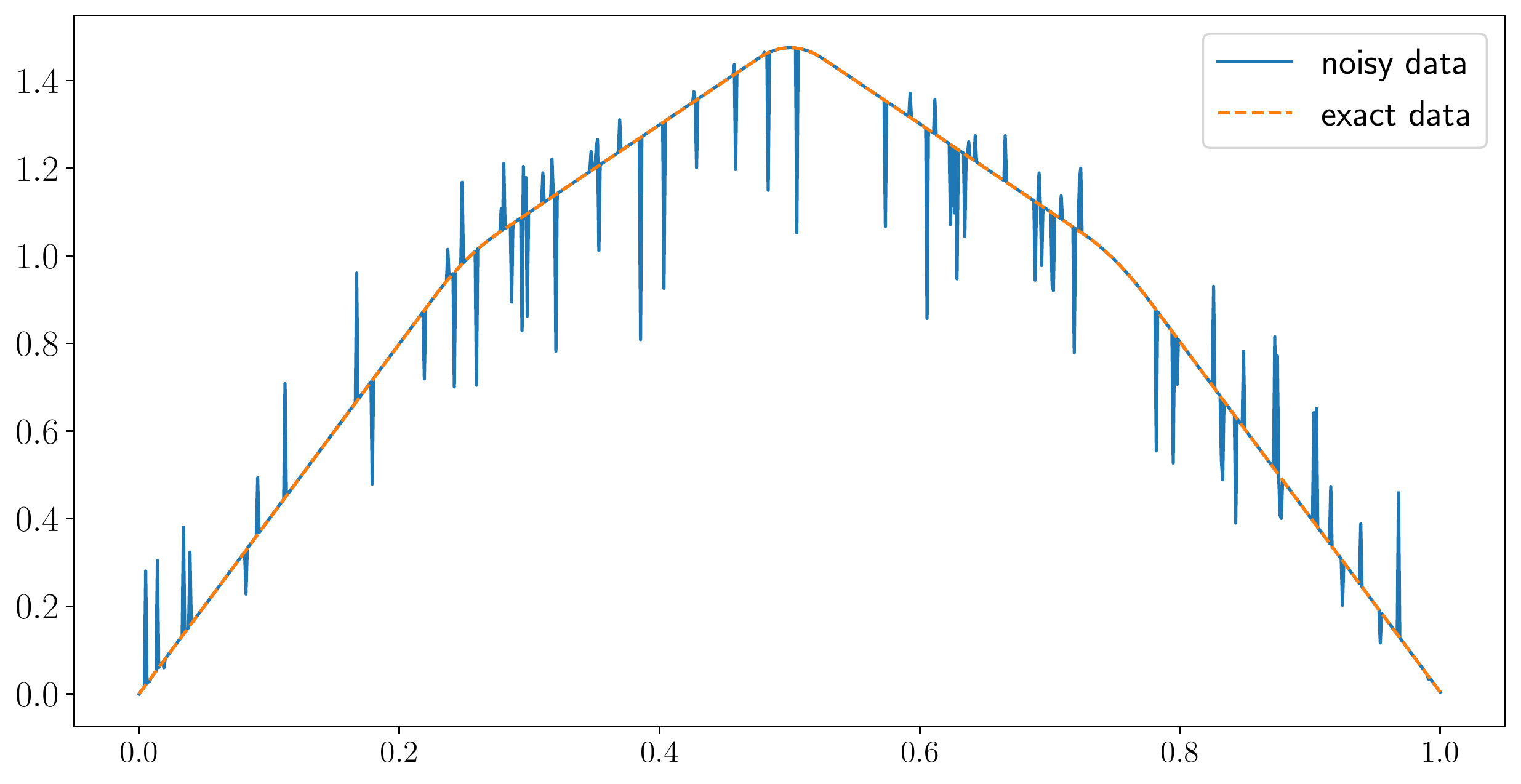}
& \includegraphics[width=.505\textwidth]{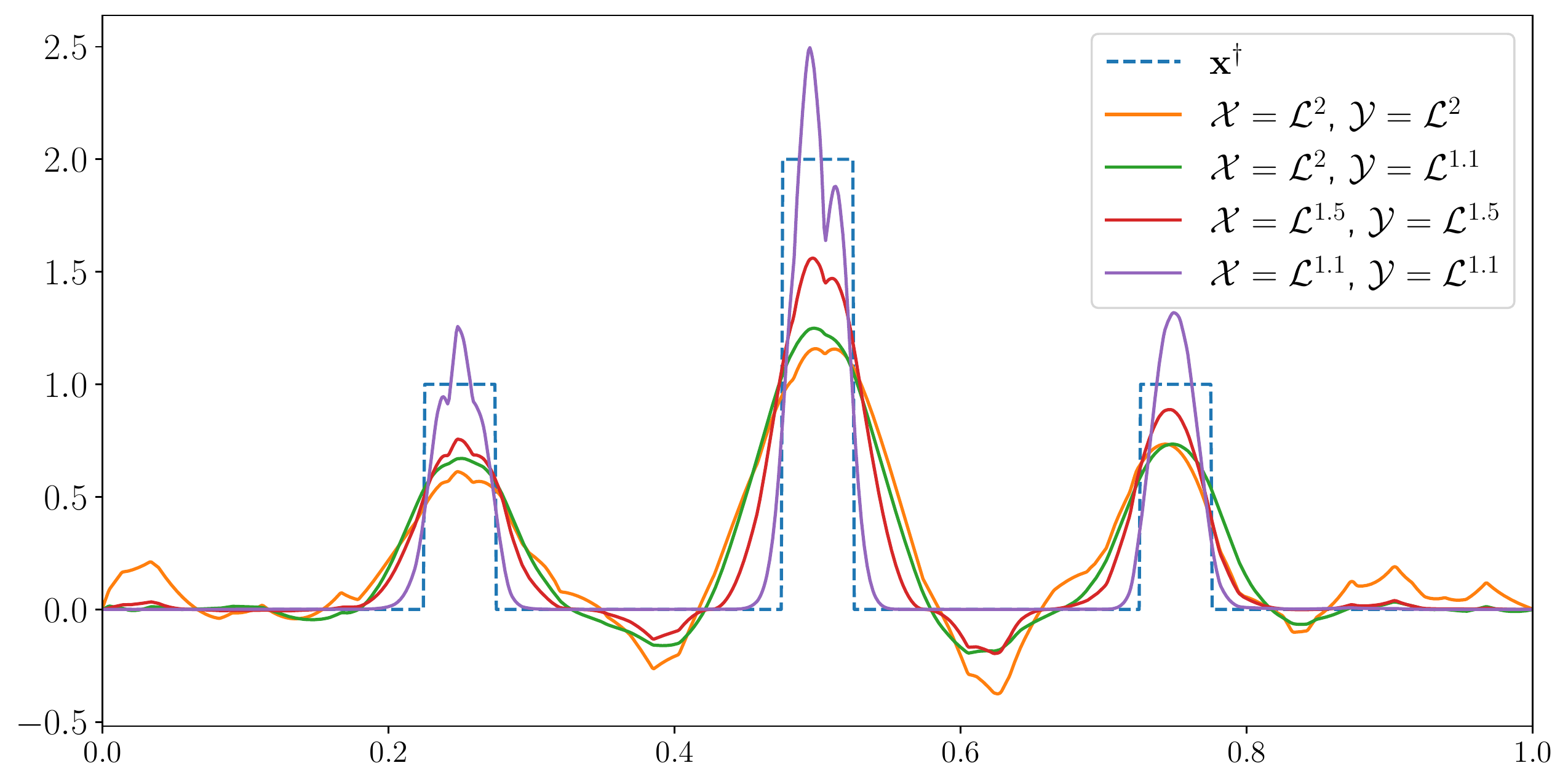} \\
{\scriptsize{(a) Data with impulse noise }}& {\scriptsize{(b) Reconstructions with respect to $\CX$ and $\CY$}}
\end{tabular}
\caption{The reconstruction performance in case of impulse noise. The algorithms utilised $N_b=100$ batches and were run for $250$ epochs.}
\label{fig:noisy_reconstruction}
\end{figure}

Lastly, we investigate the convergence behaviour of the method for the generalised model \eqref{eqn:q_kaczmarz} in Section \ref{sec:further_kaczmarz}, where stochastic directions $g(\xbs,\ybs,i)$ are defined as $g(\xbs,\ybs,i)=\VA_i^\ast \svaldmapY{q}(\VA_i\xbs-\ybs_i)$, with $q=r_{\CY}$ different from the convexity parameter $p$ of the space $\CX$. The results in Fig. \ref{fig:kaczmarz_q} show that this can indeed be beneficial for the performance of the method: the reconstructions are more accurate not only in terms of the solution support, but also in terms of the magnitudes of the non-zero entries. However, the precise mechanism of the excellent performance remains largely elusive.

\begin{figure}[h!]
\centering
\setlength{\tabcolsep}{0pt}
\begin{tabular}{cc}
\includegraphics[width=.488\textwidth]{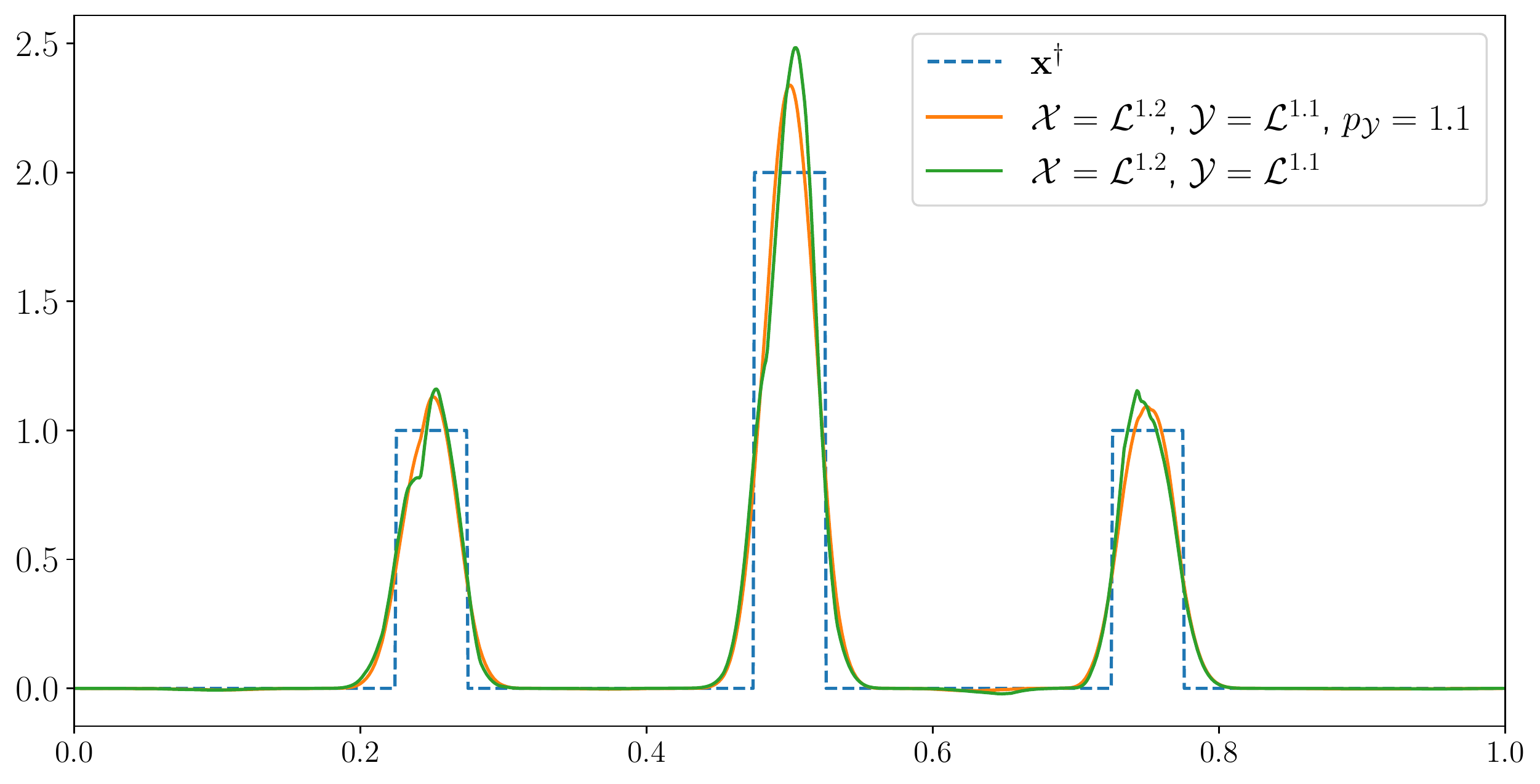}
& \includegraphics[width=.488\textwidth]{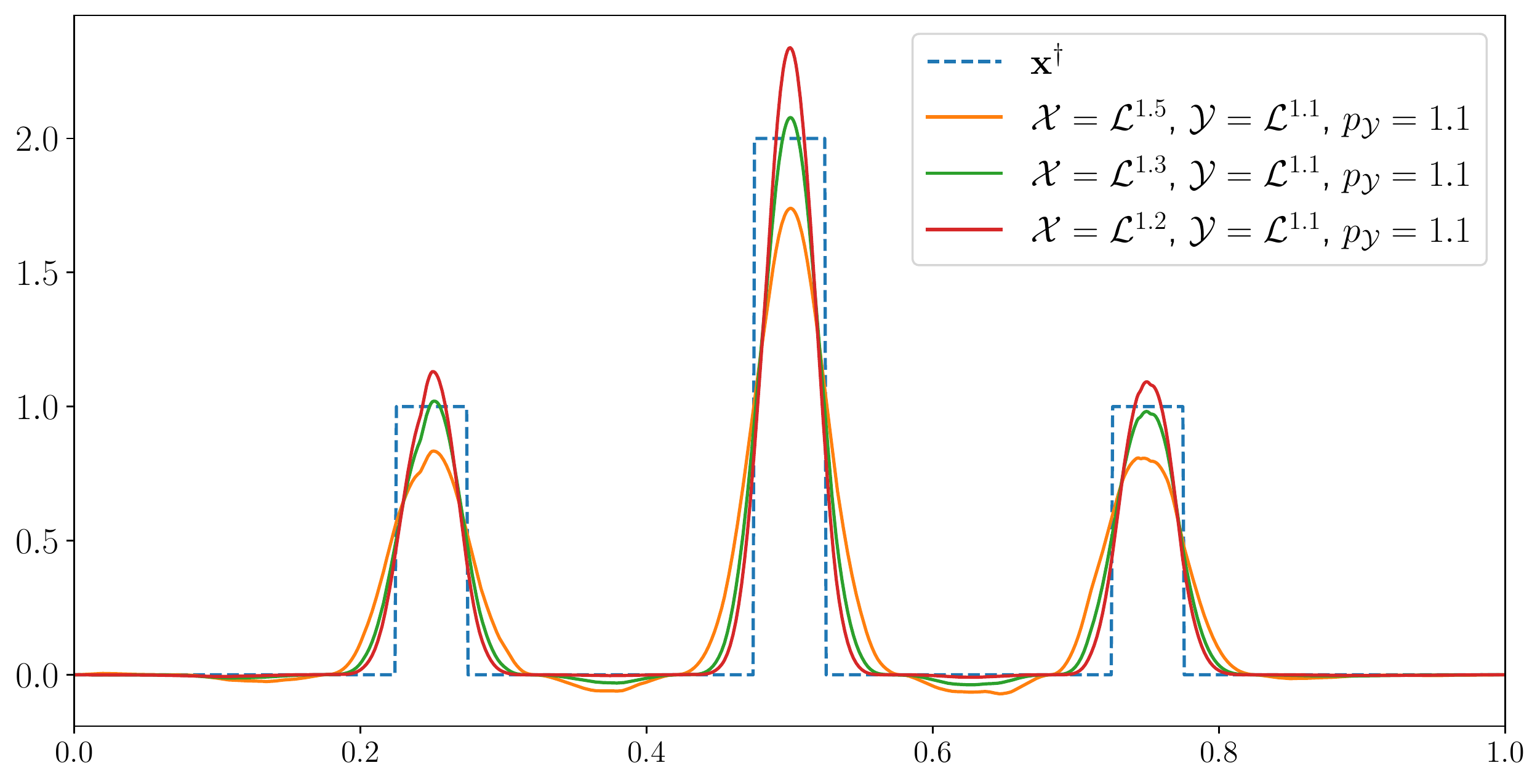} \\
{\scriptsize{(a) Standard vs generalised Kaczmarz }}& {\scriptsize{(b) Changing $\CX$ in generalised Kaczmarz}}
\end{tabular}
\caption{The dependence of the reconstructions in the case of impulse noise on the choice of $q$ parameter in the generalised model \eqref{eqn:q_kaczmarz}. The results are obtained using $N_b=100$ batches, after $250$ epochs.}
\label{fig:kaczmarz_q}
\end{figure}

\subsection{Computed Tomography}

Now we numerically investigate the behaviour of SGD on computed tomography (CT), with respect to the model spaces $\CX$ and $\CY$ and data noise. In CT reconstruction, we aim at determining the density of cross sections of an object by measuring the
attenuation of X-rays as they propagate through the object \cite{Natterer:2001}. Mathematically, the forward map is given by the Radon transform.
In the experiments, the discrete forward operator $\VA$ is defined by a $2D$ parallel beam geometry, with $180$ projection angles on a $1$ angle separation, $256$ detector elements, and pixel size of $0.1$. The sought-for signal $\xbs^\dag$ is a (sparse) phantom, cf. Fig. \ref{fig:CT_data}(a).
After applying the forward operator $\VA$, either Gaussian (with mean zero and variance $0.01$) or salt-and-pepper noise is added.
In the latter setting we consider low (with $5\%$ of values changed to either salt or pepper values) and high ($10\%$ of values changed) noise regimes.
The resulting sinograms (i.e. measurement data) are shown in Fig. \ref{fig:CT_data}(b)-(d).
Note that standard quality metrics in image assessment, such as peak signal to noise ratio or mean squared error, are computed using the distance between images in the $\ell^2$-norm, which have an implicit bias towards Hilbert spaces and smooth signals, whereas using a metric that emphasises sparsity is more pertinent to sparsity promoting spaces.
To provide a balanced comparison, we report the following two metrics based on normalised $\ell^1$- and $\ell^2$-norms:
$\delta_1(\xbs)=\|\xref-\xbs\|_{\ell^1}/\|\xref\|_{\ell^1}$ and  $\delta_2(\xbs)=\|\xref-\xbs\|_{\ell^2}/\|\xref\|_{\ell^2}.$ 

\begin{figure}[h!]
\centering
\setlength{\tabcolsep}{0pt}
\begin{tabular}{cc}
\includegraphics[height=.3\textwidth]{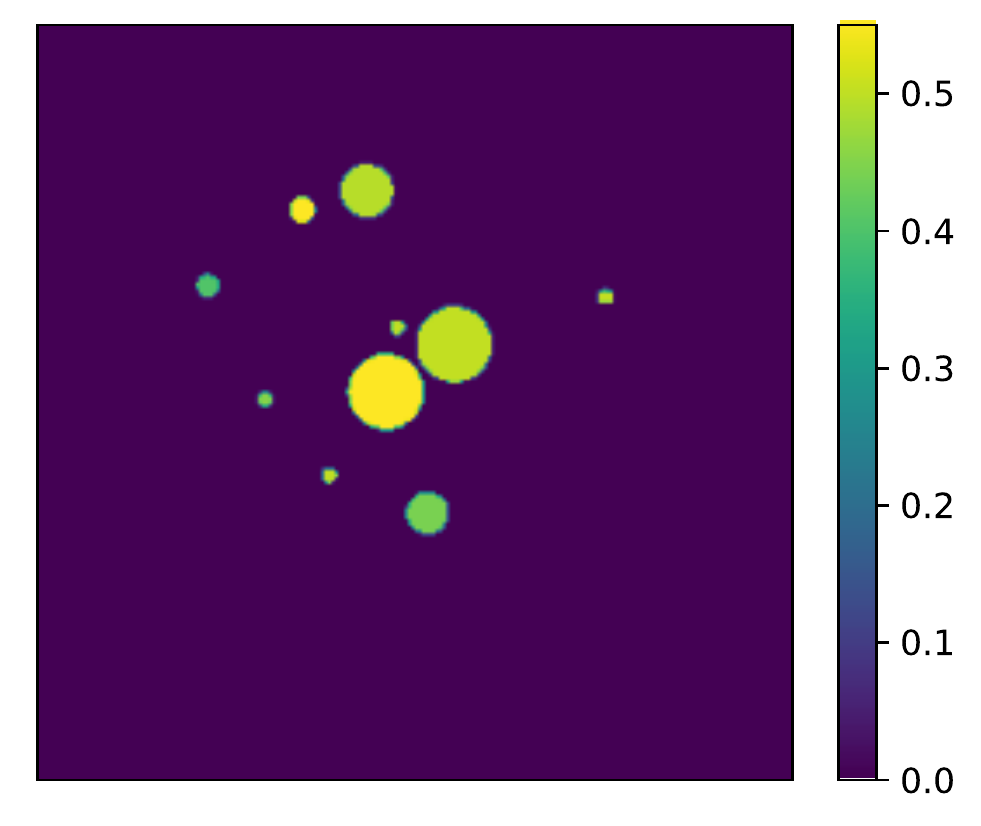}
& \includegraphics[height=.3\textwidth]{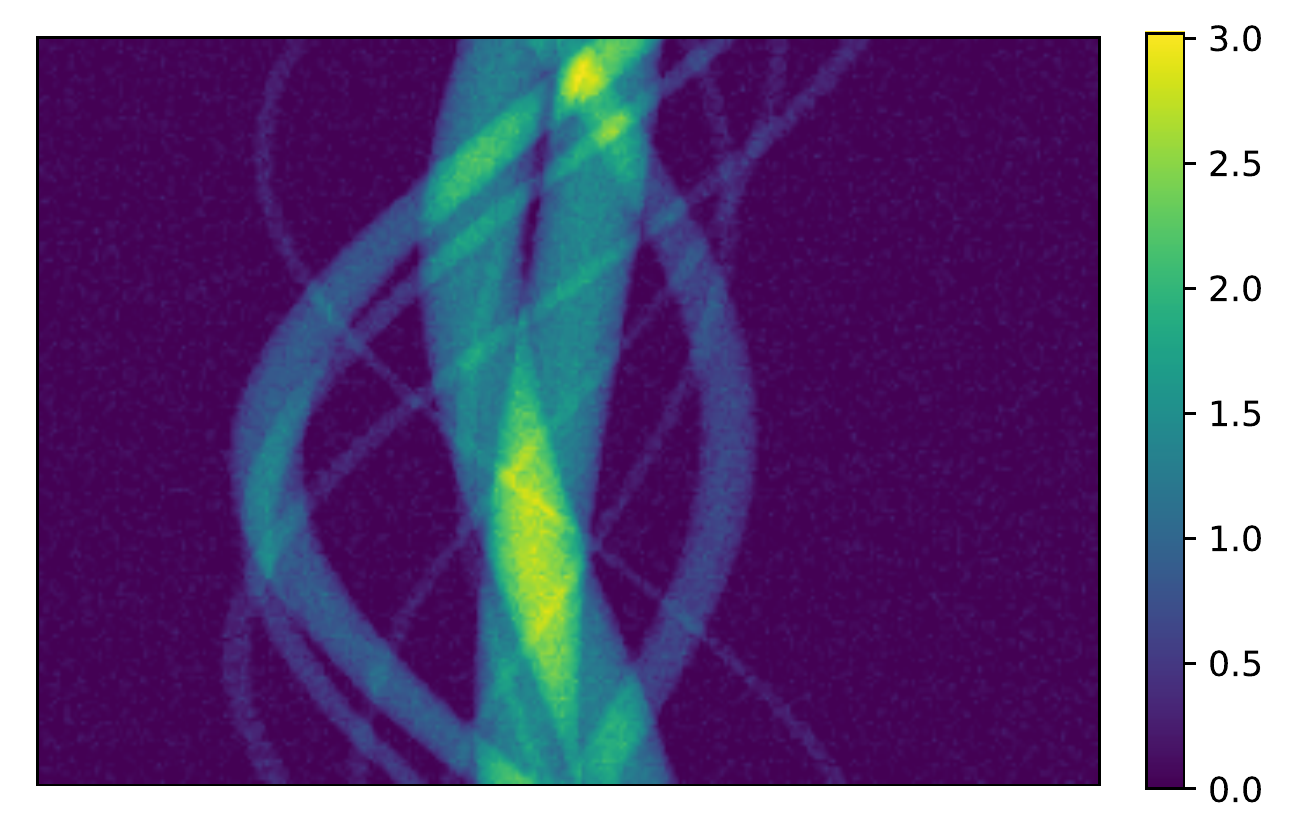}\\
{\scriptsize{(a) Original phantom }}& {\scriptsize{(b) Gaussian noise measurement }}\\
\includegraphics[height=.3\textwidth]{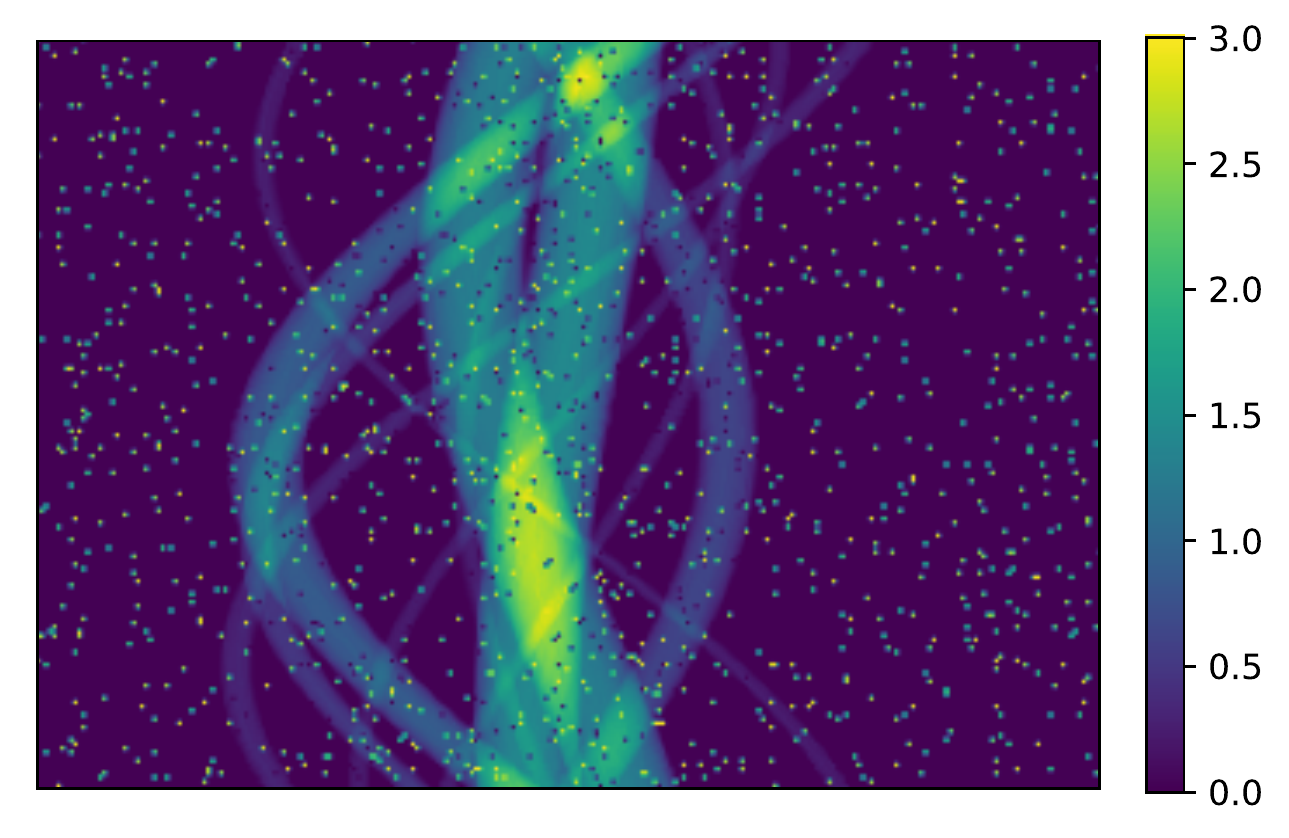} & \includegraphics[height=.3\textwidth]{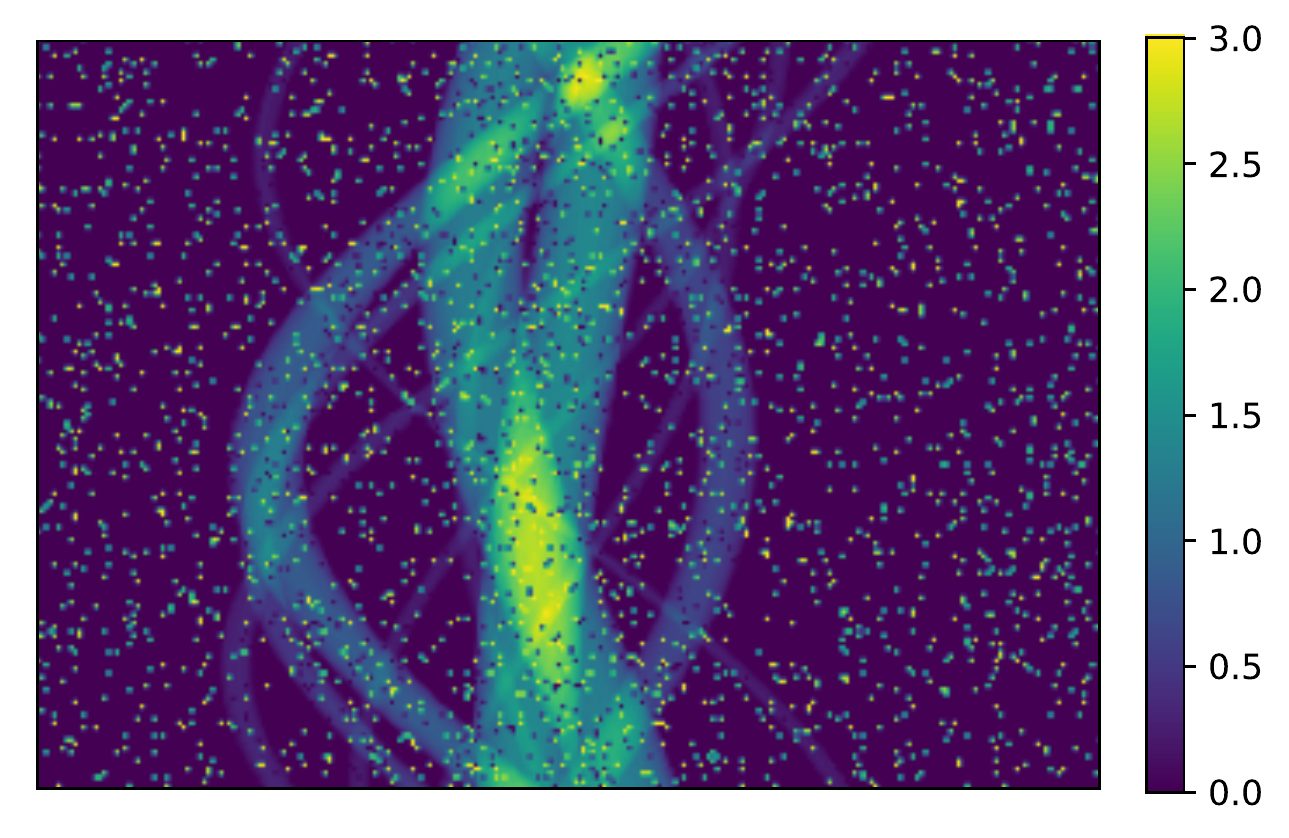} \\
{\scriptsize{(c) Low noise salt-and-pepper measurement }} & {\scriptsize{(d) High noise salt-and-pepper measurement}}
\end{tabular}
\caption{The plot in {\rm(}a{\rm)} shows the phantom to be recovered and {\rm(}b{\rm)}-{\rm(}d{\rm)} show noisy measurements used in the recovery: in {\rm(}b{\rm)}, random Gaussian noise was added, and {\rm(}c{\rm)}-{\rm(}d{\rm)} are sinogram data degraded by salt-and-pepper noise in the low {\rm(}$5\%${\rm)} and high {\rm(}$10\%${\rm)} noise regimes.}
\label{fig:CT_data}
\end{figure}

First, we show the performance on Gaussian noise, where we compare the Hilbert setting ($\CX=\CY=\CL^2$) with two Banach settings ($\CX=\CL^{1.1}$, $\CY=\CL^2$, and $\CX=\CY=\CL^{1.1}$).
In the reconstruction, we employ step-sizes $\mu_k = \frac{L_{\max}/2}{1+0.05 (k/N_b)^{1/p^\ast+0.01}}$, with $L_{\max} = \max_{j\in[N_b]} \|\VA_j\|$.
Fig. \ref{fig:CT_gaussian_60} shows exemplary reconstructions.
In all three settings much of the noise is retained in the reconstruction, and whereas the Hilbert setting is better at recovering the magnitude of non-zero entries, the Banach settings are better at recovering the support.
Moreover, we observe that the Banach setting with a sparse signal space $\CX=\CL^{1.1}$, and a smooth observation space $\CY=\CL^2$, has the best performance in terms of $\delta_1$ and $\delta_2$ metrics.
The Hilbert model performs better than the fully sparse model $\CX=\CY=\CL^{1.1}$ in terms of the smooth metric $\delta_2$, but worse in the sparsity promoting metric ($\delta_1$).
We also consider the Banach setting for the generalised model \eqref{eqn:q_kaczmarz}, with $\CX=\CY=\CL^{1.1}$ and $p_{\CY}=1.1$, where we study the effects of early stopping.
Fig. \ref{fig:CT_gaussian_epochs} shows that this setting recovers the support more accurately (and actually does so very early on) and recovers the magnitudes better, but that a form of regularisation (through e.g. early stopping) can be beneficial, since in the later epochs SGD iterates again tend to overshoot on the support. 
A similar behaviour can observed for other studied Banach space settings, but not for the Hilbert space setting, which does not recover the support.

\begin{figure}[h!]
\centering
\small
\setlength{\tabcolsep}{0pt}
\begin{tabular}{ccc}
\includegraphics[height=.27\textwidth]{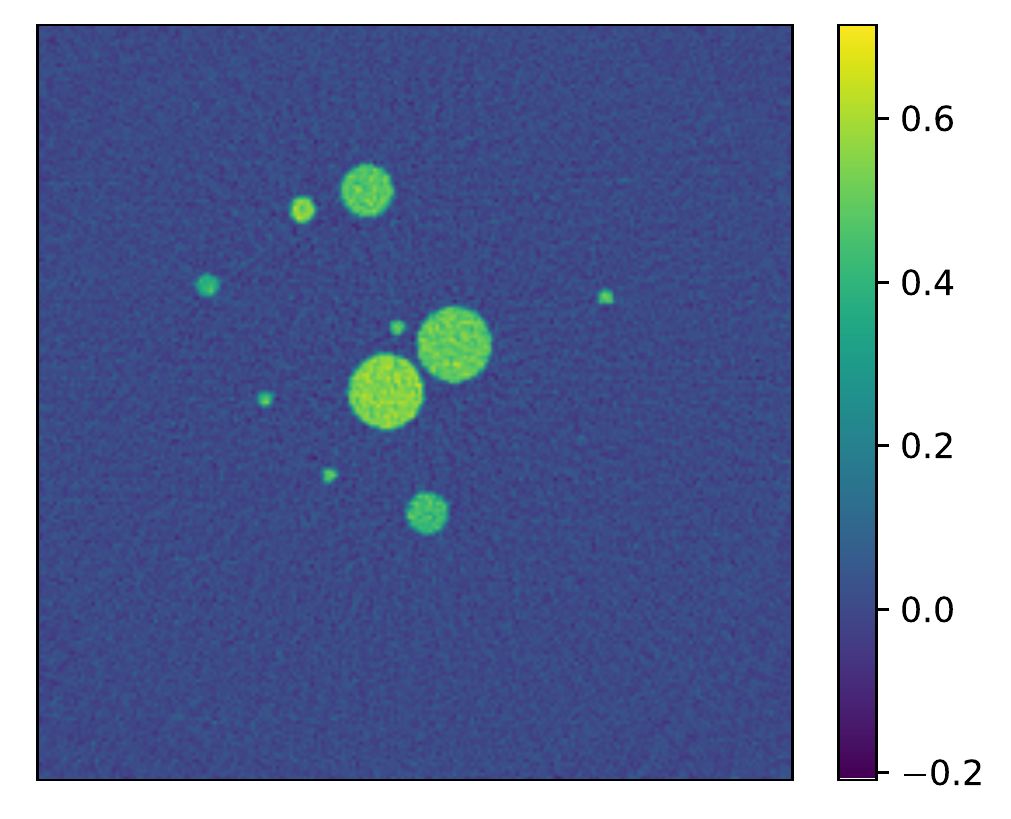}
& \includegraphics[height=.27\textwidth]{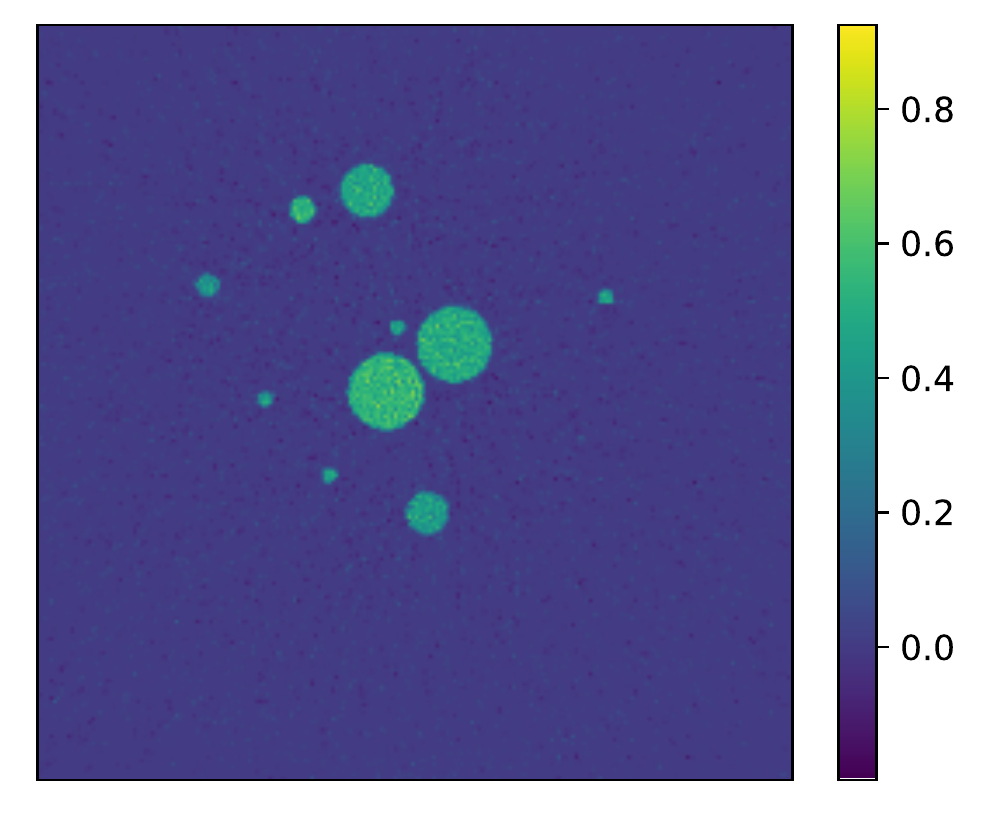} & \includegraphics[height=.27\textwidth]{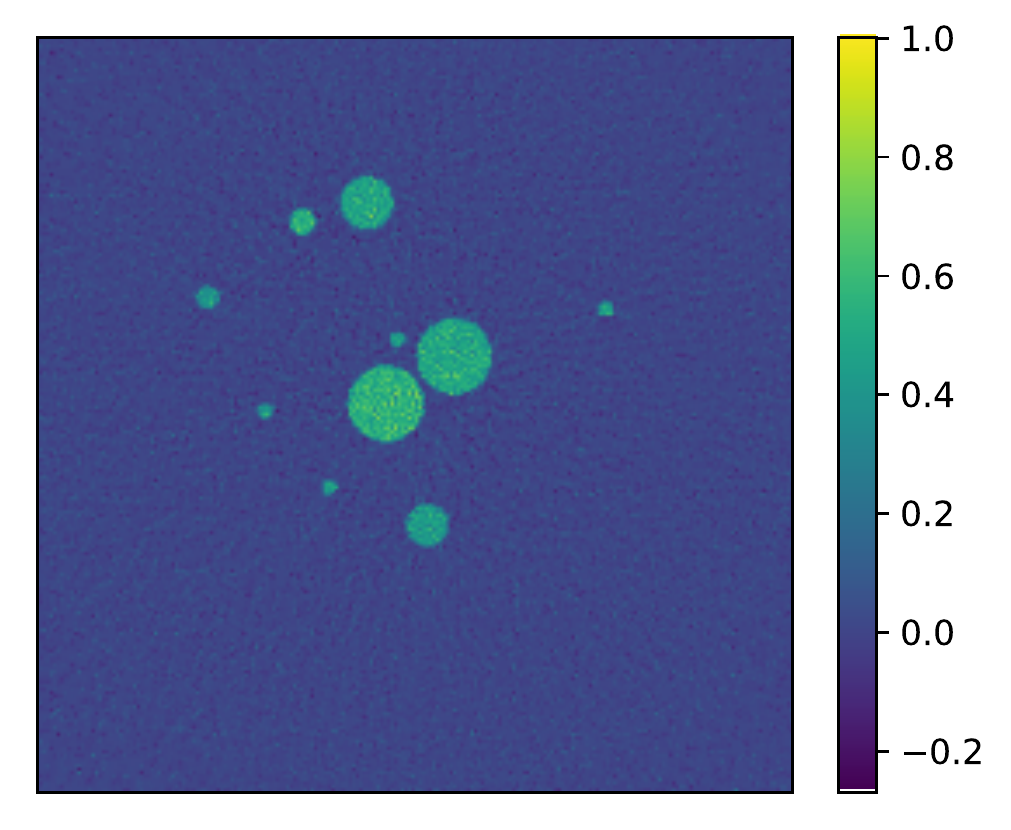} \\
{\scriptsize{(a) $\CX=\CY=\CL^2$ }}& {\scriptsize{(b)$\CX=\CL^{1.1}$, $\CY=\CL^{2}$ }}& {\scriptsize{(c) $\CX=\CY=\CL^{1.1}$ }}\\
{\scriptsize{$\delta_1(\xbs)/\delta_2(\xbs)$: $2.643/0.528$ }}& {\scriptsize{$\delta_1(\xbs)/\delta_2(\xbs)$: $0.711/0.341$ }}& {\scriptsize{$\delta_1(\xbs)/\delta_2(\xbs)$: $2.195/0.620$}}
\end{tabular}
\caption{The reconstruction of the phantom from the observed sinograms degraded by Gaussian noise, cf. Fig. \ref{fig:CT_data}{\rm(}b{\rm)}.
The algorithms use $N_b=60$ batches and were run for $200$ epochs.}
\label{fig:CT_gaussian_60}
\end{figure}

\begin{figure}[h!]
\centering
\small
\setlength{\tabcolsep}{0pt}
\begin{tabular}{ccc}
\includegraphics[height=.27\textwidth]{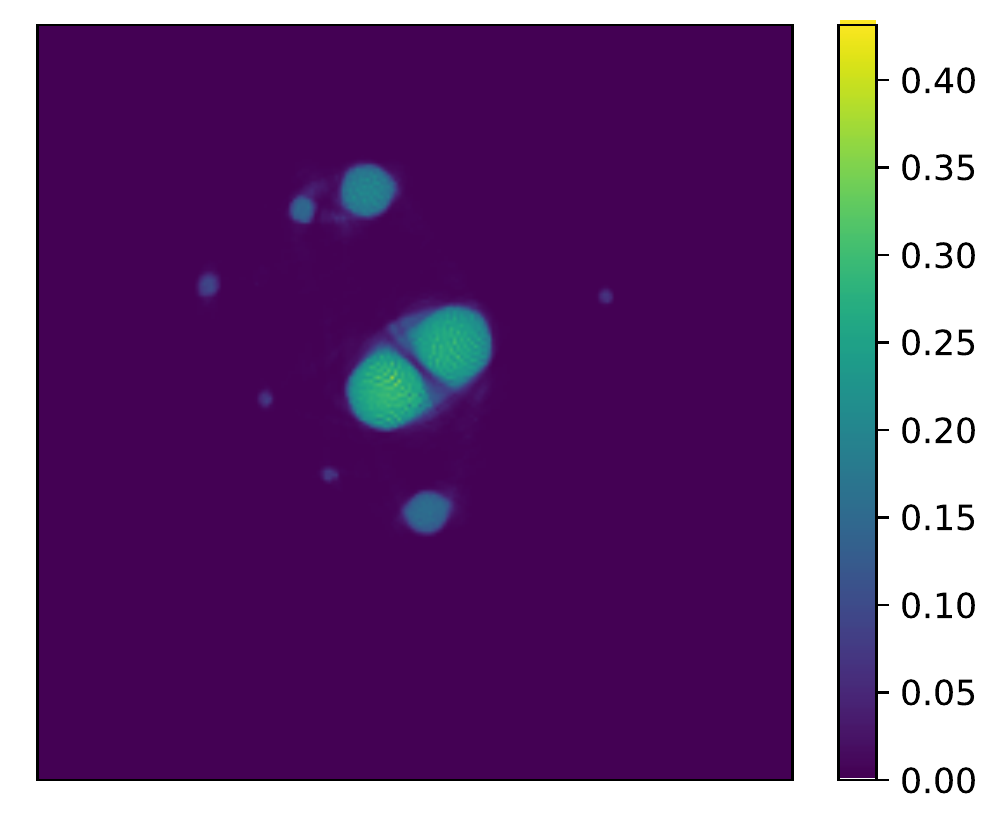}
& \includegraphics[height=.27\textwidth]{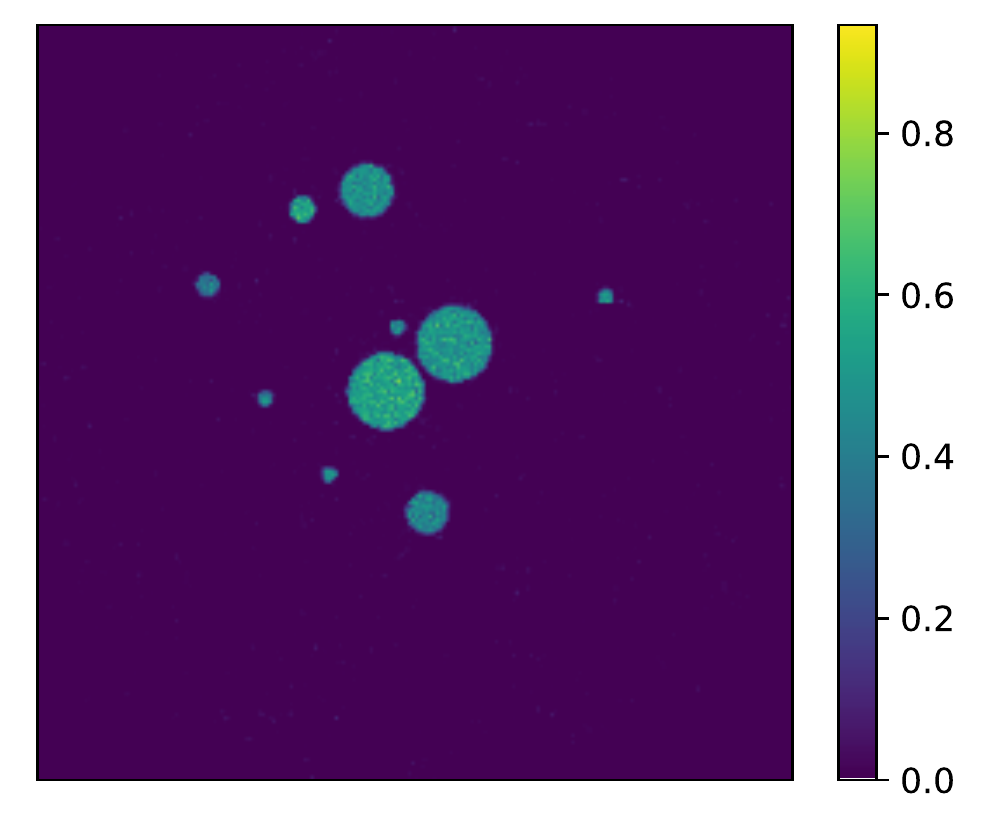} & \includegraphics[height=.27\textwidth]{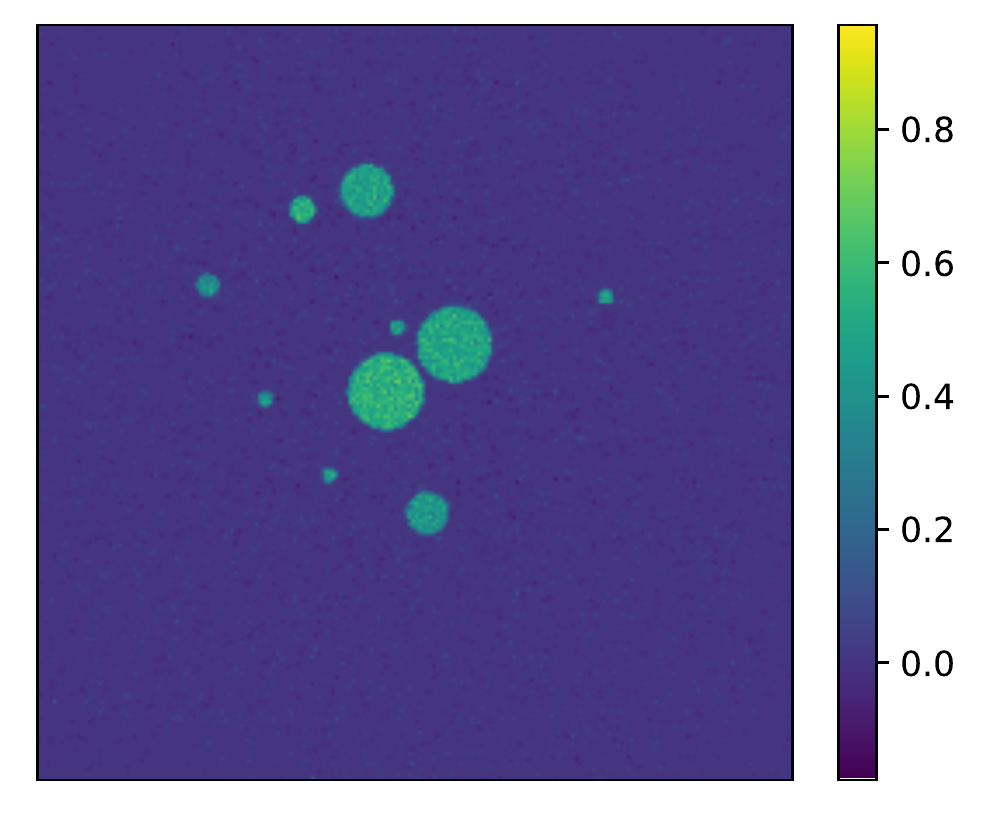} \\
{\scriptsize{(a) $5$ epochs }}& {\scriptsize{(b) $50$ epochs}}& {\scriptsize{(c) $200$ epochs}}\\
{\scriptsize{$\delta_1(\xbs)/\delta_2(\xbs)$: $0.702/0.627$ }}& {\scriptsize{$\delta_1(\xbs)/\delta_2(\xbs)$: $0.263/0.235$ }}& {\scriptsize{$\delta_1(\xbs)/\delta_2(\xbs)$: $0.604/0.308$}}
\end{tabular}
\caption{The evolution of the quality of reconstruction from sinograms degraded by Gaussian noise with respect to the number of epochs. The algorithm uses $\CX=\CY=\CL^{1.1}$ and $p_{\CY}=1.1$, with $N_b=60$ batches.}
\label{fig:CT_gaussian_epochs}
\end{figure}

We next investigate the performance for low and high salt-and-pepper noise.
We compare the Hilbert setting with two Banach settings: the standard SGD with $\CX=\CY=\CL^{1.1}$ and the generalised model \eqref{eqn:q_kaczmarz} with  $\CX=\CY=\CL^{1.1}$ and $p_{\CY}=1.1$.
For the reconstruction, we employ step-sizes $\mu_k = \frac{0.5}{1+0.05 (k/N_b)^{1/p^\ast+0.01}}$.
The results in Fig. \ref{fig:CT_sp_60} show the reconstructions after $200$ epochs with $N_b=60$ batches. In the low noise regime, the Hilbert setting can reconstruct the general shape of the phantom, but retains a lot of the noise and exhibits streaking artefacts in the background. The reconstruction in the high noise regime is of much poorer quality.
The standard Banach SGD shows good behaviour in the low-noise setting, reconstructing well both the sparsity structure and the magnitudes, but its performance degrades in the high noise setting.
In sharp contrast, the model \eqref{eqn:q_kaczmarz} shows a nearly perfect reconstruction performance - the phantom is well recovered, with intensities on the correct scale, for both low and high noise regimes.
Similar as before, we observe that Banach methods tend to slightly overestimate the overall intensities, though the recovered values are comparable to the true solution.
Overall, the Hilbert setting shows a qualitatively worst performance, in both $\ell^1$- and $\ell^2$-norm sense, and the model \eqref{eqn:q_kaczmarz} shows the best performance.

\begin{figure}[h!]
\centering
\small
\setlength{\tabcolsep}{-2pt}
\begin{tabular}{ccc}
\includegraphics[height=.27\textwidth]{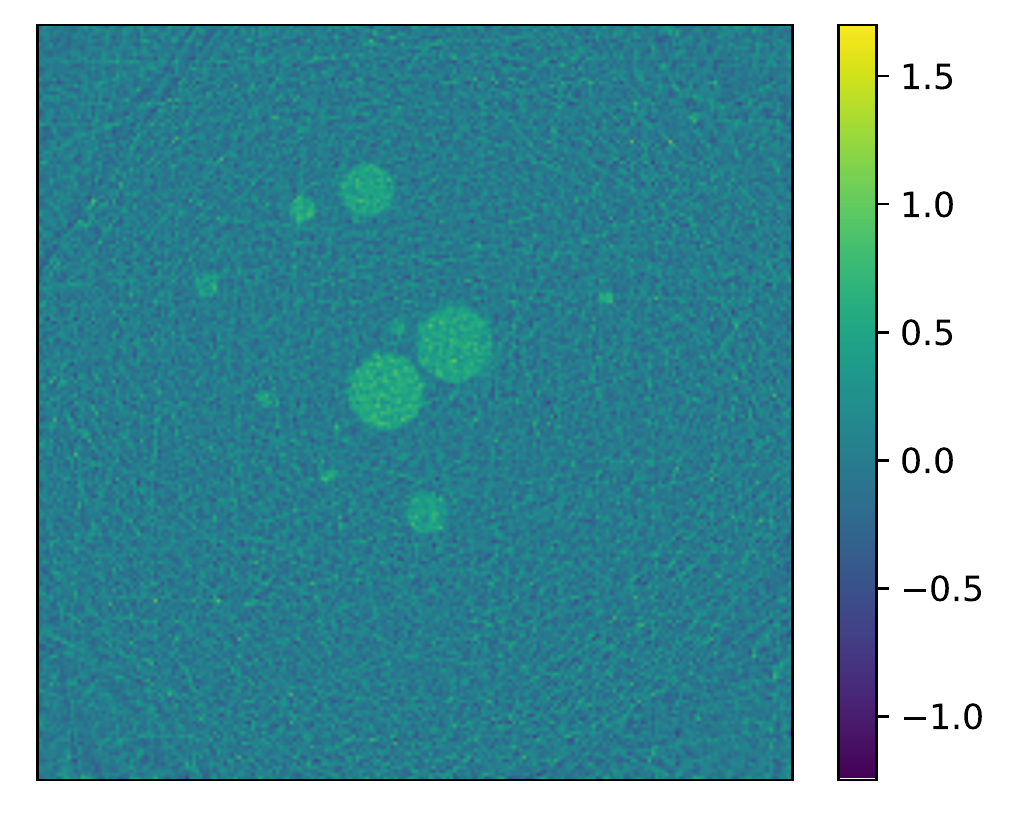}
& \includegraphics[height=.27\textwidth]{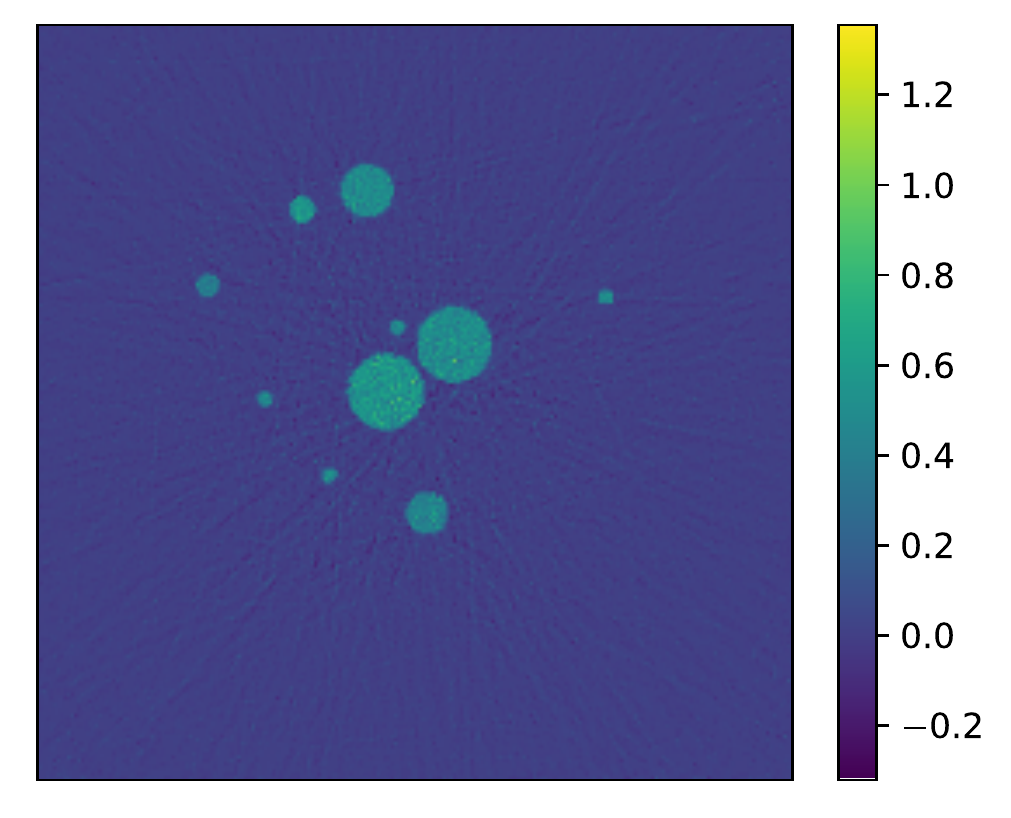} & \includegraphics[height=.27\textwidth]{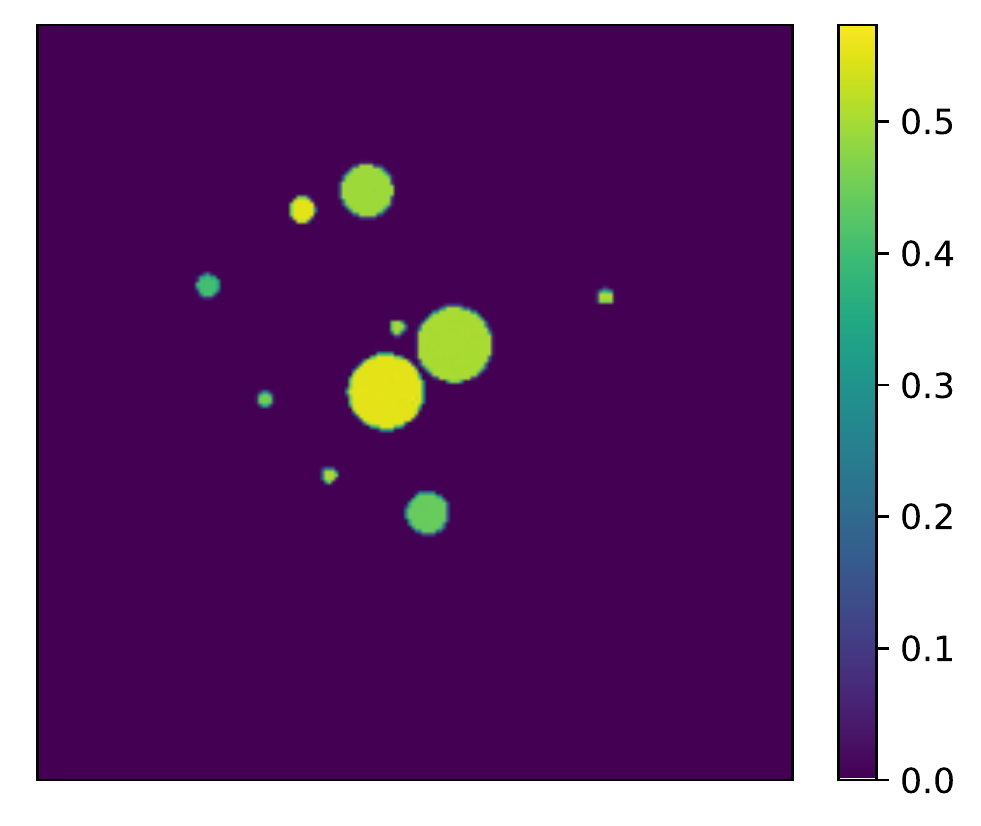} \\
{\scriptsize{(a) $\CX=\CY=\CL^2$ in low noise}}& {\scriptsize{(b)$\CX=\CY=\CL^{1.1}$ in low noise}}& {\scriptsize{(c) $\CX=\CY=\CL^{1.1}$, $p_{\CY}=1.1$ in low noise}}\\
{\scriptsize{$\delta_1(\xbs)/\delta_2(\xbs)$: $18.67/3.71$ }}&{\scriptsize{ $\delta_1(\xbs)/\delta_2(\xbs)$: $1.80/0.544$}} &{\scriptsize{ $\delta_1(\xbs)/\delta_2(\xbs)$: $2.43/3.68\cdot \text{e-}3$ }}\\
\includegraphics[height=.27\textwidth]{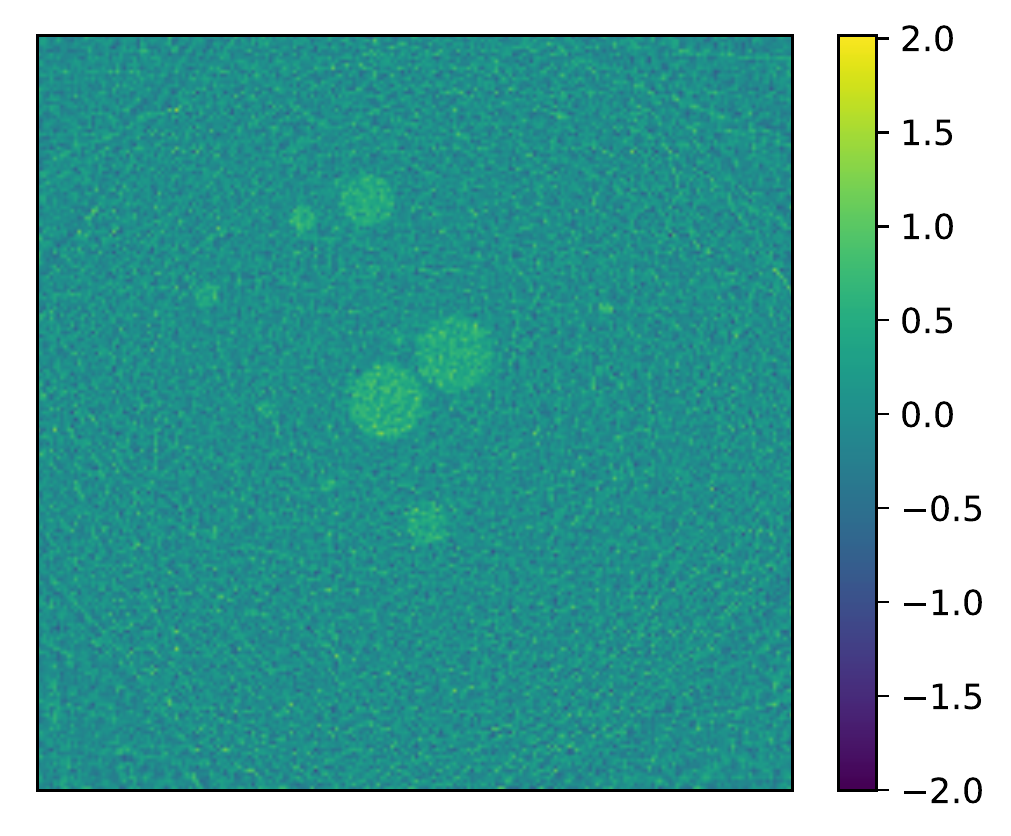}
& \includegraphics[height=.27\textwidth]{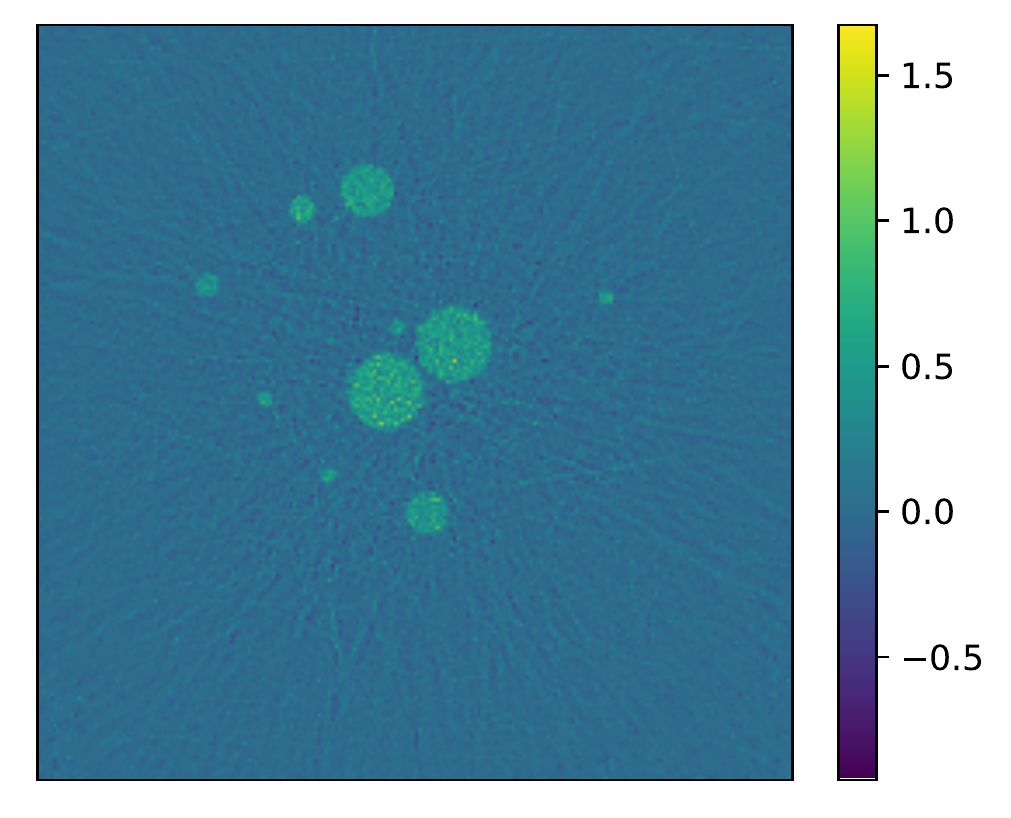} & \includegraphics[height=.27\textwidth]{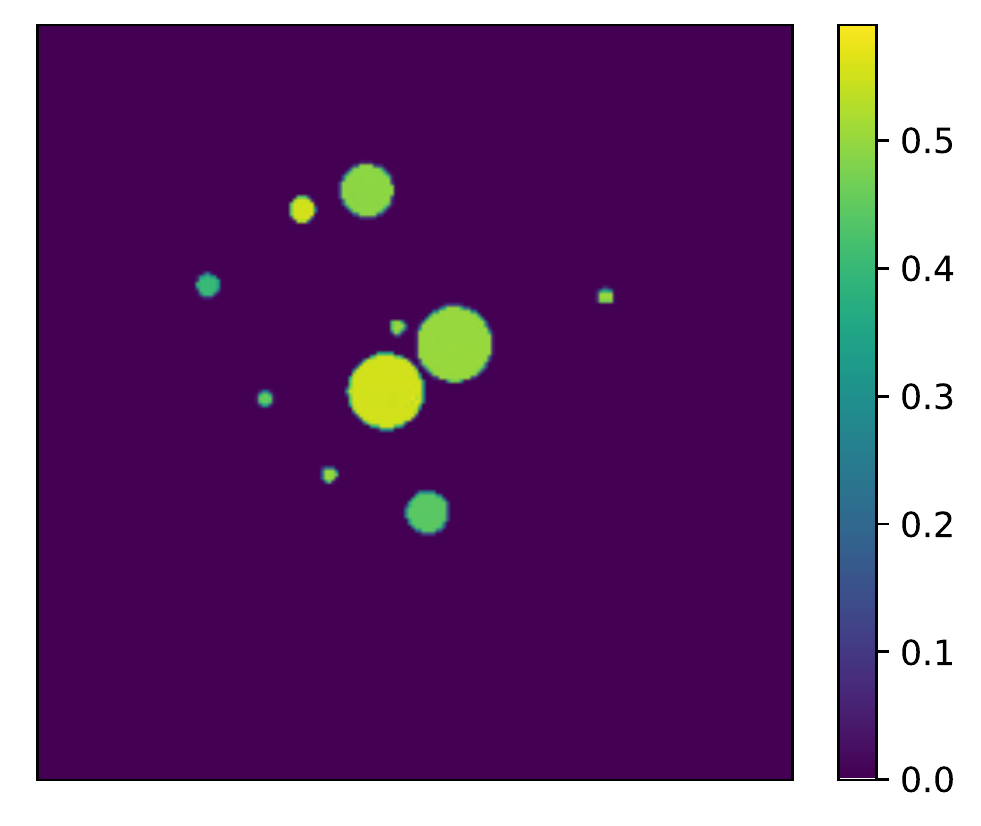} \\
{\scriptsize{(a) $\CX=\CY=\CL^2$ in high noise}}& {\scriptsize{(b) $\CX=\CY=\CL^{1.1}$ in high noise}}& {\scriptsize{(c) $\CX=\CY=\CL^{1.1}$, $p_{\CY}=1.1$ in high noise }}\\
{\scriptsize{$\delta_1(\xbs)/\delta_2(\xbs)$: $26.61/5.19$ }}& {\scriptsize{$\delta_1(\xbs)/\delta_2(\xbs)$: $5.37/1.54$ }}& {\scriptsize{$\delta_1(\xbs)/\delta_2(\xbs)$: $3.72/6.03\cdot\text{e-}3$}}
\end{tabular}
\caption{The reconstruction of the phantom from the observed sinograms, degraded with low {\rm(}top{\rm)} and high {\rm(}bottom{\rm)} salt-and-pepper noise, respectively, 
obtained using the Hilbert space  model {\rm(}$\CX=\CY=\CL^2${\rm)} {\rm(}left{\rm)}, the Banach model {\rm(}$\CX=\CY=\CL^{1.1}${\rm)} {\rm(}middle{\rm)} and the Banach model {\rm(}$\CX=\CY=\CL^{1.1}${\rm)} with the generalised Kaczmarz scheme {\rm(}$p_{\CY}=1.1${\rm)} {\rm(}right{\rm)}. The algorithms use $N_b=60$ batches and were run for $200$ epochs.}
\label{fig:CT_sp_60}
\end{figure}

Lastly, we investigate a more challenging setting with noise affecting not only the sinograms, but also the original phantoms. Then the ground-truth image is only approximately sparse.
The phantom is degraded with Gaussian noise (zero mean and variance $0.01$) after which
we apply the forward operator to the resulting noisy phantom. We then add either Gaussian (zero mean and variance $0.01$)  or salt-and-pepper noise (affecting $3\%$ of measurements); see Fig. \ref{fig:CT_prepost_data} for
representative images. The reconstruction algorithms use SGD with a decaying step-size schedule, $\mu_k = \frac{0.2}{1+0.05 (k/N_b)^{1/p^\ast+0.01}}$.

\begin{figure}[h!]
\centering
\small
\setlength{\tabcolsep}{0pt}
\begin{tabular}{ccc}
\includegraphics[height=.22\textwidth]{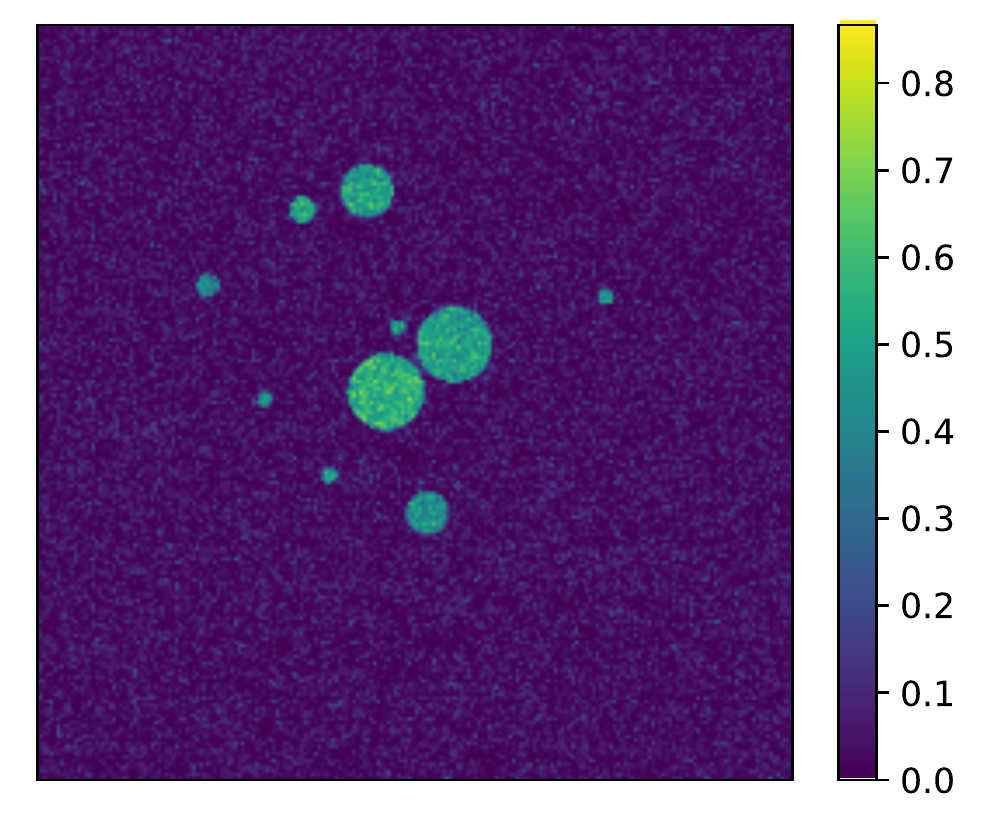}
& \includegraphics[height=.22\textwidth]{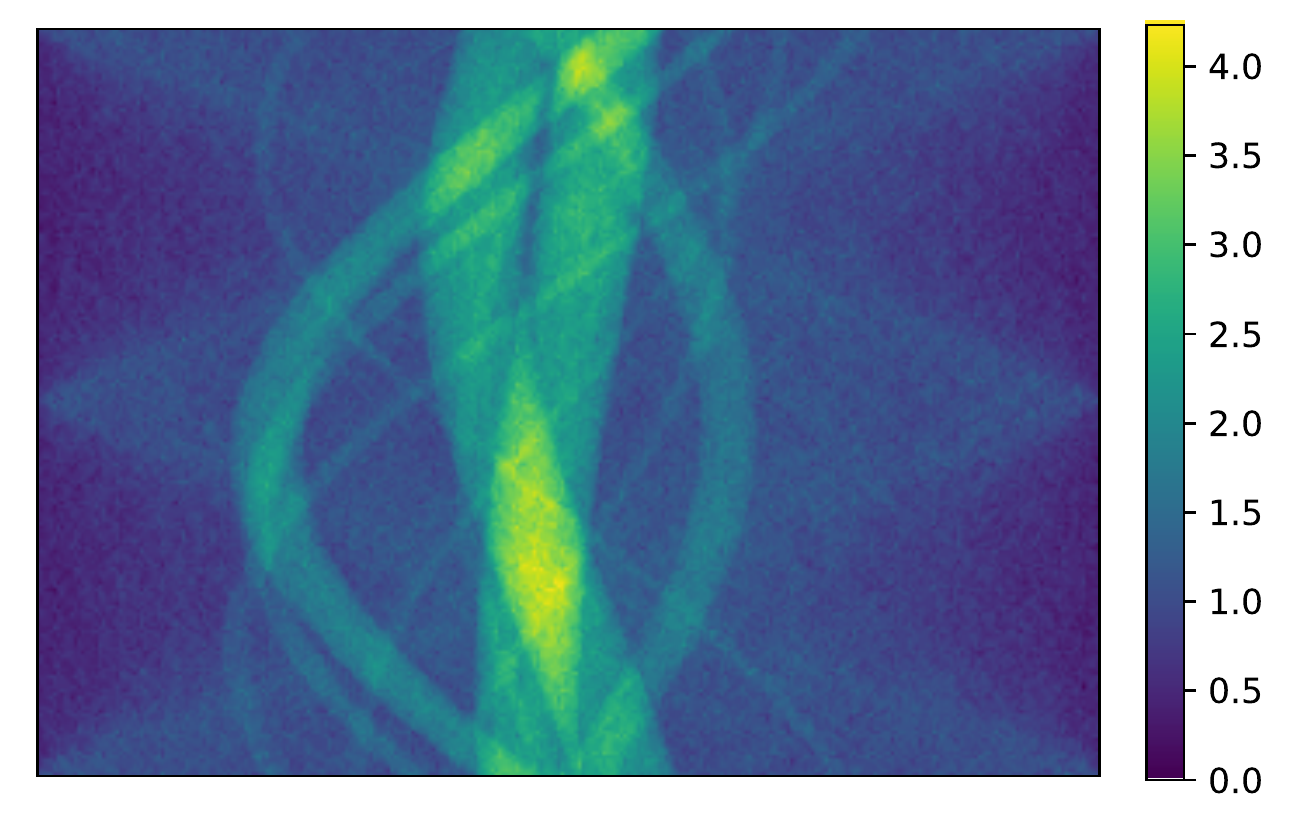} & \includegraphics[height=.22\textwidth]{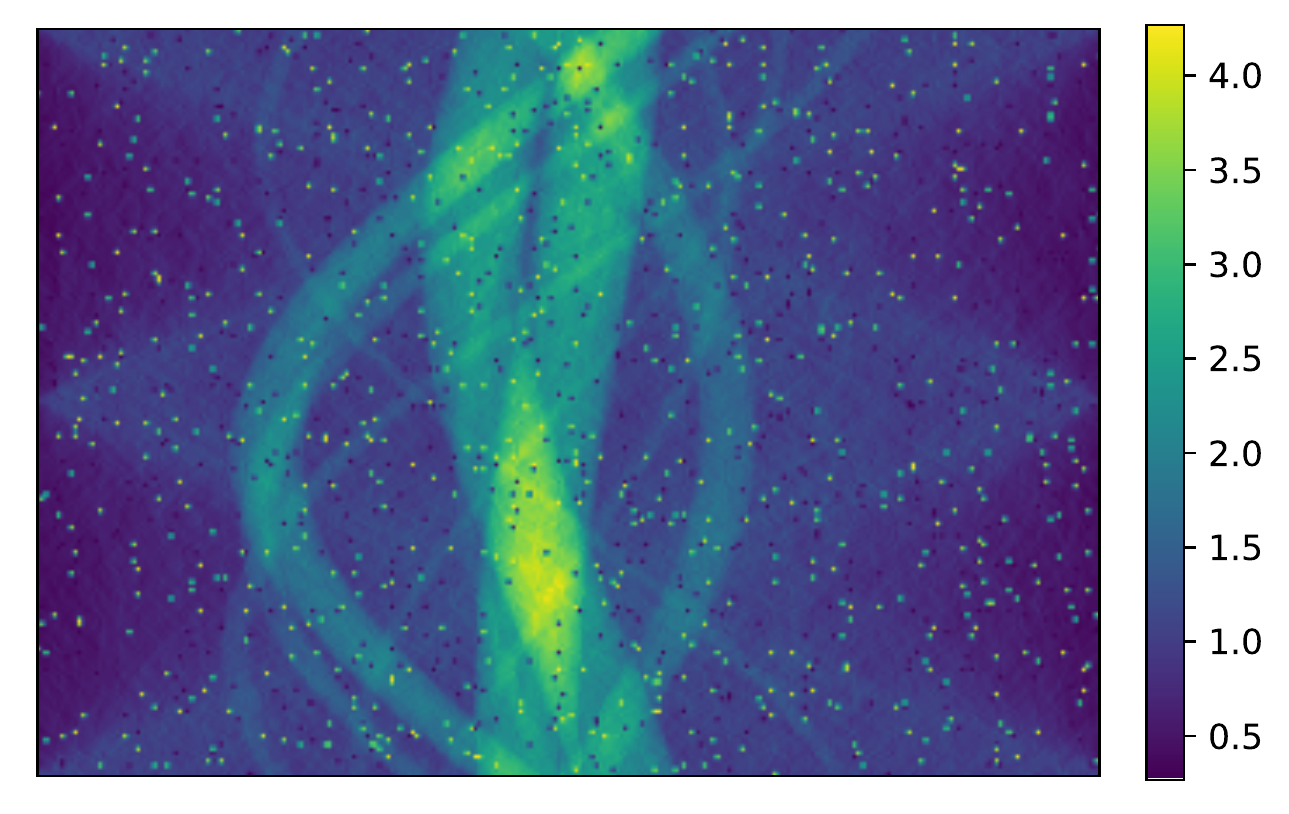} \\
{\scriptsize{(a) Noisy Phantom }}& {\scriptsize{(b) Gaussian measurement noise }}& {\scriptsize{(c) Salt-and-pepper measurement noise}}
\end{tabular}
\caption{The phantoms and sinograms for the forward model with both pre and post measurement noise. The phantom on the left is degraded by Gaussian noise. After applying the forward operator, either Gaussian {\rm(}middle{\rm)} or salt-and-pepper noise {\rm(}right{\rm)} is added to the sinogram.}
\label{fig:CT_prepost_data}
\end{figure} 

The reconstructions for data with Gaussian noise in both phantom and sinogram are shown in Fig. \ref{fig:CT_prepost_gaussgauss}. As before, reconstructions in the Hilbert setting are comparable, but slightly worse than that with the Banach ones. Banach methods are better at recovering the sparsity structure of the solution, and have better reconstruction quality metrics, though they do not completely remove the noise. 
In the second setting, with the Gaussian noise affecting the phantom and salt-and-pepper noise affecting the sinogram, the difference in reconstruction quality in the Hilbert space and Banach space settings is significantly more pronounced, cf. Fig. \ref{fig:CT_prepost_gausssnp}. In both settings, the choice of spaces $\CX$ and $\CY$ can have a big impact on the reconstruction quality, especially on the amount of noise retained in the background. Moreover, further improvements can be achieved by explicitly penalising the objective function.

\begin{figure}[h!]
\centering
\small
\setlength{\tabcolsep}{0pt}
\begin{tabular}{ccc}
\includegraphics[height=.27\textwidth]{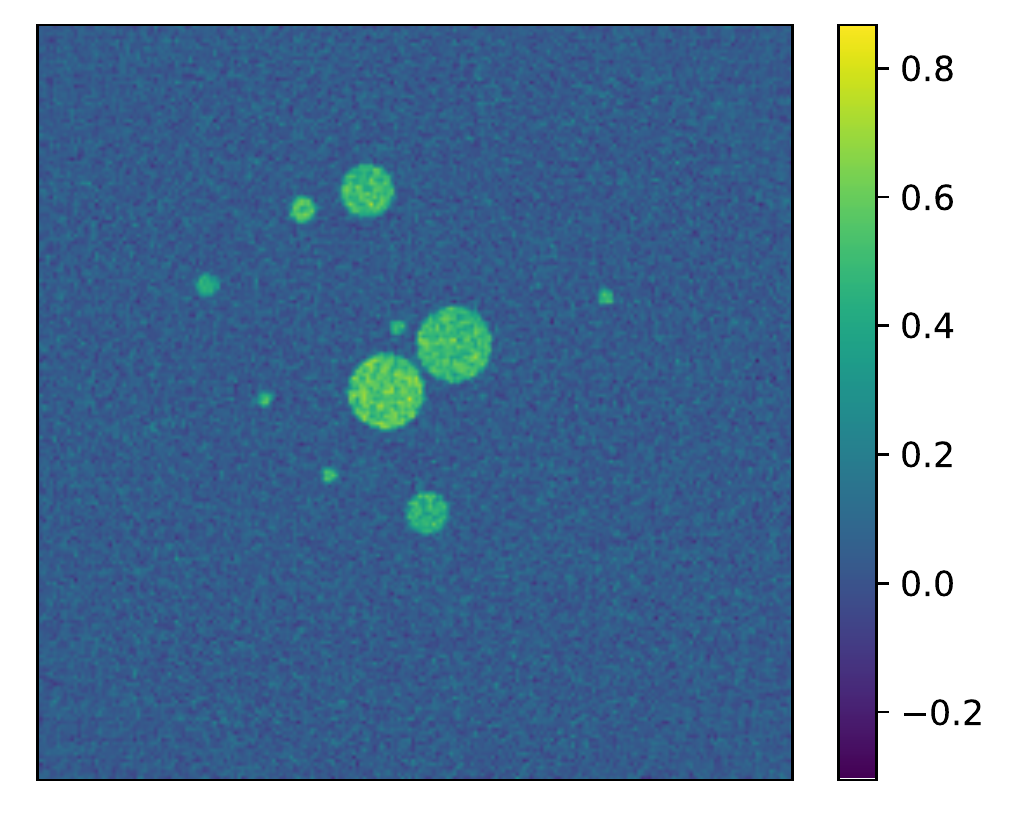}
& \includegraphics[height=.27\textwidth]{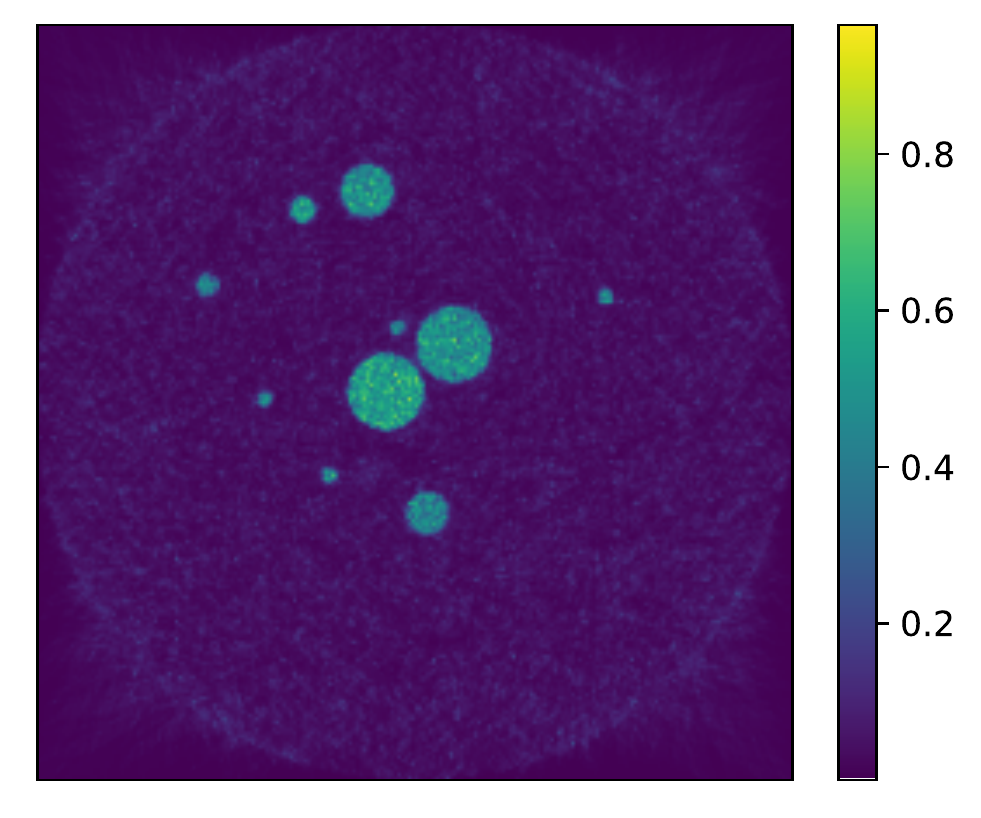} & \includegraphics[height=.27\textwidth]{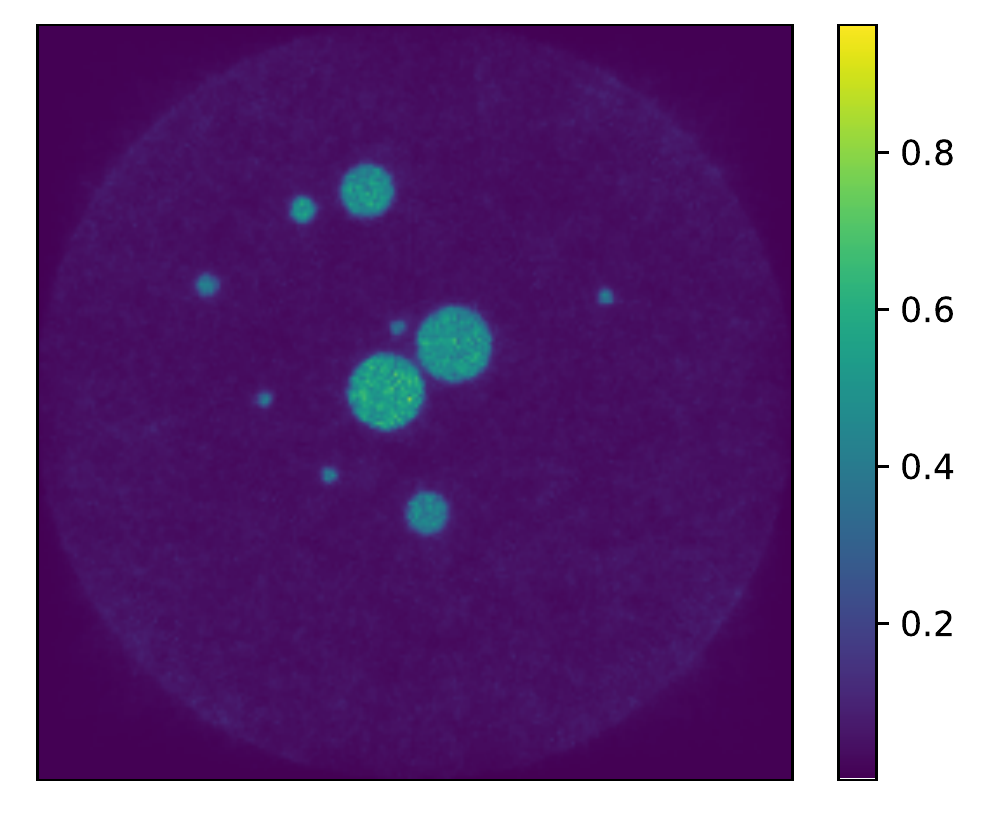} \\
{\scriptsize{(a) $\CX=\CY=\CL^2$ }}& {\scriptsize{(b) $\CX=\CY=\CL^{1.1}$, $p_{\CY}=1.1$ }}& {\scriptsize{(c) $\CX=\CL^{1.1}$, $\CY=\CL^{1.9}$, $p_{\CY}=1.9$ }}\\
{\scriptsize{$\delta_1(\xbs)/\delta_2(\xbs)$: $5.65/1.12$ }}& {\scriptsize{$\delta_1(\xbs)/\delta_2(\xbs)$: $3.16/0.632$}} & {\scriptsize{$\delta_1(\xbs)/\delta_2(\xbs)$: $2.99/0.561$}}
\end{tabular}
\caption{The reconstruction of the phantom from the observed sinograms with pre- and post-measurement Gaussian noise.
The algorithms use $N_b=60$ batches and were run for $200$ epochs.}
\label{fig:CT_prepost_gaussgauss}
\end{figure}

\begin{figure}[h!]
\centering
\small
\setlength{\tabcolsep}{0pt}
\begin{tabular}{ccc}
\includegraphics[height=.27\textwidth]{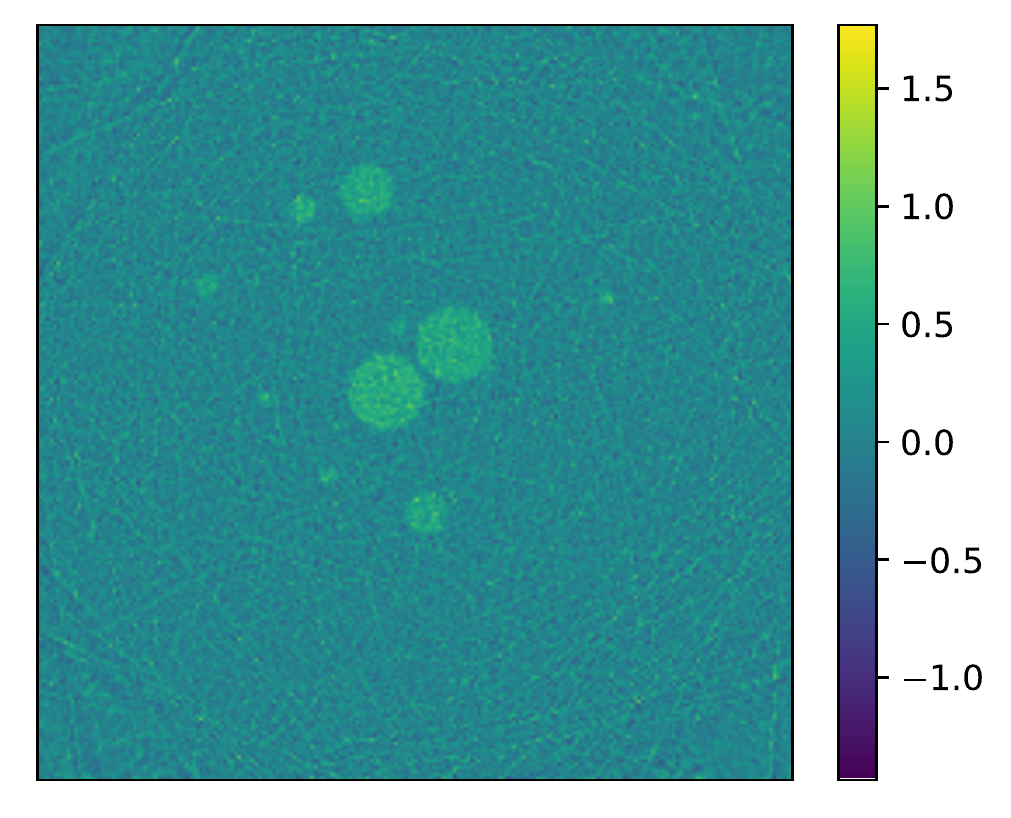}
& \includegraphics[height=.27\textwidth]{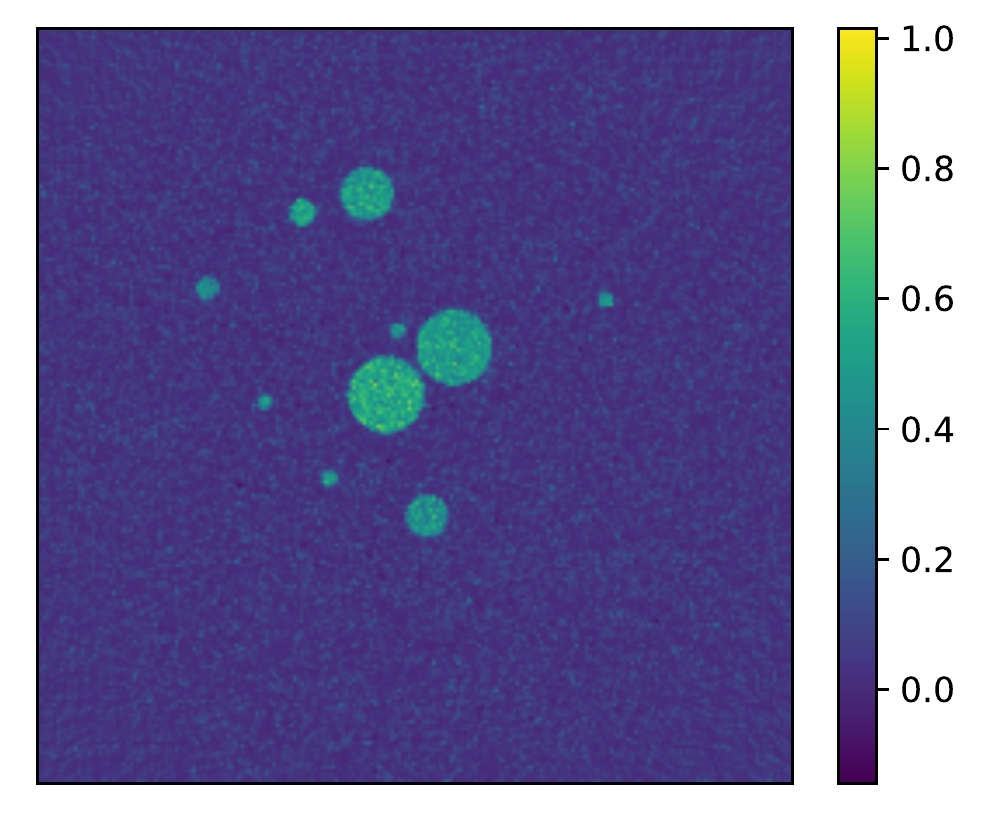} & \includegraphics[height=.27\textwidth]{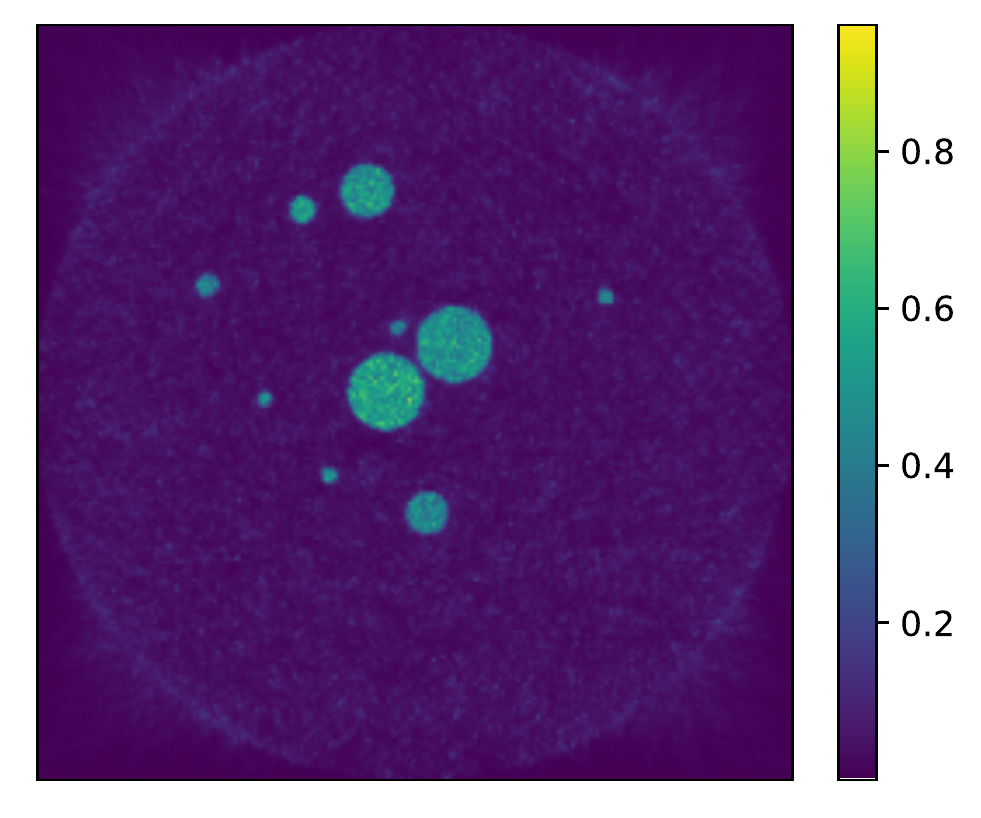} \\
(a) $\CX=\CY=\CL^2$ & (b)$\CX=\CY=\CL^{1.3}$, $p_{\CY}=1.3$ & (c) $\CX=\CY=\CL^{1.1}$, $p_{\CY}=1.1$ \\
$\delta_1(\xbs)/\delta_2(\xbs)$: $17.48/3.52$ & $\delta_1(\xbs)/\delta_2(\xbs)$: $3.46/0.86$ & $\delta_1(\xbs)/\delta_2(\xbs)$: $3.14/0.62$
\end{tabular}
\caption{The reconstructed phantom from the sinograms with a Gaussian pre-measurement and a salt-and-pepper {\rm(}post-{\rm)}measurement noise.
The algorithms use $N_b=60$ batches and were run for $400$ epochs.}
\label{fig:CT_prepost_gausssnp}
\end{figure}

\section*{Acknowledgements}
{We are very grateful to three anonymous referees for their constructive comments which have led to a significant improvement of the quality of the paper.}
\appendix
\section{Technical results and proofs}
\begin{lemma}[{\cite[Lemma 6]{P87}}]\label{lem:polyak_series}
Let $(\delta_n)_n$ be a sequence of non-negative scalars, $(\mu_n)_n$ a sequence of positive scalars, and $\alpha>0$.
If
\[
\delta_{n+1}\leq \delta_n-\mu_{n+1}\delta_n^{1+\alpha}, \text{ for all } n=0,\ldots,N,
\]
then
\[
\delta_N\leq {\delta_0}{\Big(1+\alpha\delta_0^\alpha\sum_{n=1}^{N}\mu_n \Big)^{-1/\alpha}}.
\]
\end{lemma}

\subsection{Two elementary estimates}
In this section, we present two elementary estimates on the SGD iterates for exact data that are useful in establishing the regularising property.
\begin{lemma}\label{lem:iterate_boundedness}
Let the sequence $(\xbs_k)_{k\in\bbN}$ be generated by iterations \eqref{eqn:sgd}, and let the step-sizes $(\mu_k)_{k\in\mathbb{N}}$ satisfy $\mu_k^{p^\ast-1}\leq\frac{p^\ast}{G_{p^\ast}L_{\max}^{p^\ast}}$ for all $k\in\bbN$, and stochastic update directions $\gbs_k$ be of the form \eqref{eqn:Kaczmarz_linear_gradients}. Then for any $\widehat\xbs\in\CX_{\min}$, the sequence $(\bregman{\xbs_k}{\widehat\xbs})_{k\in\bbN}$ is non-increasing. In particular, if $\bregman{\xbs_0}{\widehat\xbs}\leq \rho$, then $\bregman{\xbs_k}{\widehat\xbs}\leq\rho$ for all $k$.
\end{lemma}
\begin{proof}
Let $\Delta_k=\bregman{\xbs_k}{\widehat\xbs}$. By Lemma \ref{lem:descent_property}, we have
\begin{align*}
\Delta_{k+1} &\leq \Delta_k-\mu_{k+1}\DP{\gbs_{k+1}}{\xbs_k-\widehat\xbs} + \frac{G_{p^\ast}}{p^\ast}\mu_{k+1}^{p^\ast}\xsN{\gbs_{k+1}}^{p^\ast}.
\end{align*}
By the definition of duality map and the choice of the update directions $\gbs_k$, we have
\begin{align*}
\DP{\gbs_{k+1}}{\xbs_k-\widehat\xbs}&=\DP{\svaldmapY{p}(\VA_{i_{k+1}}\xbs_k-\ybs_{i_{k+1}})}{\VA_{i_{k+1}}\xbs_k-\ybs_{i_{k+1}}}=\yN{\VA_{i_{k+1}}\xbs_k-\ybs_{i_{k+1}}}^p,\\
\xsN{\gbs_{k+1}}^{p^\ast}&=\xsN{\VA_{i_{k+1}}^\ast \svaldmapY{p}(\VA_{i_{k+1}}\xbs_k-\ybs_{i_{k+1}})}^{p^\ast}\leq \|\VA_{i_{k+1}}^\ast\|^{p^\ast}\ysN{\svaldmapY{p}(\VA_{i_{k+1}}\xbs_k-\ybs_{i_{k+1}})}^{p^\ast}\\
&\leq L_{\max}^{p^\ast}\yN{\VA_{i_{k+1}}\xbs_k-\ybs_{i_{k+1}}}^{(p-1)p^\ast}=L_{\max}^{p^\ast} \yN{\VA_{i_{k+1}}\xbs_k-\ybs_{i_{k+1}}}^{p}.
\end{align*}
Consequently,
\begin{align*}
\Delta_{k+1}&\leq\Delta_k-\mu_{k+1}\DP{\gbs_{k+1}}{\xbs_k-\widehat\xbs} + \frac{G_{p^\ast}}{p^\ast}\mu_{k+1}^{p^\ast}\xsN{\gbs_{k+1}}^{p^\ast}\\
&\leq \Delta_k-\Big(1-L_{\max}^{p^\ast}\frac{G_{p^\ast}}{p^\ast}\mu_{k+1}^{p^\ast-1}\Big)\mu_{k+1}\yN{\VA_{i_{k+1}}\xbs_k-\ybs_{i_{k+1}}}^{p}.
\end{align*}
Since $\mu_k^{p^\ast-1}\leq\frac{p^\ast}{G_{p^\ast}L_{\max}^{p^\ast}}$ by assumption, $\Delta_{k+1}\leq \Delta_k\leq \Delta_0$, completing the proof. 
\end{proof}

\begin{lemma}[Coercivity of the Bregman distance]\label{lem:bregman_bound_xk_bound}
If $\Delta_k=\bregman{\xbs_k}{\xref}\leq C<\infty$ for all $k$, then $\xN{\xbs_k}^p\leq (2p^\ast)^{p}(\xN{\xref}^p\vee C)$, for all $k\in\bbN$.
\end{lemma}
\begin{proof}
By the definition of $\Delta_k$ and the Cauchy-Schwarz inequality, we have
\[ \Delta_k \ge \frac{1}{p^\ast}\xN{\xbs_{k}}^p+\frac{1}{p}\xN{\xref}^p-\xN{\xref}\xN{\xbs_k}^{p-1}.\]
Then we have $\xN{\xbs_{k}}^{p-1}(\frac{1}{p^\ast}\xN{\xbs_{k}}-\xN{\xref})\le \Delta_k.$
If now $\frac{1}{p^\ast}\xN{\xbs_{k}}-\xN{\xref}\le\frac{1}{2p^\ast}\xN{\xbs_{k}}$, it follows $\xN{\xbs_{k}}^p\le(2p^\ast)^{p}\xN{\xref}^p$.
Otherwise, if $\frac{1}{p^\ast}\xN{\xbs_{k}}-\xN{\xref}\ge\frac{1}{2p^\ast}\xN{\xbs_{k}}$, we have
\[ \frac{1}{2p^\ast}\xN{\xbs_{k}}^p \le \xN{\xbs_{k}}^{p-1}\LRR{\frac{1}{p^\ast}\xN{\xbs_{k}}-\xN{\xref}}\le \Delta_k.\]
Combining these two bounds gives $\xN{\xbs_{k}}^p\le(2p^\ast)^{p}\LRR{\xN{\xref}^p\vee\Delta_k}$.
\end{proof}

\subsection{Proof of Lemma \ref{lem:coupled_noise_convergence}}
To prove Lemma \ref{lem:coupled_noise_convergence}, we need the following simple fact.
\begin{lemma}\label{lem:uniform_bounded_on_filtration}
For any fixed $k\in\bbN$, the clean iterates $\xbs_k$ generated by \eqref{eqn:sgd_cleaniterates} are uniformly bounded, i.e. there exists $C_k>0$ such that
$\sup_{\omega\in\CF_k} \xN{\xbs_k}\leq C_k<\infty.$
\end{lemma}
\begin{proof}
If stepsizes $\mu_k$ satisfy the conditions of Lemma \ref{lem:iterate_boundedness}, the statement 
is direct from Lemma \ref{lem:bregman_bound_xk_bound}, and moreover, $C_k$ can be chosen to 
be independent of $k$. Otherwise we proceed by induction.
The induction basis is trivial. Indeed, by the triangle inequality and the definition of duality maps, we have
\begin{align*}
    \xN{\xbs_{k+1}}^{p-1}&=\xsN{\dmapX{p}(\xbs_k)-\mu_{k+1}\gbs_{k+1}}\\
    &\leq \xN{\xbs_k}^{p-1} + L_{\max}\mu_{k+1}\ysN{\svaldmapY{p}(\VA_{i_{k+1}}\xbs_k-\ybs_{i_{k+1}})}\\
    &\leq \xN{\xbs_k}^{p-1} + L_{\max}\mu_{k+1} \yN{\VA_{i_{k+1}}\xbs_k-\ybs_{i_{k+1}}}^{p-1}\\
    &\leq \xN{\xbs_k}^{p-1} +  L_{\max}^p\mu_{k+1}\xN{\xbs_k-\xref}^{p-1}.
\end{align*}
Now under the inductive hypothesis $\sup_{\omega\in\CF_k} \xN{\xbs_k}\leq C_k<\infty$, we have
\begin{align*}
    \xN{\xbs_{k+1}}^{p-1}&\leq \xN{\xbs_k}^{p-1} + L_{\max}^p\mu_{k+1} (\xN{\xbs_k}^{p-1}+\xN{\xref}^{p-1})\\
    &\leq C_k^{p-1}(1+L_{\max}^p\mu_{k+1}) + L_{\max}^p\mu_{k+1}\xN{\xref}^{p-1}.
\end{align*}
This directly proves the statement of the lemma.
\end{proof}

Now we can present the proof of Lemma \ref{lem:coupled_noise_convergence}.
\begin{proof}[Proof of Lemma \ref{lem:coupled_noise_convergence}]
For any sequence $(\delta_j)_{j\in\bbN}$, with $\lim_{j\rightarrow\infty} \delta_j=0$, we consider 
a sequence of random vectors $(\xbs_k^{\delta_j},\xbs_k)_{j\in\bbN}$. We will show by induction that 
(for any fixed $k\in\bbN$) the sequence $(\bregman{\xbs^{\delta_j}_k}{\xbs_k})_j$ is uniformly 
bounded, i.e. $\sup_{\omega\in\CF_k} \bregman{\xbs^{\delta_j}_k}{\xbs_k}<\infty$, converges to 
$0$ point-wise, and that $\xbs_k^{\delta_j}$ is uniformly bounded. The remaining two claims regarding 
the convergence of $\xN{\xbs^{\delta}_k -\xbs_k}$ and $\xsN{\dmapX{p}(\xbs_k^\delta)-\dmapX{p}(\xbs_k)}$ then follow directly.
For notational brevity, we also suppress the sequence notation $\delta_j$, and only use $\delta$. For the induction base, by Theorem \ref{thm:bregman_properties}(i) and (iv), we have
\begin{align*}
    &\bregman{\xbs^\delta_1}{\xbs_1}=\bregmanS{\xbs_1}{\xbs^\delta_1}\\
    \leq& \frac{G_{p^\ast}}{p^\ast} \xsN{\dmapX{p}(\xbs_0)-\dmapX{p}(\xbs_0)-\mu_1(\gbs_1^\delta- \gbs_1)}^{p^\ast}
    =\frac{G_{p^\ast}}{p^\ast}\mu_1^{p^\ast}\xsN{\gbs_1^\delta- \gbs_1}^{p^\ast},
\end{align*}
where $\gbs_1^\delta=g(\xbs_0, \ybs^\delta, i_1)$ and $\gbs_1=g(\xbs_0, \ybs, i_1)$.
Specifically, in the case \eqref{eqn:Kaczmarz_linear_gradients}, we have
\begin{align*}
    \xsN{\gbs_1^\delta\!-\! \gbs_1\!}^{p^\ast} &\!=\!\xsN{\VA_{i_1}^\ast\!\big(\svaldmapY{p}(\VA_{i_1}\!\xbs_0 \!-\!\ybs_{i_1}^\delta\!)\!-\!\svaldmapY{p}(\VA_{i_1}\!\xbs_0\!-\!\ybs_{i_1}\!)\big)\!}^{p^\ast}\!\\
    &\leq\! L_{\max}^{p^\ast} \ysN{\svaldmapY{p}(\VA_{i_1}\!\xbs_0\!-\!\ybs^\delta_{i_1}\!)\!-\!\svaldmapY{p}(\VA_{i_1}\!\xbs_0\!-\!\ybs_{i_1}\!)\!}^{p^\ast}.
\end{align*}
Since $\CY$ is by assumption uniformly smooth, by Theorem \ref{rem:dmap_properties}(iv), we have
\begin{align*}
   &\ysN{\svaldmapY{p}(\VA_{i_1}\!\xbs_0 \!-\!\ybs_{i_1}^\delta\!)\!-\!\svaldmapY{p}(\VA_{i_1}\!\xbs_0 \!-\!\ybs_{i_1}\!)\!}
    \!\\
\leq& C\! \max\{1,\yN{\VA_{i_1}\!\xbs_0 \!-\!\ybs_{i_1}^\delta\!},\yN{\VA_{i_1}\!\xbs_0\! -\!\ybs_{i_1}\!}\}^{p-1}\bar{\rho}_\CY(\yN{\ybs_{i_1}\!-\!\ybs_{i_1}^\delta\!}\!).
\end{align*}
Upon maximising over $\CF_1$, the term in the maximum is uniformly bounded.
Since $\bar{\rho}_\CY:=\rho_\CY(\tau)/\tau\leq 1$, $\bregman{\xbs_1^\delta}{\xbs_1}$ is uniformly bounded.
Since $\lim_{\tau\rightarrow0}\bar{\rho}_\CY(\tau)=0$, it follows that $\lim_{\delta\searrow0}\bregman{\xbs_1^\delta}{\xbs_1}=0$, point-wise.
By the $p$-convexity of $\CX$, we have
\begin{align*}
    0\leq\frac{C_p}{p}\xN{\xbs_1^\delta-\xbs_1}^p\leq\bregman{\xbs_1^\delta}{\xbs_1}.
\end{align*}
Thus,  $\xN{\xbs_1^\delta-\xbs_1}$ is uniformly bounded and $\lim_{\delta\searrow0}\xN{\xbs_1^\delta-\xbs_1}=0$, point-wise.
By the uniform boundedness of $\xN{\xbs_1^\delta-\xbs_1}$ and Lemma \ref{lem:uniform_bounded_on_filtration}, the sequence $\xbs_1^\delta$ is also uniformly bounded:
\begin{align}\label{eqn:bound_on_noisy_iterates}
\xN{\xbs^{\delta}_1}\leq \xN{\xbs_1^\delta-\xbs_1} + \xN{\xbs_1}.
\end{align}
For some $k>0$, assume that $\bregman{\xbs_k^\delta}{\xbs_k}$ is uniformly bounded, converges to $0$ point-wise, as $\delta\to 0^+$.
Using the $p$-convexity of $\CX$, it follows that $\xN{\xbs_k^\delta-\xbs_k}$ is uniformly bounded and converges to $0$ point-wise, and using again Lemma \ref{lem:uniform_bounded_on_filtration}, it follows that $\xbs_k^\delta$ is also uniformly bounded. Then by Theorem \ref{thm:bregman_properties}(i) and (iv), we have
\begin{align*}
    &\bregman{\xbs^\delta_{k+1}}{\xbs_{k+1}}=\bregmanS{\xbs_{k+1}}{\xbs^\delta_{k+1}}\\
    \leq& \frac{G_{p^\ast}}{p^\ast} \xsN{\dmapX{p}(\xbs_k^\delta)-\dmapX{p}(\xbs_k)-\mu_{k+1}(\gbs_{k+1}^\delta- \gbs_{k+1})}^{p^\ast}\\
    \leq&\frac{G_{p^\ast}}{p^\ast}\big(\xsN{\dmapX{p}(\xbs_k^\delta)-\dmapX{p}(\xbs_k)}+\mu_{k+1}\xsN{\gbs_{k+1}^\delta- \gbs_{k+1}}\big)^{p^\ast}.
\end{align*}
Now we separately analyse the two terms in the parenthesis. First, using the uniform smoothness of $\CX$ (and Theorem \ref{rem:dmap_properties}(iv) with  $\bar{\rho}_{\CX^\ast}(\tau)<C\tau^{p^\ast-1}$, cf. Definition \ref{defn:smoothness_and_convexity}), we have
\begin{align}\label{eqn:bounds_on_duality_distance}
\xsN{\dmapX{p}(\xbs_k^\delta)-\dmapX{p}(\xbs_k)} \leq C \max\{1, \xN{\xbs_k^\delta},\xN{\xbs_k}\}^{p-1} \bar{\rho}_\CX(\xN{\xbs^\delta_k-\xbs_k}).
\end{align}
Since the right hand side is uniformly bounded and converges to $0$ point-wise by the induction hypothesis, the same holds for the left hand side.
Next we decompose the second term into a sum of two perturbation terms
\begin{align*}
    \xsN{g(\xbs_k^\delta,\ybs^\delta, i_{k+1})-g(\xbs_k,\ybs, i_{k+1})}&\leq
     \xsN{g(\xbs_k,\ybs^\delta, i_{k+1})-g(\xbs_k,\ybs, i_{k+1})}\\
    &\quad + \xsN{g(\xbs_k^\delta,\ybs^\delta, i_{k+1})-g(\xbs_k,\ybs^\delta, i_{k+1})}:={\rm I}+{\rm II}.
\end{align*}
First, by the assumption $\CY$ being uniformly smooth and Theorem \ref{rem:dmap_properties}(iv), we have
\begin{align*}
{\rm I} &=\xsN{\VA_{i_{k+1}}^\ast\big(\svaldmapY{p}(\VA_{i_{k+1}}\xbs_k -\ybs^\delta_{i_{k+1}})-\svaldmapY{p}(\VA_{i_{k+1}}\xbs_k -\ybs_{i_{k+1}})\big)} \\
&\leq L_{\max} \ysN{\svaldmapY{p}(\VA_{i_{k+1}}\xbs_k -\ybs^\delta_{i_{k+1}})-\svaldmapY{p}(\VA_{i_{k+1}}\xbs_k -\ybs_{i_{k+1}})}\\
    &\leq C L_{\max} \max\{1,\yN{\VA_{i_{k+1}}\xbs_k -\ybs^\delta_{i_{k+1}}},\yN{\VA_{i_{k+1}}\xbs_k -\ybs_{i_{k+1}}}\}^{p-1}\bar{\rho}_\CY(\yN{\ybs_{i_{k+1}}-\ybs_{i_{k+1}}^\delta}).
\end{align*}
By the induction hypothesis and repeating the arguments from the base of induction, the right hand side is uniformly bounded and converges to $0$ point-wise. Second, similarly, we have
\begin{align*}
{\rm II} &=\xsN{\VA_{i_{k+1}}^\ast\big(\svaldmapY{p}(\VA_{i_{k+1}}\xbs_k^\delta -\ybs^\delta_{i_{k+1}})-\svaldmapY{p}(\VA_{i_{k+1}}\xbs_k -\ybs^\delta_{i_{k+1}})\big)} \\
    &\leq L_{\max} \ysN{\svaldmapY{p}(\VA_{i_{k+1}}\xbs_k^\delta -\ybs^\delta_{i_{k+1}})-\svaldmapY{p}(\VA_{i_{k+1}}\xbs_k^\delta -\ybs^\delta_{i_{k+1}})}\\
    &\leq C L_{\max} \max\{1,\yN{\VA_{i_{k+1}}\xbs_k^\delta -\ybs^\delta_{i_{k+1}}},\yN{\VA_{i_{k+1}}\xbs_k -\ybs^\delta_{i_{k+1}}}\}^{p-1}\bar{\rho}_\CY(\yN{\VA_{i_{k+1}}(\xbs_k^\delta-\xbs_k)}).
\end{align*}
By the same arguments, the right hand side is uniformly bounded.
Moreover, $\yN{\VA_{i_{k+1}}(\xbs_k^\delta-\xbs_k)}\leq L_{\max} \xN{\xbs_k^\delta-\xbs_k}$, which by the induction hypothesis converges point-wise to $0$. Putting all these bounds together yields that $ \bregman{\xbs^\delta_{k+1}}{\xbs_{k+1}}$ is uniformly bounded and converges point-wise to $0$.
Using Vitaly's theorem, the desired statement follows directly.
Since $\bregman{\xbs^\delta_{k}}{\xbs_{k}}$ is uniformly bounded and converges point-wise to $0$ for any $k$, then so does $\xN{\xbs^\delta_{k}-\xbs_{k}}$, and consequently by the inequality \eqref{eqn:bounds_on_duality_distance} (and \eqref{eqn:bound_on_noisy_iterates}) so does $\xsN{\dmapX{p}(\xbs_k^\delta)-\dmapX{p}(\xbs_k)}$.
The second part of the claim thus follows. This completes the proof of the induction step, and hence also the lemma.
\end{proof}
\bibliographystyle{abbrv}

\end{document}